\documentclass{article}
\usepackage{latexsym}
\usepackage{amsmath, amsfonts, amsthm}
\usepackage{epsfig}
\usepackage{stmaryrd}
\usepackage{multicol}
\usepackage{pdfsync}
\usepackage{dsfont}
\usepackage{comment}
\usepackage{algorithmic}
\usepackage{algorithm}
\usepackage{caption}
\usepackage{psfrag}
\usepackage{subfig}
\usepackage{xcolor}
\usepackage{graphics}
\usepackage{enumitem}
\usepackage{mathtools}
\usepackage{bbm}
\usepackage{mathrsfs} 
\usepackage{color}
\usepackage{hyperref}
\usepackage[title]{appendix}
\usepackage{bm}

\usepackage{tikz-cd}

\hypersetup{
colorlinks = true, 
urlcolor = blue, 
linkcolor = red, 
citecolor = blue 
}

\newcommand{\prob}{{\bf P}}

\newcommand{\e}{{\bf E}}

\newcommand{\norm}[1]{\left\lVert #1 \right\rVert}
\newcommand{\bae}{\begin{equation}\begin{aligned}}
\newcommand{\eae}{\end{aligned}\end{equation}}
\newcommand{\beq}{\begin{equation}}
\newcommand{\eeq}{\end{equation}}

\newtheorem{theorem}{Theorem}[section]
\newtheorem{corollary}[theorem]{Corollary}
\newtheorem{lemma}[theorem]{Lemma}

\theoremstyle{definition}

\newtheorem{assumption}[theorem]{Assumption}

\theoremstyle{remark}
\newtheorem{remark}[theorem]{Remark}

\numberwithin{equation}{section}

\DeclareMathOperator*{\argmax}{arg\,max}

\DeclarePairedDelimiter\floor{\lfloor}{\rfloor}

\title{Global Convergence of the ODE Limit for Online Actor-Critic Algorithms in Reinforcement Learning}

\author{Ziheng Wang\footnote{Mathematical Institute, University of Oxford, Oxford, OX2 6GG, UK (wangz1@math.ox.ac.uk)} \ and Justin Sirignano\footnote{Mathematical Institute, University of Oxford, Oxford, OX2 6GG, UK (Justin.Sirignano@maths.ox.ac.uk).} }

\begin{document}

\maketitle

\begin{abstract}
Actor-critic algorithms are widely used in reinforcement learning, but are challenging to mathematically analyse due to the online arrival of non-i.i.d. data samples. The distribution of the data samples dynamically changes as the model is updated, introducing a complex feedback loop between the data distribution and the reinforcement learning algorithm. We prove that, under a time rescaling, the online actor-critic algorithm with tabular parametrization converges to an ordinary differential equation (ODE) as the number of updates becomes large. The proof first establishes the geometric ergodicity of the data samples under a fixed actor policy. Then, using a Poisson equation, we prove that the fluctuations of the data samples around a dynamic probability measure, which is a function of the evolving actor model, vanish as the number of updates become large.  Once the ODE limit has been derived, we study its convergence properties using a two time-scale analysis which asymptotically de-couples the critic ODE from the actor ODE. The convergence of the critic to the solution of the Bellman equation and the actor to the optimal policy are proven. In addition, a convergence rate to this global minimum is also established. Our convergence analysis holds under specific choices for the learning rates and exploration rates in the actor-critic algorithm, which could provide guidance for the implementation of actor-critic algorithms in practice. 
\end{abstract}

\section{Introduction}

\hspace{1.4em}	Actor-critic (AC) algorithms \cite{ konda2000actor, konda1999actor} have become some of the most successful and widely-used methods in reinforcement learning (RL) \cite{sutton2018reinforcement}.
AC algorithms are typically implemented in two ways: either as a batch or online algorithm. In the batch setting, there is a ``double for loop" where one update of the actor in the outer for loop is followed by a large number of critic updates in the inner for loop to obtain a good approximation of the value function for the current policy. The convergence of batch AC has recently been studied in \cite{kumar2019sample, wang2019neural, yang2019global}. An online, two time-scale AC algorithm was first proposed in \cite{konda2000actor}, where the actor and critic are updated simultaneously with i.i.d. data samples. In this paper, we study a class of online actor-critic \cite{khodadadian2021finite, wu2020finite, xu2020non}  algorithms where the data samples arrive from a Markov chain \cite{konda2002actor} (instead of i.i.d. data samples) and prove the actor/critic converge to the solution of an ODE as the number learning steps becomes large. It is then proven that the solution of the ODE converges to the optimal policy. 

We consider an actor-critic algorithm where the actor and critic are updated simultaneously at each new time step by using the data samples from simultaneous simulations of two different Markov decision processes (MDPs). Specifically, the data samples used to update the critic are from the original MDP while the samples for the actor are from an artificial MDP with a sightly different transition probability (which will be clearly defined in Section \ref{ActorCriticAlgorithm}) such that the update direction of the actor asymptotically convergences to the unbiased policy gradient direction (see the algorithm in \cite{yang2019global} for details). The data samples from the MDPs are non-i.i.d. and the transition probability function depends upon the action selected at each time step. Actions are selected using the actor's current policy. Therefore, the stationary distributions of the MDPs change as the actor evolves during learning. In order for the critic converge to the value function, an exploration component is included in the selection of the actions, where the exploration decays to zero as the number of learning steps becomes large. We find that carefully choosing the decay rate for the exploration as well as the learning rate is crucial for proving global convergence of the limit ODEs to the optimal policy and the learning rates we use can be easily implemented in practice.

\subsection{Related literature} 

\paragraph{Policy gradient} The policy gradient (PG) method \cite{sutton1999policy} is one of the most important concepts in RL and has achieved great empirical success \cite{schulman2015trust, schulman2017proximal}. However, PG algorithms involve non-convex optimization problems for tabular policy parameters \cite{agarwal2020optimality} and are thus difficult to analyse mathematically. Recently, \cite{agarwal2020optimality, bhandari2019global, khodadadian2021linear, mei2021understanding, mei2020global} have established the convergence and convergence rate to the global optimum for the standard PG method by assuming the value function is known. \cite{bhandari2019global} proved that projected PG on the simplex does not suffer from spurious local optima. \cite{agarwal2020optimality} proves that all of the stationary points of PG for a softmax tabular policy are actually the global optimum and natural PG converges at rate $O\left(\frac{1}{t}\right)$. \cite{mei2020global} proves the convergence rate $O\left(\frac1t\right)$ for the PG method with a softmax tabular policy. 

\paragraph{Actor-critic}
The AC algorithm was first developed in \cite{sutton1999policy} and then extended to the Natural Actor-Critic (AC) algorithm in  \cite{peters2008natural}.
Batch AC algorithms \cite{kumar2019sample, wang2019neural, yang2019global, yang2018finite} involve a ``double for loop" where the outer iteration updates the actor and, for each update of the actor, there is a large sub-iteration to solve the critic. \cite{yang2019global} studied the global convergence of AC algorithms under the Linear Quadratic Regulator. \cite{yang2018finite} analyzed the finite-sample performance of the batch AC algorithm. \cite{kumar2019sample} considered the sample complexity for the “decoupled” AC methods with i.i.d. data samples. \cite{wang2019neural}, under the over-parametrized two-layer neural-network proved that the neural AC algorithm converges to a global optimum at a sub-linear rate. In online AC \cite{khodadadian2021finite, konda2000actor, wu2020finite, xu2020non}, the actor and critic models update simultaneously but with two time-scales. The actor updates at a slower rate while critic updates more quickly to provide the actor an accurate policy gradient. 
\cite{konda2000actor} studies an online AC algorithm with Markovian data samples without using the ODE method and prove convergence to a stationary point. \cite{wu2020finite} proves that two time-scale algorithms with non-i.i.d. data samples and linear function approximation finds an $\epsilon$-stationary point with $O(\epsilon^{-\frac52})$ samples, where $\epsilon$ measures the squared norm of the policy gradient. \cite{xu2020non}, under the compatibility condition \cite{kakade2001natural, sutton1999policy} between actor and critic, shows that two time-scale AC requires sample complexity at order $O(\epsilon^{-2.5}\log^3(\epsilon^{-1}))$ to converge to an $\epsilon$-stationary point. By carefully decreasing the exploration rate, \cite{khodadadian2021finite} shows that the two time-scale natural AC algorithm has sample complexity of $O\left(\delta^{-6}\right)$ for convergence to the global optimum. \cite{khodadadian2021finite2}
proposes an off-policy variant of the natural AC algorithm based on Importance Sampling, where they use the Q-trace algorithm for the critic and provide a sample complexity of $O\left( \epsilon^{-3} \log\left( \frac1 \epsilon\right) \right)$.

\paragraph{Stochastic approximation in RL}
Stochastic approximation \cite{borkar1997stochastic, borkar2009stochastic, borkar2000ode} can be seen as a general framework to analyze RL algorithms. Two time-scale stochastic approximations \cite{dalal2018finite, gupta2019finite} are one of the most popular methods for AC \cite{borkar1997actor, khodadadian2021finite, konda1999actor,  wu2020finite, xu2020non} algorithms.
\cite{borkar1997stochastic, borkar2009stochastic, borkar2000ode} establish the classical ODE method and use it for the stability and convergence analysis of the (two time-scale) stochastic approximation where the stochastic error is a martingale difference sequence.   \cite{dalal2018finite, gupta2019finite} proved convergence rate and finite time analysis for the two time-scale linear stochastic approximation in RL under an i.i.d. assumption. \cite{borkar1997actor, konda1999actor} use the ODE method for two time-scale stochastic approximation in AC algorithms where the actor is updated by a policy iteration algorithm.

Our paper studies a different class of algorithms than this previous literature. We consider the global convergence of the ODE limit for the online tabular AC algorithm. First, we use a time re-scaling \cite{sirignano2019asymptotics} of the algorithm \eqref{ActorCriticAlgorithm} to map it into a time interval $[0,T]$, and the mathematical analysis required to prove convergence to the ODE limit is different from the classical ODE method in stochastic approximation theory \cite{borkar1997stochastic, borkar2009stochastic, borkar2000ode, konda2003linear}. Second, unlike the batch AC with a nested loop structure \cite{kumar2019sample, wang2019neural}, our online algorithm updates the actor and critic simultaneously with dynamic Markovian sampling.  Third,  \cite{konda2000actor, wu2020finite, xu2020non} also studies an online AC algorithm with non-i.i.d. data samples. However,  they only prove convergence to a stationary point while we prove global convergence for the tabular AC algorithm by analyzing the limit ODE. The algorithm in \cite{konda2003linear} uses data samples which arrive from a time non-homogeneous Markov chain (non-i.i.d). However, the value function in \cite{konda2003linear} is the averaged reward while ours is the discounted sum of the rewards and \cite{konda2003linear} uses policy iteration to update the actor while we use the policy gradient theorem.

In our paper, we include exploration in the policy so that the Markov chain visits all states and actions. The exploration decays to zero at a certain rate as the number of learning steps become large. A careful choice of the exploration rate and the learning rate is necessary in order to prove global convergence. In particular, the exploration rate does not satisfy the standard conditions (sum of the squares is finite) in stochastic approximation theory in \cite{borkar1997stochastic, borkar2009stochastic, borkar1997actor, borkar2000ode, konda1999actor}. However, by using the time-rescaling limit, we are still able to establish an ODE limit for a class of actor-critic algorithms.

\section{Actor-Critic Algorithms} \label{ActorCriticAlgorithm}

\hspace{1.4em} Let $\bm{\mathcal{M}} = (\bm{\mathcal{X}}, \bm{\mathcal{A}}, p, \mu, r, \gamma)$ be an MDP, where $\bm{\mathcal{X}}$ is a finite discrete state space, $\bm{\mathcal{A}}$ is a finite discrete action space, $p(x' | x,a)$ is the transition probability function, $\mu$ is the initial probability distribution of the Markov chain, $r(x,a)$ is a bounded reward function, and the discount factor is $\gamma \in(0,1)$. Let the policy $f(x, a)$ be the probability of selecting action $a$ in state $x$. The state and action-value functions $V^{f}(\cdot): \bm{\mathcal{X}} \rightarrow \mathbb{R}$ and $V^{f}(\cdot,\cdot): \bm{\mathcal{X}} \times \bm{\mathcal{A}} \rightarrow \mathbb{R}$ are defined as the expected discounted sum of future rewards when actions are selected from the policy $f$:
\bae
\label{value function}
V^{f}(x) = \e \left[\sum_{k=0}^{\infty} \gamma^{k} \cdot r\left(x_{k}, a_{k}\right) \mid x_{0}=x\right], \quad V^{f}(x, a) = \e \left[\sum_{k=0}^{\infty} \gamma^{k} \cdot r\left(x_{k}, a_{k}\right) \mid x_{0}=x, a_{0}=a\right], \\ 
\eae
where $a_{k} \sim f\left(x_{k}, \cdot \right),$ and $x_{k+1} \sim p\left(\cdot \mid x_{k}, a_{k}\right)$ for all $k \in \mathbb{Z}^{+}$.\footnote{Note that the series in equation \eqref{value function} converge since $\gamma \in (0,1)$ and $r(x,a)$ is bounded.} Note that the transition kernel $p$ and policy $f$ induce a Markov chain on the state-action space $\bm{\mathcal{X}}\times \bm{\mathcal{A}}$. Then for any $(x,a) \in \bm{\mathcal{X}}\times \bm{\mathcal{A}}$, define the state and state-action visiting measures respectively as $\nu_\mu^{f}$ and $\sigma_\mu^{f}$, where
\beq
\label{visiting}
\nu_\mu^{f}(x)= \sum_{k=0}^{\infty} \gamma^{k} \cdot \prob\left(x_{k}=x\right), \quad \sigma_\mu^{f}(x, a)= \sum_{k=0}^{\infty} \gamma^{k} \cdot \prob\left(x_{k}=x, a_{k}=a\right)
\eeq
and $x_{0} \sim \mu(\cdot), a_{k} \sim f\left(x_{k}, \cdot \right)$, $x_{k+1} \sim p \left(\cdot \mid x_{k}, a_{k} \right)$ for all $k \geq 0$. The goal of reinforcement learning is to learn the optimal policy $f^*$ which maximizes the expected discounted sum of the future rewards:
$$
\max\limits_{f} J(f),
$$
where the objective function $J(f)$ is defined as
\beq
J(f) = \e \left[\sum_{k=0}^{\infty} \gamma^{k} \cdot r\left(x_{k}, a_{k}\right)\right]=\sum\limits_{x\in \bm{\mathcal{X}}} \mu(x) V^{f}(x) = \sum\limits_{(x,a) \in \bm{\mathcal{X}}\times \bm{\mathcal{A}}} \sigma_\mu^{f}(x,a) r(x,a). 
\eeq

Policy-based reinforcement learning methods optimize the objective function over a class of policies $\left\{f_{\theta} \mid \theta \in \bm{\mathcal{B}}\right\}$ using the policy gradient theorem \cite{sutton2018reinforcement}. In practice, the value function in the policy gradient theorem is unknown and must therefore also be estimated by a statistical learning algorithm. Online actor-critic algorithms simultaneously estimate the value function using a critic model and the optimal policy using an actor model. In this paper, we specifically study a class of online actor-critic algorithms where the ``actor'' is a tabular softmax policy 
\bae
\label{policy family}
f_\theta(x,a) = \frac{e^{ \theta(x,a) } }{\displaystyle \sum_{a' \in \bm{\mathcal{A}}} e^{ \theta(x,a') }},
\eae
with parameters $\theta = \big{(} \theta(x,a) \big{)}_{(x, a) \in \bm{\mathcal{X}} \times \bm{\mathcal{A}}}$.  The ``critic" $Q=(Q(x,a))_{(x, a) \in \bm{\mathcal{X}} \times \bm{\mathcal{A}}}$ is also tabular with a separate parameter for each state-action pair. The policy $f_{\theta}(x) = \big{(} f_{\theta}(x, a) \big{)}_{a \in \bm{\mathcal{A}}}$ is a probability distribution on the set of actions $\bm{\mathcal{A}}$.

Define a new MDP $\bm{\widetilde{\mathcal{M}}} = (\bm{\mathcal{X}}, \bm{\mathcal{A}}, \widetilde{p}, \mu, r, \gamma)$ with the transition probability function
\beq
\label{new transition}
\widetilde{p}\left(x^{\prime} \mid x, a\right)=\gamma \cdot p\left(x^{\prime} \mid x, a\right) + (1-\gamma) \cdot \mu\left(x^{\prime}\right),
\eeq    
Note that \eqref{new transition} is similar to the transition probability of MDP $\bm{\mathcal{M}}$ except that with probability $1-\gamma$ the state will be randomly re-initialized with distribution $\mu$ \cite{konda2002actor, wang2019neural, xu2020non}. \cite{konda2002actor} proved that the stationary distribution of $\bm{\widetilde{\mathcal{M}}}$ under policy $f$ is the $\frac{1}{1 - \gamma} \sigma_{\mu}^f$ in \eqref{visiting}. At the learning step k, we use $\theta_k$ to denote the estimate for the policy parameters while $Q_k$ is the estimate for the value function under the policy $f_{\theta_k}$. At step $k$, the sample $(\widetilde{x}_k. \widetilde{a}_k)$ used to update the actor parameters $\theta_k$ is generated from MDP $\bm{\widetilde{\mathcal{M}}}$ by policy $f_{\theta_k}$. Then we use the policy gradient theorem \cite{sutton1999policy}  to update the actor and get new policy $f_{\theta_{k+1}}$. The sample $(x_k, a_k)$ is sampled from MDP $\bm{\mathcal{M}}$ by the exploration policy $g_{\theta_k}$ (see equation \eqref{ActorwithExploration}). We then update the critic by temporal difference learning \cite{watkins1992q} to obtain the new critic approximation $Q_{k+1}$. An exploration policy is used to guarantee that the policy will have a positive probability to visit all states and actions. For notational convenience, we will sometimes use $f_k$ and $g_k$ to denote $f_{\theta_k}$ and $g_{\theta_k}$.

In summary, the samples $\{x_k, a_k\}_{k\ge 1}$ used to train the critic model are sampled from $\bm{\mathcal{M}}$ under the exploration policy $g_{k}$:
\beq
\label{origin MDP samples}
x_0, a_0 \stackrel{p(\cdot |x_0, a_0)}{\xrightarrow{~~~~~~}} x_1 \stackrel{g_0(x_0,\cdot)}{\xrightarrow{~~~~~~}} a_1 \stackrel{p(\cdot |x_1, a_1)}{\xrightarrow{~~~~~~}} x_2  \stackrel{g_1(x_1,\cdot)}{\xrightarrow{~~~~~~}} a_2 \stackrel{p(\cdot |x_2, a_2)}{\xrightarrow{~~~~~~}}x_3 \cdots
\eeq

The samples $\{\widetilde{x}_k, \widetilde{a}_k\}_{k\ge 1}$ for the actor model are sampled from $\bm{\widetilde{\mathcal{M}}}$ under the policy $f_k$: 
\beq
\label{artificial MDP samples}
\widetilde{x}_0, \widetilde{a}_0 \stackrel{\widetilde{p}(\cdot |\widetilde{x}_0, \widetilde{a}_0)}{\xrightarrow{~~~~~~}} \widetilde{x}_1 \stackrel{f_0(\widetilde{x}_0,\cdot)}{\xrightarrow{~~~~~~}} \widetilde{a}_1 \stackrel{\widetilde{p}(\cdot |\widetilde{x}_1, \widetilde{x}_1)}{\xrightarrow{~~~~~~}} \widetilde{x}_2  \stackrel{f_1(\widetilde{x}_1,\cdot)}{\xrightarrow{~~~~~~}} \widetilde{a}_2 \stackrel{\widetilde{p}(\cdot |\widetilde{x}_2, \widetilde{a}_2)}{\xrightarrow{~~~~~~}}\widetilde{x}_3 \cdots
\eeq
and $\theta_k, Q_k$ are updated according to the actor-critic algorithm:
\bae
\label{ACupdates}
Q_{k+1}(x,a) &= Q_k(x,a) + \frac{ \alpha }{N} \Big( r(x_k, a_k) + \gamma  \sum_{a''} Q_k(x_{k+1}, a'') g_k(x_{k+1}, a'') - Q_k(x_k, a_k)  \Big) \partial_{x,a} Q_k(x_k, a_k) \\
\theta_{k+1}(x,a) &= \theta_k(x,a) + \frac{ \zeta^N_k}{N} Q_k(\widetilde{x}_k, \widetilde{a}_k ) \partial_{x,a} \log f_k(\widetilde{x}_k, \widetilde{a}_k),
\eae
for $k = 0, 1, \ldots, T N$ and the notation $\partial_{x, a}$ is defined as the derivative with respect to the location $(x, a)$ in the tabular variable, that is: 
\bae
\partial_{x, a} Q_k(x_k, a_k) &:= \partial_{Q(x, a)} Q_k(x_k, a_k) = \mathbbm{1}_{\{ x_k = x, a_k = a\}}, \\ 
\partial_{x, a} \log f_k(\widetilde{x}_k, \widetilde{a}_k) &:= \partial_{\theta(x, a)} \log f_k(x_k, a_k) = \mathbbm{1}_{\{ \widetilde{x}_k = x\}} \left[ \mathbbm{1}_{\{ \widetilde{x}_k = a\}} - f_k(x, a) \right]
\eae
The actions $a_k$ in (\ref{ACupdates}) are selected from the distribution
\beq
\label{ActorwithExploration}
g_{\theta_k}(x, a) = \frac{\eta^N_k}{d_A} + (1 - \eta^N_k) \cdot f_{\theta_k}(x, a), \quad \forall (x,a) \in \bm{\mathcal{X}} \times \bm{\mathcal{A}}. 
\eeq
where $0 \leq \eta^N_k < 1$ and $d_A = | \bm{\mathcal{A}} |$. That is, with probability $\eta^N_k$, we select an action uniformly at random and, with probability $1 - \eta^N_k$, we select an action from the current estimate for the optimal policy. We let the exploration rate decay during training, i.e., $\eta^N_k \rightarrow 0$ as $k \rightarrow \infty$. Note that the step-size for the online actor-critic algorithm (\ref{ACupdates}) is $\frac{1}{N}$ and the number of learning steps is $T N$. We will later show that as $N \rightarrow \infty$ the critic and actor models converge to the solution of an ODE on the time interval $[0, T]$. Here we highlight that in order for the Q-learning algorithm converge, the policy needs to have positive probability to choose every action (see \cite{melo2001convergence, szepesvari1997asymptotic} for details) and this is why we add exploration in the policy used to generate data samples.

Finally, we remark that there are two limitations of our algorithm. First, in order to calculate an unbiased policy gradient to update the actor in the algorithm \eqref{ACupdates}, we actually sample from two MDPs, which will be computationally expensive in practice (although this is standard in the literature \cite{konda2002actor, wang2019neural, xu2020non}). Second, the exploration combined with the two time scales of the online learning algorithm \eqref{ACupdates} will lead to a slightly slower convergence rate for our algorithm.

\paragraph{Challenges for mathematical analysis:} Convergence analysis of the actor-critic algorithm \eqref{ACupdates} must address several technical challenges. The data samples are non-i.i.d. and their distribution depends upon the actor model, which changes as the parameters are updated. Actions are selected using the actor model, which influences the states visited in the Markov chain and affects the actor model's evolution in the learning algorithm. Thus, actor-critic algorithms introduce a complex feedback loop between the distribution of the data samples and the model updates. Another challenge is that the learning algorithm is not guaranteed to update the model in a descent direction for the objective function, which is an obstacle for proving global convergence to the optimal policy. Finally, due to the softmax policy, the objective function is non-convex.

\paragraph{Overview of the proof:} In our mathematical approach, we prove that the actor-critic algorithm \eqref{ACupdates} converges to an ODE under an appropriate time re-scaling.  We address the challenge of non-i.i.d. data depending upon the actor model in two steps. The proof first establishes the geometric ergodicity of the data samples to a stationary distribution $\pi^{f_{\theta}}$ under a fixed actor policy $f_{\theta}$. Then, using a Poisson equation, we prove that the fluctuations of the data samples around a dynamic probability measure $\pi^{f_{\theta_k}}$, which is a function of the evolving actor model, vanish as the number of updates become large.  

Once the ODE limit has been derived, we study its convergence properties using a two time-scale analysis which asymptotically de-couples the critic ODE from the actor ODE. The convergence of the critic to the solution of the Bellman equation and the actor to the optimal policy are proven. In addition, a convergence rate to this global minimum is also established. In order to prove the global convergence, the learning rate and exploration rate for the actor-critic algorithm must be carefully chosen.

\section{Main Result} \label{ActorCritic}

\hspace{1.4em} We prove that the actor and critic models converge to the solution of a nonlinear ODE system as the learning steps become large. Our results are proven under the following assumptions.

\begin{assumption}
\label{MDP}
The reward function $r$ is bounded in $[0,1]$. $\bm{\mathcal{X}}$ and $\bm{\mathcal{A}}$ are finite, discrete spaces.
\end{assumption}

In addition, an assumption regarding the ergodicity of the Markov chains \eqref{origin MDP samples} and \eqref{artificial MDP samples} is required. 

\begin{assumption}	\label{ergodic assumption}
For any finite $\theta$, the Markov chain $(X,A)$ for the MDP $\bm{\mathcal{M}}$ under exploration policy $g_{\theta}$ and the Markov chain $(\widetilde{X}, \widetilde{A})$ for the MDP $\bm{\widetilde{\mathcal{M}}}$ under policy $f_\theta$ are irreducible and non-periodic. Their stationary distributions  $\pi^{f}, \sigma_\mu^{f}$ (which exist and are unique by Section 1.3.3 of \cite{lawler2006introduction}) are globally Lipschitz in policy $f$. 
\end{assumption}

The global convergence proof also requires a careful choice for the learning rate and exploration rate. 

\begin{assumption}
\label{learning}
The learning rate and exploration rate are:
\bae
\label{learning rates}
\zeta_k^N &= \frac{1}{ 1 + \frac{k}{N} }, \quad \eta_k^N = \frac{1}{ 1 + \log^2 (\frac{k}{N}+1) },	\\
\text{thus} \quad \zeta^N_{\floor{Nt}} \rightarrow \zeta_t &= \frac{1}{ 1 + t }, \quad \eta^N_{\floor{Nt}} \rightarrow \eta_t = \frac{1}{ 1 + \log^2 (t+1) }.
\eae
\end{assumption}	

\begin{remark}
The learning rate and exploration rate in \eqref{learning rates} satisfy the following properties for any integer $n \in \mathbb{N}$:
\bae
\label{hyper property}
\int_0^\infty \zeta_s ds = \infty, \quad \int_0^{\infty} \zeta^2_t dt < \infty, \quad \int_0^\infty \zeta_s\eta_s ds < \infty, \quad \lim_{t \to \infty}\frac{\zeta_t}{\eta^{n}_t} = 0.
\eae
These properties are verified in the Appendix \ref{learning rate property}.
\end{remark}

The main results of this paper are the following theorems. 

\begin{theorem}[Limit Equations]
\label{limit ode}
For any $T>0$, 
\beq
\label{conv to limit ode}
\lim_{N\to \infty}\e \sup_{t\in [0,T]} \left[ \left\| \theta_{\floor{Nt}} - \bar\theta_t \right\| + \left\| Q_{\floor{Nt}} - \bar{Q}_t \right\| \right] = 0,
\eeq	
where $\bar Q_t$ and $\bar \theta_t$ satisfy the nonlinear system of ODEs:
\bae
\label{limit equations}
\frac{d \bar Q_t}{dt}(x,a) &= \alpha  \pi^{g_{\bar\theta_t}}(x, a) \left( r(x, a) + \gamma \sum_{z, a''} \bar Q_t(z, a'') g_{\bar\theta_t}(z,a'') p(z| x, a)    - \bar Q_t(x, a) \right)\\
\frac{d\bar\theta_t}{dt}(x,a) &= \zeta_t \sigma_\mu^{f_{\bar\theta_t}}(x,a) \left[ \bar Q_t(x,a) - \sum_{a'} \bar Q_t(x,a') f_{\bar\theta_t}(x,a') \right],
\eae
with initial condition $(\bar Q_0, \bar \theta_0) = (Q_0, \theta_0)$.
\end{theorem}

Thus, the critic converges to the limit variable $\bar Q_t$ while the actor converges to the limit variable $\bar \theta_t$, where $\bar Q_t$ and $\bar \theta_t$ are solutions to a nonlinear system of ODEs. We then prove the convergence of the limit ODEs \eqref{limit equations} to the value function and optimal policy. The convergence analysis also allows us to obtain convergence rates. 
\begin{theorem}[Global Convergence]
\label{global convergence theorem}
The limit critic model converges to the value function: 
\beq
\label{critic convergence}
\norm{\bar Q_t - V^{f_{\bar \theta_t}} } = O\left(\frac{1}{\log^2 t}\right).
\eeq
For an initial distribution $\mu(x)> 0, \forall x \in \mathcal{X}$, the limit actor model converges to the optimal policy:
\beq
\label{actor convergence}
J(f^*) - J(f_{\bar \theta_t}) = O\left(\frac{1}{\log t}\right),
\eeq
where $f^*$ is any optimal policy.
\end{theorem}

\begin{remark} We highlight several points regarding the convergence results in Theorem \ref{limit ode} and \ref{global convergence theorem}:
\begin{itemize}
\item[(1)] The limit ODEs in Theorem \ref{limit ode} have some  similarities to the literature of stochastic approximation and ODE method by Borkar \cite{borkar1997stochastic, borkar2009stochastic}. In these previous articles, the long-time behaviour of the discrete-time stochastic algorithm will closely follow deterministic ODEs. The derivation of the ODEs and the convergence analysis in our paper are different than in \cite{borkar1997stochastic, borkar2009stochastic}. In our algorithm the learning rate for both actor and critic have the re-scaling $\frac1N$ and we study the algorithm under a time re-scaling. Thus when $N \to \infty$, for any time interval $[t, t+\Delta t]$, the number of parameters updates $\to \infty$. Therefore, the random fluctuations in the algorithm vanish and the algorithm will converge to the limit ODE in Theorem \ref{limit ode}. 
\item[(2)] The polylog convergence rate in Theorem \ref{global convergence theorem} is a consequence of the specific choice of the exploration rate $\eta_t = \frac{1}{1+\log^2(1+t)}$. The effect of the exploration rate $\eta_t$ is similar to the effect of a learning rate. The specific function $\eta_t$ necessary to guarantee convergence is a consequence of our mathematical analysis.
\item[(3)] Theorem \ref{limit ode} and \ref{global convergence theorem} imply a convergence result for the discrete actor-critic algorithm: for any $\epsilon > 0$, there exists a $T$ and $N$ such that $\mathbb{E}[ \norm{ f_{\theta_{\floor{nT}}} - f^{\ast} } ] < \frac{\epsilon}{2}$ for all $n \geq N$. The proof follows directly from Theorems \ref{limit ode} and \ref{global convergence theorem}. Specifically, for any $\epsilon>0$, by \eqref{actor convergence} we can select a $T$ large enough such that $f_{\bar \theta_T}$ is within $\frac{\epsilon}{2}$ of the global optimal policy. Then, we can apply \eqref{conv to limit ode} to select an $N$ large enough to ensure that the discrete algorithm $f_{\theta_{\floor{NT}}}$ is within $\frac{\epsilon}{2}$ of the ODE limit $f_{\bar \theta_T}$.
\item[(4)] The global convergence to the optimal policy strongly relies on the sufficient exploration initial distribution $\mu(x) > 0$ for any state \cite{agarwal2020optimality}. However, as in \cite{agarwal2020optimality} when the numbers of states becomes large, even with the uniformly random action selection mixed in, the visiting measure of the policy could place exponentially small probability (in the size of the state space) on particular states, which will significantly decrease the convergence rate in \eqref{actor convergence}.
\end{itemize}
\end{remark}

\section{Derivation of the limit ODEs}

\hspace{1.4em} We use the following steps to prove convergence to the limit ODEs:
\begin{itemize}
\item Prove \emph{a priori} bounds for the actor and critic models. 
\item Derive random ODEs for the evolution of the actor and critic models. The ODEs will contain stochastic remainder terms from the non-i.i.d. data samples.
\item Use a Poisson equation to estimate the fluctuations of the remainder terms around zero. 
\item Use Gronwall's inequality to obtain the convergence to the limit ODEs.
\end{itemize}

\subsection{A Priori Bounds}\label{boundedness}

\hspace{1.4em} In order to prove convergence to the limit equation, we first establish \emph{a priori} bounds for the parameters. In our proof, we will use $C, C_0$ and $C_T$ to denote generic constants. For notational convenience, we will sometimes use $\xi, \xi'$ and $\xi_k, \widetilde{\xi}_k$ to denote the elements $(x,a), (x',a')$ and data samples $(x_k, a_k), (\widetilde{x}_k, \widetilde{a}_k)$, respectively.

First, we establish a priori estimates for the actor and critic models. 
\begin{lemma}
\label{AC bound}
For any fixed $T> 0, N \in \mathbb{N}$, there exists a constant $C_T$ which only depends on $T$ such that
\bae
&\sup _{(x, a) \in \mathcal{X} \times \mathcal{A}} \left|Q_{k}(x, a)\right| \le C_T < \infty, \quad \forall k \le NT\\
&\sup _{(x, a) \in \mathcal{X} \times \mathcal{A}} \left|\theta_{k}(x, a)\right| \le C_T < \infty,  \quad \forall k \le NT.
\eae
\end{lemma}
\begin{proof}
For the update algorithm in \eqref{ACupdates}
\beq
Q_{k+1}(\xi) = Q_k(\xi) + \frac{ \alpha }{N} \left( r(\xi_k) + \gamma  \sum_{a''} Q_k(x_{k+1}, a'') g_k(x_{k+1}, a'') - Q_k(\xi_k)  \right) \partial_{\xi} Q_k(\xi_k),
\eeq
and we have the bound
\beq
\sup _{\xi \in \bm{\mathcal{X}} \times \bm{\mathcal{A}}} \left|Q_{k+1}(\xi)\right| \leq \sup _{\xi \in  \bm{\mathcal{X}} \times \bm{\mathcal{A}}}\left|Q_{k}(\xi)\right|+\frac{C}{N} \sup _{\xi \in \bm{\mathcal{X}} \times \bm{\mathcal{A}}}\left|Q_{k}(\xi)\right|+\frac{C}{N}.
\eeq
Then, using a telescoping series, we have
\bae
\sup _{\xi \in  \bm{\mathcal{X}} \times \bm{\mathcal{A}}}\left|Q_{k}(\xi)\right| &=\sup _{\xi \in  \bm{\mathcal{X}} \times \bm{\mathcal{A}}}\left|Q_{0}(\xi)\right| + \sum_{j=1}^{k}\left(\sup _{\xi \in \bm{\mathcal{X}} \times \bm{\mathcal{A}}}\left|Q_{j}(\xi)\right| - \sup _{\xi \in  \bm{\mathcal{X}} \times \bm{\mathcal{A}}}\left|Q_{j-1}(\xi)\right| \right) \\
& \leq \sup _{\xi \in  \bm{\mathcal{X}} \times \bm{\mathcal{A}}}\left|Q_{0}(\xi)\right| + \sum_{j=1}^{k}\left(\frac{C}{N} \sup _{\xi \in  \bm{\mathcal{X}} \times \bm{\mathcal{A}}}\left|Q_{j-1}(\xi)\right| + \frac{C}{N}\right) \\
& \leq \sup _{\xi \in \bm{\mathcal{X}} \times \bm{\mathcal{A}}}\left|Q_{0}(\xi)\right| + \frac{C}{N} \sum_{j=1}^{k} \sup _{\xi \in  \bm{\mathcal{X}} \times \bm{\mathcal{A}}}\left|Q_{j-1}(\xi)\right| + C\\
& \leq C + \frac{C}{N} \sum_{j=1}^{k} \sup _{\xi \in  \bm{\mathcal{X}} \times \bm{\mathcal{A}}}\left|Q_{j-1}(\xi)\right|,
\eae
where the last inequality follows from the fact that $Q_0$ is a fixed finite vector. Then, by the discrete Gronwall lemma and using $\frac{k}{N} \le T$, we have
\beq
\label{critic priori}
\sup _{(x, a) \in  \bm{\mathcal{X}} \times \bm{\mathcal{A}}}\left|Q_{k}(x, a)\right| \leq C \exp(\frac{Ck}{N}) \le  C \exp(CT) = C_T, \quad \forall k \leq  NT.
\eeq

Recall that the update for the actor model is
\bae
\theta_{k+1}(\xi) &= \theta_k(\xi) + \frac{ \zeta^N_k}{N} Q_k(\widetilde{\xi}_{k} ) \partial_{\xi} \log f_k(\widetilde{\xi}_k) \\
&= \theta_k(\xi) + \frac{ \zeta^N_k}{N} Q_k( \widetilde{\xi}_{k} ) \mathbbm{1}_{\{ \widetilde{x}_k = x\}} \left[ \mathbbm{1}_{\{ \widetilde{a}_k = a\}} - f_k(\widetilde{x}_k, a) \right],
\eae
which together with the bound for the critic in \eqref{critic priori} leads to
\beq
\sup_{\xi \in \bm{\mathcal{X}} \times \bm{\mathcal{A}}}\left|\theta_{k+1}(\xi)\right| \leq \sup_{\xi \in  \bm{\mathcal{X}} \times \bm{\mathcal{A}}}\left|\theta_{k}(\xi)\right|+\frac{C_T}{N}.
\eeq
Then, using a telescoping series, we immediately obtain the bound in the statement of the lemma. 
\end{proof}

\subsection{ Evolution of the Pre-limit Process}

\hspace{1.4em} From their definitions in \eqref{value function}, $V^{f}(x)$ and $V^{f}(x,a)$ are related via the formula
\beq
V^{f}(x)=\sum_a V^{f}(x, a)f(x,a).
\eeq

Define the state and state-action visiting measures, respectively, as $\nu_\mu^{f}$ and $\sigma_\mu^{f}$, where
\beq
\label{visiting measure}
\nu_\mu^{f}(x)= \sum_{k=0}^{\infty} \gamma^{k} \cdot \prob\left(x_{k}=x\right), \quad \sigma_\mu^{f}(x, a)= \sum_{k=0}^{\infty} \gamma^{k} \cdot \prob\left(x_{k}=x, a_{k}=a\right),
\eeq
where $x_{0} \sim \mu(\cdot), a_{k} \sim f\left(x_{k}, \cdot \right)$ and $x_{k+1} \sim p \left(\cdot \mid x_{k}, a_{k} \right)$ for all $k \geq 0$. By definition, we have
$\sigma_\mu^{f}(x, a)=f(x, a) \cdot \nu_\mu^{f}(x)$ and, by \cite{konda2002actor}, the stationary distribution of $\bm{\widetilde{\mathcal{M}}}$ is the corresponding visitation measure of $\bm{\mathcal{M}}$.  
\paragraph{Notation} We first clarify some of the notation that will be used in the analysis. 
\begin{itemize}
\item[(a)] For any $k \ge 0$, let $\prob_{\theta_k}$ denote the transition probability for the Markov chain $(X,A)$ induced by $\bm{\mathcal{M}}$ under softmax policy $f_k$ and let $\Pi_{\theta_k}$ denote the transition probability for the Markov chain $(\widetilde{X}, \widetilde{A})$ induced by $\bm{\widetilde{\mathcal{M}}}$ under exploration policy $g_k$. That is,
\bae
\label{2dim transition}
\prob_{\theta_k}(x, a ; x', a')&= p(x'|x, a) g_k(x', a'),\\
\Pi_{\theta_k} ( x, a;  x', a') &= \widetilde{p}(x'| x, a) f_k(x', a').
\eae
\item[(b)] Let  $\sigma_\mu^{f_k}$ and $\pi^{g_k}$ denote the stationary distributions (whose existence and uniqueness are given by Assumption \ref{ergodic assumption}) for the transition probability $\Pi_{\theta_k}$ and $\prob_{\theta_k}$, respectively.
\item[(c)]
Define the $\sigma$-field of events generated by the samples $  \xi_1, \cdots, \xi_n, \widetilde{\xi}_1, \ldots, \widetilde{\xi}_n$ in \eqref{origin MDP samples} and \eqref{artificial MDP samples} to be $\mathscr{F}_{n}$. Then, for any Borel function $h(\theta, \xi)$,
\beq
\label{transition integral}
\e\left[h\left(\theta_{n}, \widetilde{\xi}_{n+1}\right) \mid \mathscr{F}_{n}\right] = \sum\limits_{y \in \bm{\mathcal{X}} \times \bm{\mathcal{A}}} h\left(\theta_{n}, y \right) \Pi_{\theta_{n}}\left(\xi_{n}; y \right).
\eeq
For any function $h(\theta, \xi)$, we shall denote the partial mapping $\xi \to h(\theta, \xi)$ by $h_\theta$ and define the function
$$
\Pi_{\theta}h_\theta(\xi) := \sum_{y \in \bm{\mathcal{X} \times \mathcal{A}}} h(\theta, y) \Pi_\theta(\xi; y).
$$
\end{itemize}

Using the visiting measures in \eqref{visiting measure}, the policy gradient can be evaluated using the following formula. 
\begin{theorem}[Policy Gradient Theorem \cite{sutton1999policy}]
For the MDP starting from $\mu$, the policy gradient for $f_{\theta}$ is 
\beq
\label{policy gradient theorem}
\nabla_{\theta} J(f_{\theta}) = \sum_{x,a} \sigma_\mu^{f_{\theta}}(x,a) V^{f_{\theta}}(x,a) \nabla_{\theta} \log f_{\theta}(x,a),
\eeq
\end{theorem}

Let the advantage function of policy $f$ denoted by
\beq
A^{f}(x, a) = V^{f}(x, a)-V^{f}(x),\quad \forall (x, a) \in  \bm{\mathcal{X}} \times \bm{\mathcal{A}},
\eeq
and the gradient $\nabla_\theta J(f_\theta)$ can be evaluated using the following formula when $f_{\theta}$ satisfies the softmax policy \eqref{policy family}.
\begin{lemma}
\label{softmax gradient}
Define $\partial_{x,a} J(f_\theta) := \frac{\partial J(f_\theta)}{\partial \theta(x,a)}$ and then for the tabular policy \eqref{policy family}, by policy gradient theorem \eqref{policy gradient theorem}, we have 
\beq
\label{policy gradient}
\partial_{x,a} J(f_\theta) =  \sigma_\mu^{f_{\theta}}(x,a)A^{f_{\theta}}(x,a).
\eeq
\end{lemma}

\begin{proof}
By the policy gradient theorem, we have 
\bae
\label{element}
\partial_{x,a} J(f_\theta) &=  \sum_{x', a'} \nu_\mu^{f_\theta}(x')f_\theta(x', a') \mathbbm{1}_{\{ x' = x\}} \left[ \mathbbm{1}_{\{ a' = a\}} - f_\theta(x', a) \right] V^{f_\theta}(x', a') \\
&= \sum_{a'} \nu_\mu^{f_\theta}(x)f_\theta(x, a')  \left[ \mathbbm{1}_{\{ a' = a\}} - f_\theta(x, a) \right] V^{f_\theta}(x, a') \\
&= \nu_\mu^{f_\theta}(x)f_\theta(x, a) V^{f_\theta}(x, a) -  \nu_\mu^{f_\theta}(x)f_\theta(x, a)  \left[ \sum_{a'}f_\theta(x, a') V^{f_\theta}(x, a') \right]\\
&= \nu_\mu^{f_\theta}(x)f_\theta(x, a) A^{f_\theta}(x, a)\\
&= \sigma_\mu^{f_{\theta}}(x,a)A^{f_{\theta}}(x,a).
\eae
\end{proof}	
Using a telescoping series and the update equation for the actor (\ref{ACupdates}),
\beq
\label{discrete}
\theta_{\floor{Nt}}(x,a)=\theta_0(x,a) +\frac{1}{N}\sum_{k=0}^{\floor{Nt}-1}  \zeta^N_k  Q_k(\widetilde{x}_{k}, \widetilde{a}_k ) \partial_{x,a} \log f_k(\widetilde{x}_k, \widetilde{a}_k).
\eeq
Note that $\xi= (x, a), \widetilde{\xi}_k = (\widetilde{x}_k, \widetilde{a}_k)$ and define
\beq
\label{sgd error}
M^{N}_t(\xi) = \frac{1}{N}\sum_{k=0}^{\floor{Nt}-1} \zeta^N_k Q_k(\widetilde{\xi}_k) \partial_{\xi} \log f_k(\widetilde{\xi}_k) - \frac{1}{N}\sum_{k=0}^{\floor{Nt}-1} \sum_{\xi' \in \bm{\mathcal{X}} \times \bm{\mathcal{A}}}  \zeta^N_k Q_k(\xi') \partial_{\xi} \log f_k(\xi') \sigma_\mu^{f_k}(\xi') ,
\eeq
where $\sigma_\mu^{f_k}$ is the visiting measure for $\mathcal{M}$ under policy $f_k$. Combining \eqref{discrete} and \eqref{sgd error}, we obtain the following pre-limit equation for the actor parameters:
\bae
\label{actor pre limit}
&\theta_{\floor{Nt}}(x,a) - \theta_0(x,a) \\
=& \frac{1}{N}\sum_{k=0}^{\floor{Nt}-1} \sum_{\xi' \in \bm{\mathcal{X}} \times \bm{\mathcal{A}} }  \zeta^N_k Q_k(\xi') \partial_{\xi} \log f_k(\xi') \sigma_\mu^{f_k}(\xi')  + M_t^N(x,a) \\
=& \frac{1}{N}\sum_{k=0}^{\floor{Nt}-1} \zeta_k^N \sum_{a'} \nu_\mu^{f_k}(x) f_k(x, a') \left[ \mathbbm{1}_{\{ a' = a\}} - f_k(x, a) \right] Q_k(x, a') + M_t^N(x,a) \\
=& \frac{1}{N}\sum_{k=0}^{\floor{Nt}-1} \zeta^N_k \sigma_\mu^{f_k}(x,a) \left[ Q_k(x,a) - \sum_{a'} Q_k(x,a') f_k(x,a') \right]  + M_t^N(x,a) \\
\overset{(a)}{=}& \int_0^t \zeta^N_{\floor{Ns}} \sigma_\mu^{f_{\floor{Ns}}}(x,a) \left[ Q_{\floor{Ns}}(x,a) - \sum_{a'} Q_{\floor{Ns}}(x,a') f_{\floor{Ns}}(x,a') \right] ds + M_t^N(x,a) + O(N^{-1}),
\eae
where step (a) uses the a priori bound for the critic $Q_k$ in Lemma \ref{AC bound}.

Similarly, we can show that the critic model satisfies
$$
Q_{\floor{Nt}}(\xi)=Q_0(\xi) + \frac{\alpha}{N}\sum_{k=0}^{\floor{Nt}-1}\left[ r(\xi_k) + \gamma  \sum_{a''}Q_k(x_{k+1}, a'') g_k(x_{k+1}, a'') - Q_k(\xi_k) \right]\partial_{x,a} Q_k(x_k, a_k).
$$
Define
\bae
\label{sgd error Q}
M^{1,N}_t(\xi) &= -\frac{1}{N}\sum_{k=0}^{\floor{Nt}-1} Q_k(\xi_k) \partial_{\xi} Q_k(\xi_k) + \frac{1}{N}\sum_{k=0}^{\floor{Nt}-1}
\sum_{\xi'} Q_k(\xi') \partial_{\xi} Q_k(\xi') \pi^{g_k}(\xi'),  \\
M^{2,N}_t(\xi) &= \frac{1}{N}\sum_{k=0}^{\floor{Nt}-1} r(\xi_k) \partial_{\xi} Q_k(\xi_k) - \frac{1}{N}\sum_{k=0}^{\floor{Nt}-1}
\sum_{\xi'} r(\xi') \partial_{\xi} Q_k(\xi') \pi^{g_k}(\xi'), \\
M^{3,N}_t(\xi) &= \frac{1}{N}\sum_{k=0}^{\floor{Nt}-1} \sum_{a''} \gamma Q_k(x_{k+1}, a'') g_k(x_{k+1}, a'') \partial_{\xi} Q_k(\xi_k)\\
&-\frac{1}{N}\sum_{k=0}^{\floor{Nt}-1} \sum_{\xi'} \sum_{z, a''} \gamma Q_k(z, a'') g_k(z, a'') \partial_{\xi} Q_k(\xi') \pi^{g_k}(\xi') p(z| \xi').
\eae
where $\pi^{g_k}$ is the stationary distribution of Markov chain $(X,A)$ induced by $\mathcal{M}$ under policy $g_k$. Note that
$$
\partial_{x,a} Q_k(x_k, a_k) = \mathbbm{1}_{\{ x_k = x, a_k=a \}}.
$$
Then, we obtain the following pre-limit equation for the critic:
\bae
\label{critic pre limit}
Q_{\floor{Nt}}(\xi) &= Q_0(\xi) + \alpha \int_0^t \pi^{g_{\floor{Ns}}}(\xi) \left[r(\xi) + \gamma \sum_{z, a''} Q_{\floor{Ns}}(z, a'') g_{\floor{Ns}}(z, a'') p(z|\xi) - Q_{\floor{Ns}}(\xi) \right]ds\\
&+ \alpha \left( M^{1,N}_t(\xi) + M^{2,N}_t(\xi) + M^{3,N}_t(\xi) \right) + O(N^{-1}).
\eae

\subsection{Poisson Equations}

\hspace{1.4em} Now we rigorously derive the limit ODEs by using a Poisson equation to bound the fluctuations of the non-i.i.d data samples around the trajectory of the limit ODE. In fact, we first prove
\beq
\lim_{N\to \infty}\e \sup_{t\in [0,T]}\left| M_t^N(x,a) \right| = 0, \quad \forall (x,a) \in \bm{\mathcal{X}} \times \bm{\mathcal{A}}.
\eeq
Using a similar method, we can also prove the convergence of $M_t^1, M_t^2,$ and $M_t^3$.

It is known that a finite state Markov chain which is irreducible and non-periodic has a geometric convergence rate to its stationary distribution \cite{meyn2012markov}. We are able to prove a uniform geometric convergence rate for the Markov chains in our paper under the \emph{time-evolving actor policy updated using the actor-critic algorithm \eqref{ACupdates}}.
\begin{lemma} 
\label{stationary}
Let $\Pi^n_{\theta_k}$ denote the $n$-step transition matrix under the policy $f_{\theta_k}$. Then, for any fixed $T> 0$, there exists an integer $n_0$ such that the following uniform estimates hold for all $\{\theta_k\}_{0 \le k\le NT}$ and $N \in \mathbb{N}$ for the algorithm \eqref{ACupdates}. 
\begin{itemize}
\item Lower bound for the stationary distribution: 
\beq
\label{lower bound}
\inf_{k \le NT} \sigma_\mu^{f_k}(x, a) \ge C \epsilon_T^{n_0}, \quad \forall (x, a) \in \bm{\mathcal{X} \times \mathcal{A}},
\eeq
where $C, \epsilon_T>0$ are positive constants. 
\item Uniform geometric ergodicity: 
\beq
\label{geometric}
\sup_{ k \le NT } \| \Pi^{n}_{\theta_k}(\xi; \cdot) - \sigma_\mu^{f_k}(\cdot) \| \le  (1-\beta_T)^{\lfloor \frac{n}{n_0} \rfloor} \quad \forall \xi \in \bm{\mathcal{X} \times \mathcal{A}},
\eeq
where $\beta_T \in (0,1)$ is a positive constant. 
\end{itemize}
\end{lemma}
\begin{proof}
By Assumption \ref{ergodic assumption} and Lemma 1.8.2 of \cite{norris1998markov}, for any fixed $\widetilde\theta \in \mathbb{R}^d$, there exists an $n_0 = n_0(\widetilde\theta) \in \mathbb{N}$ such that
\beq
\label{positive}
\Pi^{n_0}_{\widetilde\theta}\left(\xi; \xi'\right) > 0 \quad \forall \xi, \xi'.
\eeq
For any $(\xi, \xi')$,
\bae
\label{theta0}
\Pi^{n_0}_{\widetilde\theta}(\xi; \xi') &= \sum_{\xi_1, \cdots, \xi_{n_0-1}} \Pi_{\widetilde\theta}(\xi; \xi_1) \cdots \Pi_{\widetilde\theta}(\xi_{n_0-1}; \xi')\\
&= \sum_{\xi_1, \cdots, \xi_{n_0-1}}  \widetilde{p}(x_1|x,a) f_{\widetilde\theta}(x_1,a_1) \cdots  \widetilde{p}(x'|x_{n_0-1}, a_{n_0-1}) f_{\widetilde\theta}(x', a'),
\eae
where the constant $C$ is defined as
\beq
C = C(n_0) := \inf_{x, a, x'} \sum_{\xi_1, \cdots, \xi_{n_0-1}} \widetilde{p}(x_1|x,a) \cdots \widetilde{p}(x'|x_{n_0-1}, a_{n_0-1}) > 0,
\eeq
where $C> 0$ is because \eqref{positive}. 

Due to $f_\theta$ being a softmax policy and the bound from Lemma \ref{AC bound}, there exists a constant $\epsilon_T >0$ such that
\beq
\label{policy lower bound}
\inf_{k\le NT} f_k ( x, a ) > \epsilon_T, \quad \forall (x, a) \in \bm{\mathcal{X} \times \mathcal{A}}. 
\eeq
Then, using similar analysis as in \eqref{theta0} with constant $n_0 = n_0(\widetilde\theta)$ and $C = C(n_0)$, we have for all $k \le NT$ 
\beq
\label{uniform bound}
\Pi^{n_0}_{\theta_k}(\xi; \xi') \ge C \epsilon_T^{n_0}, \quad \forall \xi, \xi'.
\eeq
Thus, we can derive a lower bound for the stationary distribution 
\bae
\inf_{k \le NT} \sigma_\mu^{f_k}(x', a') &= \inf\limits_{k \le NT} \sum_{x, a} \sigma_\mu^{f_k}(x, a) \Pi^{n_0}_{\theta_k}(x, a; x', a')\\
&\ge \inf_{k \le NT} \sum_{x, a} \sigma_\mu^{f_k}(x, a)  C \epsilon_T^{n_0}\\
&\overset{(a)}{=} C \epsilon_T^{n_0}\\
&>0,
\eae
where step (a) is because $\sigma^{f_\theta}$ is a probability and thus the summation equals to 1.
We can now establish the uniform geometric ergodicity of the Markov chain. Let us choose $\beta_T = \inf\limits_{k \le NT} \min\limits_{\xi, \xi'} \Pi^{n_0}_{\theta_k}(\xi, \xi') >0 $ in \eqref{geometric}, where $\beta_T >0$ is by \eqref{uniform bound}. Thus, for $\forall k \le NT$, the Markov chain with transition probability $\Pi_{\theta_k}$ satisfies Doeblin's condition. In particular, we can show that
\beq
\Pi^{n_0}_{\theta_k}(\xi, \xi') \ge \beta_T > 0, \quad \forall \xi, \xi'.
\eeq
Since $n_0$ and $\beta_T$ are independent of $\theta_k$, we can apply Theorem 16.2.4 of \cite{meyn2012markov} to prove that for all $k \le NT$ 
\beq
\| \Pi^{n}_{\theta_k}(\xi; \cdot) - \sigma_\mu^{f_k}(\cdot) \| \le  (1-\beta_T)^{\lfloor \frac{n}{n_0} \rfloor} \quad \forall \xi \in \bm{\mathcal{X} \times \mathcal{A}}, 
\eeq
which proves the uniform geometric ergodicity \eqref{geometric}.
\end{proof}

Then, using the same method as in Lemma \ref{stationary}, we can prove a similar result for the MDP $\bm{\mathcal{M}}$ with exploration policy $g_k$.
\begin{corollary}
\label{origin mdp estimate}
Let $\prob^n_{\theta_k}$ denote the $n$-step transition matrix under policy $g_k$. Then, for any fixed $T < \infty$, there exists an integer $n_0$ and a constant
\beq
\label{exploration lower bound}
C = C(n_0) := \inf_{x, a, x'} \sum_{\xi_1, \cdots, \xi_{n_0-1}} p(x_1|x,a) \cdots p(x'|x_{n_0-1}, a_{n_0-1}) > 0,
\eeq
such that the following uniform estimate holds for all $\{\theta_k\}_{0 \le k\le NT}$ and $N \in \mathbb{N}$ for the update algorithm \eqref{ACupdates}:
\begin{itemize}
\item Lower bound for the stationary distribution:
\beq
\label{uniform bound2}
\inf_{k \le NT} \pi^{g_k}(x, a) \ge C \left( \eta_{\lfloor NT \rfloor}^N \right)^{n_0}, \quad \forall (x,a) \in \bm{\mathcal{X} \times \mathcal{A}}.
\eeq
\item Uniform geometric ergodicity: 
\beq
\label{geometric2}
\sup_{ k \le NT } \| \prob^{n}_{\theta_k}(\xi; \cdot) - \pi^{g_k}(\cdot) \| \le  (1-\beta_T)^{\lfloor \frac{n}{n_0} \rfloor} \quad \forall \xi \in \bm{\mathcal{X} \times \mathcal{A}},
\eeq
where $\beta_T = C \left(\eta_{\lfloor NT \rfloor}^N \right)^{n_0} \in (0,1)$ is a positive constant. 
\end{itemize}
\end{corollary}

\begin{remark}
Without loss of generality, we suppose the integer $n_0$ in Lemma \ref{stationary} and Corollary \ref{origin mdp estimate} are the same. The proof of Corollary \ref{origin mdp estimate} is the same as the proof of Lemma \ref{stationary} and the detailed proof can be found in Appendix \ref{appendix origin mdp}.
\end{remark}

In order to prove the stochastic fluctuation term vanishes as $N \rightarrow \infty$, we first introduce a Poisson equation with a uniformly bounded solution.
\begin{lemma}
\label{possion equation}
For any $N \in \mathbb{N}$, state-action pair $\xi = (x,a)$, $T> 0$ and $k \le NT$, the Poisson equation
\beq
\label{possion}
\nu_{\theta_k}(\xi') - \Pi_{\theta_k}\nu_{\theta_k}(\xi') =  \mathbbm{1}_{\{ \xi' = \xi\}} - \sigma^{f_k}(\xi), \quad \xi' \in \bm{\mathcal{X}\times \mathcal{A}}
\eeq
has a solution\footnote{We do not prove uniqueness of the solution to the Poisson equation \eqref{possion}. For the purposes of our later analysis, it is only necessary to find a uniformly bounded solution $\nu_\theta$ which satisfies \eqref{nu}.}
\beq
\label{nu}
\nu_{\theta_k}(\xi') := \sum_{n\ge 0} \left[ \Pi^n_{\theta_k} (\xi'; \xi) - \sigma^{f_k}(\xi) \right],
\eeq
and there exists a constant $C_T$ (which only depends on $T$) such that
\beq
\label{uniform}
\sup_{k \le NT} \left|\nu_{\theta_k}(\xi')\right| \leq C_T, \quad \forall \xi' \in \bm{\mathcal{X}\times \mathcal{A}}.	
\eeq
\end{lemma}

\begin{proof}
Due to the uniform geometric convergence rate \eqref{geometric} for all $k \le NT$ in Lemma \ref{stationary}, there exists a $\beta_T>0 $ (independent with $k$) such that for any $\xi' \in \bm{\mathcal{X}} \times \bm{\mathcal{A}}$
\beq
\label{uniform ergodicity}
\left| \Pi^{n}_{\theta_k}(\xi'; \xi) - \sigma_\mu^{f_k}(\xi) \right| \le  (1-\beta_T)^{\lfloor \frac{n}{n_0} \rfloor}, \quad\forall k \le NT
\eeq
which can be used to show the convergence of the series in \eqref{nu}. Consequently, $\nu_{\theta_k}$ is well-defined. The uniform bound \eqref{uniform} follows from 
\beq
\left|\nu_{\theta_k}(\xi')\right| \le \sum_{n\ge 0} \left| \Pi^n_{\theta_k} (\xi'; \xi) - \sigma_\mu^{f_k}(\xi) \right| \leq \sum_{n \ge 0} (1-\beta_T)^{\lfloor \frac{n}{n_0} \rfloor} \le C_T.
\eeq 
Finally, we can verify that $\nu_{\theta}$ is a solution to the Poisson equation by observing that
\bae
\Pi_{\theta_k} \nu_{\theta_k}(\xi') &= \sum\limits_{y} \nu_{\theta_k}(y)\Pi_{\theta_k}(\xi';y)\\ 
&= \sum\limits_{y} \left( \sum_{n\ge 0} \left[ \Pi^n_{\theta_k} (y; \xi) - \sigma_\mu^{f_k}(\xi) \right] \right) \Pi_{\theta_k}(\xi'; y) \\ 
&\overset{(a)}{=}  \sum_{n\ge 0} \left( \sum\limits_{y} \left[ \Pi^n_{\theta_k} (y; \xi) - \sigma_\mu^{f_k}(\xi) \right] \Pi_{\theta_k}(\xi'; y) \right) \\ 
&=  \sum_{n\ge 1} \left[ \Pi^{n}_{\theta_k} (\xi'; \xi) - \sigma_\mu^{f_k}(\xi) \right] \\ 
&= \nu_{\theta_k}(\xi') - ( \mathbbm{1}_{\{ \xi' = \xi\}} - \sigma_\mu^{f_k}(\xi) ),
\eae
where the step $(a)$ uses \eqref{uniform ergodicity} and the Dominated Convergence Theorem.
\end{proof}

Using the Poisson equation \eqref{possion equation}, we can prove that the fluctuations of the data samples around a dynamic visiting measure $\sigma_\mu^{f_k}$ decay when the iteration steps becomes large.

\begin{lemma}
\label{fluctuation}
For any fixed state action pair $\xi = (x,a)$ and $T > 0$,    
\beq
\label{online convergence}
\lim_{N \to \infty} \e\left| \frac{1}{N}\sum_{ k=0 }^{ \lfloor NT \rfloor - 1 } \zeta^N_k \left[ \mathbbm{1}_{\{ \widetilde{\xi}_{k} = \xi\}} - \sigma_\mu^{f_k}(\xi)\right] \right| = 0.
\eeq	
\end{lemma}
\begin{proof}
We define the error $\epsilon_k$ to be 
\bae
\epsilon_{k} :=& \zeta^N_k \left[ \mathbbm{1}_{\{ \widetilde{\xi}_{k+1} = \xi\}} - \sigma_\mu^{f_k}(\xi)\right] \\
=& \zeta^N_k \left[ \nu_{\theta_{k}}\left(\widetilde{\xi}_{k+1}\right)- \Pi_{\theta_{k}} \nu_{\theta_{k}}\left(\widetilde{\xi}_{k+1}\right) \right] \\
=& \zeta^N_k\left[ \nu_{\theta_{k}}(\widetilde{\xi}_{k+1})-\Pi_{\theta_{k}} \nu_{\theta_{k}}\left(\widetilde{\xi}_{k}\right) \right] + \zeta_k \left[ \Pi_{\theta_{k}} \nu_{\theta_{k}}\left(\widetilde{\xi}_{k}\right)-\Pi_{\theta_{k}} \nu_{\theta_{k}}\left(\widetilde{\xi}_{k+1}\right) \right],
\eae
where we have used the definition of the Poisson equation \eqref{possion}. Let
\beq
\psi_{\theta}(y) = \Pi_{\theta} \nu_{\theta}(y).
\eeq
Then, we have that
\bae
\sum_{k=0}^{\lfloor NT \rfloor -1} \epsilon_{k}=& \sum_{k=0}^{\lfloor NT \rfloor-1} \zeta^N_k \left[ \nu_{\theta_{k}}(\widetilde{\xi}_{k+1})-\Pi_{\theta_{k}} \nu_{\theta_{k}}\left(\widetilde{\xi}_{k}\right) \right] + \sum_{k=0}^{\lfloor NT \rfloor-1} \zeta^N_k \left[ \left(\psi_{\theta_{k}}\left(\widetilde{\xi}_{k}\right)-\psi_{\theta_{k}}\left(\widetilde{\xi}_{k+1}\right)\right) \right] \\
=& \sum_{k=0}^{\lfloor NT \rfloor-1} \zeta^N_k \left[\nu_{\theta_{k}}(\widetilde{\xi}_{k+1})-\Pi_{\theta_{k}} \nu_{\theta_{k}}\left(\widetilde{\xi}_{k}\right)\right] + \zeta^N_0 \psi_{\theta_{0}}\left(\widetilde{\xi}_{0}\right) + \sum_{k=1}^{\lfloor NT \rfloor-1} \zeta^N_{k} \left[\psi_{\theta_{k}}\left(\widetilde{\xi}_{k}\right)-\psi_{\theta_{k-1}}\left(\widetilde{\xi}_{k}\right)\right] \\
+& \sum_{k=1}^{\lfloor NT \rfloor-1} \left(\zeta^N_{k} - \zeta^N_{k-1} \right) \psi_{\theta_{k-1}}\left(\widetilde{\xi}_{k}\right)- \zeta^N_{\lfloor NT \rfloor-1} \psi_{\theta_{\lfloor NT \rfloor-1}}\left(\widetilde{\xi}_{\lfloor NT \rfloor}\right)
\eae
Define the error term 
\beq
\label{decompose}
\sum_{k=0}^{\lfloor NT \rfloor-1} \epsilon_{k} = \sum_{k=0}^{\lfloor NT \rfloor-1} \epsilon_{k}^{(1)} + \sum_{k=1}^{\lfloor NT \rfloor-1} \epsilon_{k}^{(2)} + \sum_{k=1}^{\lfloor NT \rfloor-1} \epsilon_{k}^{(3)} + \rho_{\lfloor NT \rfloor ; 0},
\eeq
where
\bae
\epsilon_{k}^{(1)} &= \zeta^N_k \left[ \nu_{\theta_{k}}\left(\widetilde{\xi}_{k+1}\right)-\Pi_{\theta_{k}} \nu_{\theta_{k}}\left(\widetilde{\xi}_{k}\right) \right], \\
\epsilon_{k}^{(2)} &= \zeta^N_k \left[\psi_{\theta_{k}}\left(\widetilde{\xi}_{k}\right)-\psi_{\theta_{k-1}}\left(\widetilde{\xi}_{k}\right)\right], \\
\epsilon_{k}^{(3)} &= \left(\zeta^N_{k} - \zeta^N_{k-1} \right)\psi_{\theta_{k-1}}\left(\widetilde{\xi}_{k}\right), \\
\rho_{\lfloor NT \rfloor ; 0} &= \zeta^N_0 \psi_{\theta_{0}}\left(\widetilde{\xi}_{0}\right)-  \zeta^N_{\lfloor NT \rfloor-1} \psi_{\theta_{\lfloor NT \rfloor-1}}\left(\widetilde{\xi}_{\lfloor NT \rfloor}\right).
\eae
To prove the convergence \eqref{online convergence}, it suffices to appropriately bound the fluctuation term $ \left| \sum\limits_{k=0}^{\lfloor NT \rfloor-1} \epsilon_k \right|$. Actually, the first term can be bound due to the martingale property while the second term can be bounded using the uniform geometric ergodicity and Lipschitz continuity. The third and fourth terms are uniformly bounded by \eqref{uniform}. 

For the first term in \eqref{decompose}, note that
\beq
\e\left\{ \nu_{\theta_{k}}\left(\widetilde{\xi}_{k+1}\right) \mid \mathscr{F}_{k}\right\} = \Pi_{\theta_{k}} \nu_{\theta_{k}}\left(\widetilde{\xi}_{k}\right),
\eeq
thus 
$$
\left\{Z_n = \sum_{k=0}^{n-1} \gamma_k^{(1)}, \ \mathscr{F}_n \right\}_{n\ge 0} 
$$
is a martingale and since the conditional expectation is a contraction in $L^{2}$, we have 
\beq
\e \left| \Pi_{\theta_{k}} \nu_{\theta_{k}}\left(\widetilde{\xi}_{k}\right)\right|^{2} \leq \e\left| \nu_{\theta_{k}}\left(\widetilde{\xi}_{k+1}\right) \right|^{2}.
\eeq
Then,
\bae
\e \left| \frac{1}{N} \sum_{k=0}^{\lfloor NT \rfloor-1} \epsilon_k^{(1)} \right|^2 &=\frac{1}{N^2} \sum_{k=0}^{\lfloor NT \rfloor-1} (\zeta^N_k)^2 \e\left| \Pi_{\theta_{k}} \nu_{\theta_{k}}\left(\widetilde{\xi}_{k}\right)-\nu_{\theta_{k}}\left(\widetilde{\xi}_{k+1}\right)\right|^{2}\\
& \leq \frac{4}{N^2} \sum_{k=0}^{\lfloor NT \rfloor-1} (\zeta^N_k)^2 \e\left|\nu_{\theta_{k}}\left(\widetilde{\xi}_{k+1}\right)\right|^{2} \\
& \overset{(a)}{\le} \frac{4C_T}{N^2} \sum_{k=0}^{\lfloor NT \rfloor-1} (\zeta^N_k)^2,
\eae
where the step (a) is by the uniform boundedness \eqref{uniform}. Thus, for any $T >0$,
\beq
\label{error1}
\lim_{N \to \infty} \e \left| \frac{1}{N} \sum_{k=0}^{\lfloor NT \rfloor-1} \epsilon_k^{(1)} \right| = 0.
\eeq 

For the second term of \eqref{decompose}, by the uniform geometric ergodicity \eqref{geometric}, for any fixed $\gamma_0>0$ we can choose $N_0$ large enough such that 
\beq
\sup_{k \le NT} \sum_{n = \lfloor N_0T \rfloor}^{\infty} \left| \Pi^n_{\theta_k}(y, \xi) - \sigma_\mu^{f_k}(\xi) \right| < \gamma_0, \quad \forall y \in \bm{\mathcal{X}} \times \bm{\mathcal{A}} \\
\eeq
\bae
\label{lipschitz bound}
&\left| \frac{1}{N} \sum_{k=1}^{\lfloor NT \rfloor-1} \epsilon_{k}^{(2)} \right|\\
=& \left|  \frac{1}{N} \sum_{k=1}^{\lfloor NT \rfloor-1} \zeta_k^N \left[\psi_{\theta_{k}}\left(\widetilde{\xi}_{k}\right)-\psi_{\theta_{k-1}}\left(\widetilde{\xi}_{k}\right)\right] \right| \\
\leq& \left|  \frac{1}{N} \sum_{k=1}^{\lfloor NT \rfloor-1} \zeta_k^N \left[\sum_{n=1}^{ \lfloor N_0T \rfloor -1} \left[\Pi_{\theta_{k}}^{n}\left(\widetilde{\xi}_k, \xi\right)-\sigma_\mu^{f_k}(\xi)\right]-\sum_{n=1}^{ \lfloor N_0T \rfloor -1} \left[\Pi_{\theta_{k-1}}^{n}\left(\widetilde{\xi}_k, \xi\right)-\sigma_\mu^{f_{k-1}}(\xi)\right]\right]  \right| + 2 C_T \gamma_{0} \\
\le& \left| \frac{1}{N} \sum_{k=1}^{\lfloor NT \rfloor-1} \zeta_k^N \sum_{n=1}^{\lfloor N_0T \rfloor -1} \left[\Pi_{\theta_{k}}^{n}\left(\widetilde{\xi}_k, \xi\right)-\Pi_{\theta_{k-1}}^{n}\left(\widetilde{\xi}_k, \xi\right)\right] \right| + \frac{\lfloor N_0T \rfloor}{N}\left|  \sum_{k=1}^{\lfloor NT \rfloor-1} \zeta_k^N \left[\sigma_\mu^{f_k}(\xi)-\sigma_\mu^{f_{k-1}}(\xi)\right] \right| + 2C_T \gamma_{0}\\
:=& I^N_1 + I^N_2 + 2C_T \gamma_0.
\eae
By Lemma \ref{AC bound}, for any $k \le NT$ we have 
$$
\left\|\theta_{k}-\theta_{k-1}\right\| \le \sum\limits_{x,a \in \mathcal{X} \times \mathcal{A}} \left| \theta_k(x,a) - \theta_{k-1}(x,a) \right| \leq \frac{C_T}{N}
$$

For any finite $n$, $\Pi_{\theta}^{n}$ is Lipschitz continuous in $\theta$. Consequently, 
\bae
I_1^N &\le \frac{ \lfloor N_0T \rfloor}{N} \sum_{k=1}^{\lfloor NT \rfloor-1}\zeta_k^N C\left\|\theta_{k}-\theta_{k-1}\right\| \le  \frac{C_T}{N},\\
I_2^N &\le \frac{ \lfloor N_0T \rfloor }{N} \sum_{k=1}^{\lfloor NT \rfloor-1}\zeta_k^N C\left\|\theta_{k}-\theta_{k-1}\right\|\le \frac{C_T}{N},
\eae
where the constant $C_T$ only depends on the fixed $N_0, T$. Thus, when $\mathrm{N}$ is large enough,
\beq
\left|\frac{1}{N} \sum_{k=1}^{\lfloor NT \rfloor-1} \epsilon_{k}^{(2)} \right| \leq 4C_T \gamma_{0}
\eeq
Since $\gamma_{0}$ is arbitrary,
\beq
\label{error2}
\lim _{N \rightarrow \infty} \mathbf{E}\left|\frac{1}{N} \sum_{k=1}^{\lfloor NT \rfloor-1} \epsilon_{k}^{(2)}\right|=0
\eeq

For the third term of \eqref{decompose}, 
\beq
\e\left|\frac{1}{N} \sum_{k=1}^{\lfloor NT \rfloor-1} \epsilon_k^{(3)}\right| = \frac{1}{N} \e\left| \sum_{k=1}^{\lfloor NT \rfloor-1}\left(\zeta^N_{k} - \zeta^N_{k-1} \right)\psi_{\theta_{k-1}}\left(\widetilde{\xi}_{k}\right)\right| \overset{(a)}{\le} \frac{C_T}{N} \left|\sum_{k=1}^{\lfloor NT \rfloor-1}\left(\zeta^N_{k-1}-\zeta^N_{k}\right)\right| = \frac{C_T}{N}
\eeq
where step (a) is by the uniform bound \eqref{uniform}. Therefore,
\beq
\label{error3}
\lim_{N \to \infty} \e \left| \frac{1}{N} \sum_{k=1}^{\lfloor NT \rfloor-1} \epsilon_k^{(3)} \right| = 0.
\eeq 
Obviously, for the last term of \eqref{decompose} by the boundedness in \eqref{uniform} we have
$$
\lim_{N\to \infty} \frac{1}{N} \rho_{\lfloor NT \rfloor ; 0} = 0,
$$
which together with \eqref{error1}, \eqref{error2} and \eqref{error3} derive the convergence of $ \frac{1}{N} \sum\limits_{k=0}^{\lfloor NT \rfloor-1} \epsilon_k$ and therefore proving \eqref{online convergence}.
\end{proof}

\subsection{Identification of the Limit ODEs}

\hspace{1.4em} We next prove the convergence of $M_t^N$, which will allow us to prove the convergence to the limit ODEs \eqref{limit equations}.
\begin{lemma}
\label{limit lemma}
For any $\xi =(x,a)$ and the stochastic error $M_t^N$ defined in \eqref{sgd error}, we have  
\beq
\lim _{N \rightarrow \infty} \sup _{t \in(0, T]} \mathbb{E}\left|M_{t}^{N}(\xi)\right|=0.
\eeq
\end{lemma}
\begin{proof} 
For any $K \in \mathbb{N}$ and $\Delta=\frac{t}{K},$ we have
\begin{equation}
\begin{aligned}
&M_{t}^{N}(\xi) \\
=& \sum_{j=0}^{K-1} \Delta \frac{1}{\lfloor\Delta N\rfloor} \sum_{k=j\lfloor\Delta N\rfloor}^{(j+1)\lfloor\Delta N\rfloor-1} \zeta^N_k\left(Q_k(\widetilde{\xi}_k) \partial_\xi \log f_{k}\left(\widetilde{\xi}_{k}\right) - \sum_{\xi' \in \bm{\mathcal{X}} \times \bm{\mathcal{A}}} Q_k(\xi')\partial_\xi \log f_{k}\left(\xi'\right) \sigma_\mu^{f_{k}}(\xi')\right) + o(1) \\
=& \sum_{j=0}^{K-1} \Delta \frac{1}{\lfloor\Delta N\rfloor} \sum_{k=j\lfloor\Delta N\rfloor}^{(j+1)\lfloor\Delta N\rfloor-1} \zeta^N_k \left(Q_{j\lfloor\Delta N\rfloor}(\widetilde{\xi}_k) \partial_\xi \log f_{j\lfloor\Delta N\rfloor}\left(\widetilde{\xi}_{k}\right) - \sum_{\xi' \in \bm{\mathcal{X}} \times \bm{\mathcal{A}}} Q_{j\lfloor\Delta N\rfloor}(\xi') \partial_\xi \log f_{j\lfloor\Delta N\rfloor}\left(\xi'\right) \sigma_\mu^{f_{k}}(\xi')\right) \\
+& \sum_{j=0}^{K-1} \Delta \frac{1}{\lfloor\Delta N\rfloor} \sum_{k=j\lfloor\Delta N\rfloor}^{(j+1)\lfloor\Delta N\rfloor-1}  \zeta^N_k \Bigg[ \left( Q_k(\widetilde{\xi}_k) \partial_\xi \log f_{k}\left(\widetilde{\xi}_{k}\right) - \sum_{\xi' \in \bm{\mathcal{X}} \times \bm{\mathcal{A}}} Q_k(\xi') \partial_\xi \log f_{k}\left(\xi'\right) \sigma_\mu^{f_{k}}(\xi') \right) \\
-& \left(Q_{j\lfloor\Delta N\rfloor}(\widetilde{\xi}_k)\partial_\xi \log f_{j\lfloor\Delta N\rfloor}\left(\widetilde{\xi}_{k}\right) - \sum_{\xi' \in \bm{\mathcal{S}} \times \bm{\mathcal{A}}} Q_{j\lfloor\Delta N\rfloor}(\xi') \partial_\xi \log f_{j\lfloor\Delta N\rfloor}\left(\xi'\right) \sigma_\mu^{f_{k}}(\xi')\right) \Bigg] + o(1)\\
:=& \sum_{j=0}^{K-1} \Delta I^N_{1,j} + \sum_{j=0}^{K-1} \Delta I^N_{2,j} + o(1). 	
\end{aligned}
\end{equation}
where the term $o(1)$ goes to zero, at least, in $L^{1}$ as $N \rightarrow \infty$. 

To prove the convergence of the first term, note that
\begin{equation}
\begin{aligned}
&Q_{j\lfloor\Delta N\rfloor}(\widetilde{\xi}_k) \partial_\xi \log f_{j\lfloor\Delta N\rfloor}\left(\widetilde{\xi}_{k}\right) -\sum_{\xi'} Q_{j\lfloor\Delta N\rfloor}(\xi') \partial_\xi \log f_{j\lfloor\Delta N\rfloor}\left(\xi'\right) \sigma_\mu^{f_k}(\xi') \\
=& \sum_{\xi'} Q_{j\lfloor\Delta N\rfloor}(\xi') \partial_\xi \log f_{j\lfloor\Delta N\rfloor}\left(\xi'\right) \mathbbm{1}_{\{ \widetilde{\xi}_k = \xi'\}} - \sum_{\xi'} Q_{j\lfloor\Delta N\rfloor}(\xi') \partial_\xi \log f_{j\lfloor\Delta N\rfloor}\left(\xi' \right) \sigma_\mu^{f_k}(\xi')\\
=& \sum_{\xi'} Q_{j\lfloor\Delta N\rfloor}(\xi') \partial_\xi \log f_{j\lfloor\Delta N\rfloor}\left(\xi'\right) \left[ \mathbbm{1}_{\{ \widetilde{\xi}_k = \xi'\}} - \sigma_\mu^{f_k}(\xi')\right].
\end{aligned}
\end{equation}
Thus, for any $j \in \{ 0,1, \ldots, K \}$,
\bae
\left| I^N_{1,j} \right| &=  \left| \frac{1}{\lfloor\Delta N\rfloor} \sum_{k=j\lfloor\Delta N\rfloor}^{(j+1)\lfloor\Delta N\rfloor-1} \zeta_k^N \sum_{\xi'} Q_{j\lfloor\Delta N\rfloor}(\xi') \partial_\xi \log f_{j\lfloor\Delta N\rfloor}\left(\xi'\right) \left[ \mathbbm{1}_{\{ \widetilde{\xi}_k = \xi'\}} - \sigma_\mu^{f_k}(\xi')\right] \right| \\
&= \left| \sum_{\xi'} Q_{j\lfloor\Delta N\rfloor}(\xi') \partial_\xi \log f_{j\lfloor\Delta N\rfloor}\left(\xi'\right) \frac{1}{\lfloor\Delta N\rfloor}\sum_{k=j\lfloor\Delta N\rfloor}^{(j+1)\lfloor\Delta N\rfloor-1} \zeta_k^N \left[ \mathbbm{1}_{\{ \widetilde{\xi}_k = \xi'\}} - \sigma_\mu^{f_k}(\xi')\right] \right|\\
&\le C_T \sum_{\xi'}\left| \frac{1}{\lfloor\Delta N\rfloor}\sum_{k=j\lfloor\Delta N\rfloor}^{(j+1)\lfloor\Delta N\rfloor-1}  \zeta_k^N \left[ \mathbbm{1}_{\{ \widetilde{\xi}_k = \xi'\}} - \sigma_\mu^{f_k}(\xi')\right] \right|,
\eae
which together with Lemma \ref{fluctuation} proves 
\beq
\label{error}
\lim_{N \to \infty} \e\left| I^N_{1,j} \right| = 0.
\eeq
Thus,
\beq
\sum_{j=0}^{K-1} \Delta I^N_{1,j}  = \Delta \sum_{j=0}^{K-1} O_P(1) = t \frac{\sum_{j=0}^{K-1} O_P(1)}{K},
\eeq
which derives the convergence of the first term.

For the second term, using the bound in Lemma \ref{AC bound}, we have for any $k\le TN $  
\begin{equation}
\begin{aligned}
\label{Q difference}
&\sup _{\xi' \in \mathcal{X} \times \mathcal{A}}\left|Q_{k}(\xi')\right| \le C_T,\\
&\sup_{\xi'} \left| Q_k(\xi') - Q_{k-1}(\xi') \right| \le \frac{C_T}{N}.
\end{aligned}
\end{equation}
Noting that 
$$
\partial_{\xi} \log f_k(\xi') = \mathbbm{1}_{\{ x' = x\}} \left[ \mathbbm{1}_{\{ a'=a \}} - f_k( x', a) \right],
$$ 
then by the Lipschitz continuity of the softmax transformation and \eqref{Q difference} we have 
\begin{equation}
\begin{aligned}
&\left| Q_{k}(\xi') \partial_\xi \log f_{k}\left(\xi'\right) - Q_{k-1}(\xi') \partial_\xi \log f_{k-1}\left(\xi'\right) \right| \\
=& \left| \left[ Q_{k}(\xi') - Q_{k-1}(\xi') \right] \partial_\xi \log f_{k}\left(\xi'\right) \right| + \left| Q_{k-1}(\xi') \left[ \partial_\xi \log f_{k}\left(\xi'\right) -\partial_\xi \log f_{k-1}\left(\xi'\right) \right] \right|\\
\le& \frac{C_T}{N} + C_T \left| \partial_\xi \log f_{k}\left(\xi'\right) -\partial_\xi \log f_{k-1}\left(\xi'\right) \right| \\
\le& \frac{C_T}{N} + C_T \left\| \theta_k - \theta_{k-1} \right\| \le \frac{C_T}{N}.
\end{aligned}
\end{equation}
Then, for any $j \in 0,1, \cdots, K-1$ and any $k\in [j\lfloor\Delta N\rfloor, (j+1)\lfloor\Delta N\rfloor -1 ]$,
\beq
\left| Q_k(\xi')\partial_\xi \log f_{k}\left(\xi'\right) - Q_{j\lfloor\Delta N\rfloor}(\xi')\partial_\xi \log f_{j\lfloor\Delta N\rfloor}\left(\xi'\right) \right| \le \frac{C(k-j\lfloor\Delta N\rfloor )}{N}.
\eeq
Thus,
\bae
\sum_{j=0}^{K-1} \Delta I^N_{2,j}
&\leq C \sum_{j=0}^{K-1} \Delta \frac{1}{\lfloor\Delta N\rfloor} \sum_{k=j \lfloor \Delta N\rfloor}^{(j+1)\lfloor\Delta N\rfloor-1} \zeta_k^N \frac{k-j\lfloor\Delta N\rfloor}{N}\\
&= C \sum_{j=0}^{K-1} \Delta \frac{1}{\lfloor\Delta N\rfloor} \sum_{k=0 }^{\lfloor\Delta N\rfloor-1} \frac{k}{N}\\
&\le  C \sum_{j=0}^{K-1} \Delta \frac{1}{\lfloor\Delta N\rfloor} \frac{\lfloor\Delta N\rfloor^2}{N}\\
&\le  C \sum_{j=0}^{K-1} \Delta  \frac{\lfloor\Delta N\rfloor}{N}\\
&\le  C \sum_{j=0}^{K-1} \Delta^2 \\
&\le  C \Delta.
\eae

Collecting our results, we have shown that
\beq
\lim _{N \rightarrow \infty} \sup _{t \in(0, T]} \e\left|M_{t}^{N}\right| \leq C \frac{T}{K}
\eeq
Note that $K$ was arbitrary. Consequently, we obtain
\beq
\lim _{N \rightarrow \infty} \sup _{t \in(0, T]} \e\left|M_{t}^{N}\right|=0,
\eeq
concluding the proof of the lemma.
\end{proof}
Following the same method, we can finish proving the convergence of the stochastic fluctuation terms and the detailed proof can be found in Appendix \ref{appendix concentration}.
\begin{lemma}
\label{concentration lemma}
For $t \in [0,T]$, $M_t^{1,N}$, $M_t^{2,N}$, $M_t^{3,N}$ $\overset{L^1}{\to} 0$ as $N\to \infty$.
\end{lemma}

Using Lemma \ref{limit lemma} and \ref{concentration lemma}, we can now finish the derivation of the limit ODEs.
\begin{proof}[Proof of Theorem \ref{limit ode}:]
Due to Assumption \ref{ergodic assumption} and the Lipschitz continuity of softmax transformation, \textcolor{black}{we know $\sigma^{f_{\bar\theta_t}}_\mu, \pi^{g_{\bar\theta_t}}$ is Lipschitz continuous in $\bar\theta_t$}. By Theorem 2.2 and Theorem 2.17 of \cite{teschl2012ordinary}, for any initial value, there exists a unique solution on $(0,+\infty)$ for the ODE system \eqref{limit equations}. Let $\bar{Q}_t(x,a)$, $\bar{\theta}_t(x,a)$ be the solution of \eqref{limit equations} with initial value $Q_0, \theta_0$. Using the bound in Lemma \ref{AC bound} and the Lipschitz continuity from Assumption \ref{ergodic assumption}, we have for $t \in [0,T]$ 
\bae
\label{decomposition}
&\left| \sigma_\mu^{f_{\floor{Nt}}}(x,a) Q_{\floor{Nt}}(x,a) -  \sigma_\mu^{f_t}(x,a) Q_t(x,a) \right| \\
\le& \left|  \sigma_\mu^{f_{\floor{Nt}}}(x,a)-  \sigma_\mu^{f_t}(x,a) \right| \cdot \left| Q_{\floor{Nt}}(x,a) \right| +  \sigma_\mu^{f_t}(x,a) \left| Q_{\floor{Nt}}(x,a) - Q_t(x,a) \right| \\
\le& C_T \left[ \left\| \theta_{\floor{Nt}} - \bar\theta_t \right\| + \left\| Q_{\floor{Nt}} - \bar{Q}_t \right\| \right],
\eae
and we can also show for the exploration policy from \eqref{ActorwithExploration} that
\bae
\label{exploration difference}
&\left| g_{\floor{Nt}}(x, a) - g_t(x, a) \right| \\
\le& \frac{\left| \eta^N_{\floor{Nt}} - \eta_t \right|}{d_A} + \left| (1 - \eta^N_{\floor{Nt}}) \cdot f_{\theta_{\floor{Nt}}}(x, a) - (1 - \eta_t) \cdot f_{\theta_t}(x, a) \right| \\
=& \frac{\left| \eta^N_{\floor{Nt}} - \eta_t \right|}{d_A} + \left| f_{\theta_{\floor{Nt}}}(x, a) - f_{\theta_t}(x, a) \right| + \left| \eta^N_{\floor{Nt}} f_{\theta_{\floor{Nt}}}(x, a) - \eta_t f_{\theta_t}(x, a) \right| \\
\le& C\left| \eta^N_{\floor{Nt}} - \eta_t \right| + C \left\| \theta_{\floor{Nt}} - \bar\theta_t \right\|.
\eae
Combining \eqref{actor pre limit}, \eqref{critic pre limit}, and \eqref{limit equations} and using the same decomposition method as in \eqref{decomposition}, we have for $t \in [0,T]$ 
\bae
\label{pass}
&\left\| \theta_{\floor{Nt}} - \bar\theta_t \right\| + \left\| Q_{\floor{Nt}} - \bar{Q}_t \right\| \\
\le& \sum_{(x,a) \in \mathcal{X} \times \mathcal{A}} \left[ \left| \theta_{\floor{Nt}}(x,a) - \bar\theta_t(x,a) \right| + \left| Q_{\floor{Nt}}(x,a) - \bar{Q}_t(x,a) \right| \right] \\
\le& C_T \int_0^t  \left[ \left\| \theta_{\floor{Ns}} - \bar\theta_t \right\| + \left\| Q_{\floor{Ns}} - \bar{Q}_t \right\| \right] ds + \left| M_t^N \right| + \sum\limits_{i=1}^3 \left| M_t^{i,N} \right| + O(N^{-1}) \\
+& C_T \int_0^t \left[ \left| \zeta^N_{\floor{Ns}} - \zeta_s \right| + \left| \eta^N_{\floor{Ns}} - \eta_s \right| \right] ds.
\eae
Define 
\bae
\varphi^N_t &:= \left\| \theta_{\floor{Nt}} - \bar\theta_t \right\| + \left\| Q_{\floor{Nt}} - \bar{Q}_t \right\| \\
B^N_t &:= \left| M_t^N \right| + \sum\limits_{i=1}^3 \left| M_t^{i,N} \right| + O(N^{-1}) + C_T \int_0^t \left[ \left| \zeta^N_{\floor{Ns}} - \zeta_s \right| + \left| \eta^N_{\floor{Ns}} - \eta_s \right| \right] ds.
\eae
Due to Lemma \ref{limit lemma} and  \ref{concentration lemma},
\beq
\label{error decay}
\lim\limits_{N\to \infty} \e \sup\limits_{t \in [0,T]} B_t^N = 0. 
\eeq
Taking the supremum and expectation of \eqref{pass}, 
\beq
\e \sup\limits_{s \in [0,t]} \varphi^N_s \le C_T \int_0^t \e \sup\limits_{r \in [0,s]} \varphi^N_r ds + \e \sup\limits_{s \in [0,t]} B_s^N, \quad \forall t \in [0,T]
\eeq
By Gronwall's lemma, we have  
\beq
\e \sup\limits_{t \in [0,T]} \varphi^N_t \le \e \sup\limits_{t \in [0,T]} B_t^N + C_T \int_0^T \e \sup\limits_{s \in [0,t]} B_s^N dt \le C_T \e \sup\limits_{t \in [0,T]} B_t^N,
\eeq
which together with \eqref{error decay} proves the convergence \eqref{conv to limit ode}.
\end{proof}

\section{Convergence of Limit ODEs}\label{global convergence}

We now study the convergence of the limit actor-critic algorithm, which satisfies the ODE system \eqref{limit equations}. 

\subsection{Critic convergence}\label{critic conv}

\hspace{1.4em} Now we prove convergence of the critic \eqref{critic convergence}, which states that the critic model will converge to the state-action value function during training. We first derive an ODE for the difference between the critic and the value function. Then, we use a comparison lemma, a two time-scale analysis, and the properties of the learning and exploration rates \eqref{hyper property} to prove the convergence of the critic to the value function.

Recall that the value function $V^{g_t}$ satisfies the Bellman equation
\beq
\label{Bellman equation}
r(x,a) + \gamma \sum_{z, a''} V^{g_{\bar \theta_t}}(z,a'') g_{\bar \theta_t}(z, a'') p(z| x,a) - V^{g_{\bar\theta_t}}(x,a) = 0.
\eeq
Define the difference
\beq
\phi_t = \bar Q_t - V^{g_{\bar \theta_t}}. 
\eeq
As a first step, we prove an a priori uniform bound for the critic in the update \eqref{limit equations}.  Without loss of generality, we initialize the ODE as $\bar Q_0 = 0$ (we can always define $\bar Q'_t = \bar Q_t - \bar Q_0$ and prove the uniform bound for $Q'_t$).
\begin{lemma}
\label{uniform critic bound}
For any state $x$ and action  $a$, we have 
\beq
\max\limits_{x, a} \left|\bar Q_t(x, a)\right| \le \frac{2}{1-\gamma}, \quad t\ge 0.
\eeq
\end{lemma}
\begin{proof}
We first prove $\max\limits_{x, a} \bar Q_t(x,a)$ cannot become larger than $\frac{2}{1-\gamma}$. Actually, if $\max_{x,a} \bar Q_t(x,a)$ ever attains $\frac{2}{1-\gamma}$, that is for some $t_0\ge 0$
\beq
\max_{x,a} \bar Q_{t_0}(x,a) = \frac{2}{1-\gamma},
\eeq
then for any state-action pair $(x_0,a_0)$ such that $Q_{t_0}(x_0,a_0) = \frac{2}{1-\gamma}$ we have		
\beq
\label{Q bound}
\frac{d \bar Q_t}{dt}(x_0, a_0)\bigg|_{t=t_0} \le \alpha \pi^{g_{\bar \theta_{t_0}}}(x_0,a_0) \left[ 1 + 2 \frac{\gamma}{ 1 - \gamma} - \frac{2}{ 1 - \gamma } \right] =  -\alpha \pi^{g_{\bar \theta_{t_0}}}(x_0, a_0) \le 0,
\eeq
and therefore $\max_{x,a} \bar Q_t(x,a)$ can never exceed $\frac{2}{1-\gamma}$. Similarly, we can prove 
\beq
\min_{x,a} \bar Q_t(x,a) \ge -\frac{2}{1-\gamma}, \quad t \ge 0,
\eeq
which concludes the proof of the lemma.
\end{proof}

We now develop an ODE comparison principle which will help us to prove the convergence \eqref{critic convergence}.
\begin{lemma}
\label{comparison lemma}
Suppose a non-negative function $Y_t$ satisfies 
\beq
\frac{dY_t}{dt} \le -\frac{C}{\log^{2n_0} t} Y_t + \frac{1}{t}, \quad t \ge t_0, 
\eeq
where $C,n_0$ are constant and $t_0 \geq 0$. Then,
\beq
\label{comparison decay}
Y_t =  O\left( \frac{1}{\log^{4} t} \right).
\eeq
\end{lemma}
\begin{proof}
First, we establish a comparison principle with the following ODE:
\bae
\label{comparison}
\frac{dZ_t}{dt} &= -\frac{C}{\log^{2n_0} t} Z_t + \frac{1}{t} \ \quad t \ge t_0,  \\
Z_{t_0} &= Y_{t_0}.
\eae
Define 
$$
V_t = Y_t - Z_t.
$$
Then, we have $V_{t_0} = 0$ and for any $t \ge t_0$
\bae
\frac{dV_t}{dt} &= \frac{dY_t}{dt} - \frac{dV_t}{dt} \\
& \le -\frac{C}{\log^{2n_0} t} Y_t + \frac{1}{t} -\left(  -\frac{C}{\log^{2n_0} t} Z_t + \frac{1}{t} \right)\\
&= -\frac{C}{\log^{2n_0} t} (Y_t - Z_t)\\
&= -\frac{C}{\log^{2n_0} t} V_t.
\eae
Then, using an integrating factor,
\beq
\frac{d}{dt} \left[ \exp\left\{ \int_{t_0}^t \frac{C}{\log^{2n_0}\tau} d\tau \right\}V_t \right] = \exp\left\{ \int_{t_0}^t \frac{C}{\log^{2n_0}\tau} d\tau \right\} \left[ \frac{dV_t}{dt} + \frac{C}{\log^{2n_0} t} V_t \right] \le 0.
\eeq 
Thus we have $V_t \le \exp\left\{ - \int_{t_0}^t \frac{C}{\log^{2n_0}\tau} d\tau \right\}V_{t_0} = 0, \quad t\ge t_0$. Therefore,
\beq
Y_t \le Z_t \quad t\ge t_0.
\eeq
Then, if we can establish a convergence rate for $Z_t$, we have a convergence rate for $Y_t$.

To solve the ODE \eqref{comparison}, note that
\beq
\frac{d}{dt} \left[ \exp\left\{ \int_{t_0}^t \frac{C}{\log^{2n_0}\tau} d\tau \right\}Z_t \right] = \exp\left\{ \int_{t_0}^t \frac{C}{\log^{2n_0}\tau} d\tau \right\} \left[ \frac{dZ_t}{dt} + \frac{C}{\log^{2n_0} t} Z_t \right] = \frac{1}{t} \exp\left\{ \int_{t_0}^t \frac{C}{\log^{2n_0}\tau} d\tau \right\}.
\eeq
Then,
\bae
Z_t &= \frac{Z_{t_0}}{\exp\left\{ \int_{t_0}^t \frac{C}{\log^{2n_0}\tau} d\tau \right\}} + \frac{ \int_{t_0}^t \frac{1}{s} \exp\left\{ \int_{t_0}^s \frac{C}{\log^{2n_0}\tau} d\tau \right\} ds}{\exp\left\{ \int_{t_0}^t \frac{C}{\log^{2n_0}\tau} d\tau \right\} } \\
&:= I_t^3 + I_t^4.
\eae
Note that for any integer $n$ and constant $\gamma>0$,
\beq
\label{high order}
\lim\limits_{t \to \infty} \frac{\log^{n} t}{t^\gamma} = 0.
\eeq
Thus, without loss of generality, we can suppose $t_0$ is large enough such that
\beq
\log^{2n_0} t \le t, \quad  t \ge t_0.
\eeq
Then, we can show that
\beq
I_t^3 \le \frac{Z_{t_0}}{\exp\left\{ \int_{t_0}^t \frac{C}{\tau} d\tau \right\}} = \frac{Z_{t_0} t_0^C}{t^C}.
\eeq
By L'Hospital's Rule, we have 
\bae
\lim\limits_{t \to \infty} \log^{4} t \cdot I_t^4 &= \lim\limits_{t \to \infty} \frac{ \frac{4 \log^{2n_0+3} t}{Ct} \int_{t_0}^t \frac{1}{s} \exp\left\{ \int_{t_0}^s \frac{C}{\log^{2n_0}\tau} d\tau \right\} ds}{\exp\left\{ \int_{t_0}^t \frac{C}{\log^{2n_0}\tau} d\tau \right\} } + \lim\limits_{t \to \infty} \frac{ \log^{2n_0+4} t}{Ct}\\
&\overset{(a)}{=} \lim\limits_{t \to \infty} \frac{ \int_{t_0}^t \frac{1}{s} \exp\left\{ \int_{t_0}^s \frac{C}{\log^{2n_0}\tau} d\tau \right\} ds}{\exp\left\{ \int_{t_0}^t \frac{C}{\log^{2n_0}\tau} d\tau \right\} } \\
&= \lim\limits_{t \to \infty} \frac{ \frac{1}{t} \exp\left\{ \int_{t_0}^t \frac{C}{\log^{2n_0}\tau} d\tau \right\} }{\exp\left\{ \int_{t_0}^t \frac{C}{\log^{2n_0}\tau} d\tau \right\} \frac{C}{\log^{2n_0}t} }\\
&= \lim\limits_{t \to \infty} \frac{\log^{2n_0} t}{Ct}\\
&=0,
\eae
where step $(a)$ is by \eqref{high order}. Therefore, we can let $t_0$ be large enough such that
\beq
I_t^4 \le \frac{1}{\log^{4} t}, \quad \forall t\ge t_0.
\eeq
Combining our results, we have
\beq
Y_t \le Z_t \le \frac{Y_{t_0} t_0^C}{t^C} + \frac{1}{\log^{4} t}, \quad t\ge t_0,
\eeq
which together with \eqref{high order} proves \eqref{comparison decay}.
\end{proof}

Using Lemma \ref{comparison lemma}, now we prove the critic convergence \eqref{critic convergence}.
\begin{proof}[Proof of \eqref{critic convergence}:]
Combining \eqref{limit equations} and \eqref{Bellman equation},
\beq
\frac{d \phi_t}{dt}(x,a) =  - \alpha \pi^{g_{\bar \theta_t}}(x, a)  \phi_t(x, a)
+ \alpha \gamma \pi^{g_{\bar \theta_t}}(x, a)  \sum_{z, a''} \phi_t(z, a'') g_{\bar \theta_t}(z,a'') p(z| x, a) + \frac{dV^{g_{\bar \theta_t}}}{dt}(x,a). 
\eeq
Let $\odot$ denote element-wise multiplication. Then,
\beq
\frac{d \phi_t}{dt} =  - \alpha \pi^{g_{\bar\theta_t}} \odot \phi_t + \alpha \gamma  \pi^{g_{\bar\theta_t}} \odot  \Gamma_t + \frac{dV^{g_{\bar\theta_t}}}{dt},
\eeq
where $\Gamma_t(x', a') = \displaystyle \sum_{z, a''} \phi_t(z, a'') g_{\bar\theta_t}(z,a'') p(z| x', a')$. 

Define the process
\beq
Y_t =  \frac{1}{2} \phi_t^{\top} \phi_t.
\eeq
Differentiating yields
\beq
\frac{d Y_t}{d t} =  \phi_t^{\top} \frac{d \phi_t}{d t} = - \alpha \phi_t^{\top}  \pi^{g_{\bar \theta_t}} \odot \phi_t   +  \alpha \gamma \phi_t^{\top}   \pi^{g_{\bar \theta_t}} \odot  \Gamma_t  + \phi_t^{\top}\frac{dV^{g_{\bar\theta_t}}}{dt}.
\label{Yeqn0}
\eeq
The second term on the last line of (\ref{Yeqn0}) becomes:
\begin{equation}
\begin{aligned}
&\bigg{|} \phi_t^{\top}   \pi^{g_{\bar\theta_t}} \odot  \Gamma_t \bigg{|} \\
=& \bigg{|} \sum_{x', a'} \phi_t(x', a') \pi^{g_{\bar\theta_t}}(x', a') \sum_{z, a''} \phi_t(z, a'') g_{\bar\theta_t}(z, a'') p(z | x', a') \bigg{|} \notag \\
=& \bigg{|} \sum_{x', a'} \sum_{z, a''} \phi_t(z, a'')  \phi_t(x', a')  g_{\bar\theta_t}(z, a'') p(z | x', a') \pi^{g_{\bar\theta_t}}(x', a')  \bigg{|} \notag \\
\leq& \sum_{x', a'} \sum_{z, a''}  \bigg{|} \phi_t(z, a'')  \phi_t(x', a') \bigg{|}  g_{\bar\theta_t}(z, a'') p(z | x', a') \pi^{g_{\bar\theta_t}}(x', a')  \notag \\
\leq& \frac{1}{2}  \sum_{x', a'} \sum_{z, a''} \bigg{(} \phi_t(z, a'')^2 +  \phi_t(x', a')^2 \bigg{)}  g_{\bar\theta_t}(z, a'') p(z | x', a') \pi^{g_{\bar\theta_t}}(x', a') \notag \\
= & \frac{1}{2}  \sum_{z, a''}  \phi_t(z, a'')^2  \sum_{x', a'}   g_{\bar\theta_t}(z, a'') p(z | x', a') \pi^{g_{\bar\theta_t}}(x', a') + \frac{1}{2}  \sum_{x', a'} \phi_t(x', a')^2 \pi^{g_{\bar\theta_t}}(x', a')  \sum_{z, a''}    g_{\bar\theta_t}(z, a'') p(z | x', a') \notag \\
=&   \frac{1}{2}  \sum_{z, a''}  \phi_t(z, a'')^2 \pi^{g_{\bar\theta_t}}(z, a'')   + \frac{1}{2}  \sum_{x', a'} \phi_t(x', a')^2 \pi^{g_{\bar\theta_t}}(x', a') \notag \\
=&  \sum_{x', a'} \phi_t(x', a')^2 \pi^{g_{\bar\theta_t}}(x', a'). 
\end{aligned}
\end{equation}
where we have used Young's inequality, the fact that $\displaystyle \sum_{z, a''}    g_{\bar\theta_t}(z, a'') p(z | x', a')  = 1$ for each $(x', a')$, and $\displaystyle \sum_{x', a'}   g_{\bar\theta_t}(z, a'') p(z | x', a') \pi^{g_{\bar\theta_t}}(x', a')  = \pi^{g_{\bar\theta_t}}(z, a'')$. Therefore,
\begin{eqnarray}
\frac{d Y_t}{d t} \leq  - \alpha (1 - \gamma) \pi^{g_{\bar\theta_t}} \cdot \phi_t^2   +  \phi_t^{\top} \frac{dV^{g_{\bar\theta_t}}}{dt} ,
\label{Yeqn}
\end{eqnarray}
where $\phi_t^2$ is an element-wise square. 

By the limit ODEs in \eqref{limit equations} and the uniform boundedness in Lemma \ref{uniform critic bound}, we have for any $(x,a)$
\beq
\label{theta uniform bound}
\left| \frac{d\bar \theta_t}{dt}(x,a) \right| = \left| \zeta_t \sigma_\mu^{f_{\bar \theta_t}}(x,a)  \left[\bar Q_t(x,a) - \sum_{a'} \bar Q_t(x,a')f_{\bar\theta_t}(x,a')\right] \right| \le C\zeta_t
\eeq
For any state $x_0$, define $\partial_{x,a} V^{f_\theta}(x_0) = \frac{\partial V^{f_\theta}(x_0) }{\partial \theta(x,a)}$. Then, for the exploration policy \eqref{ActorwithExploration}, by the policy gradient theorem \eqref{policy gradient theorem} we have 
\bae
\label{gradient value}
\left| \partial_{x,a} V^{g_{\bar \theta_t}}(x_0) \right|&= \left| \sum_{x',a'} \sigma_{x_0}^{g_{\bar \theta_t}}(x',a') V^{g_{\bar \theta_t}}(x',a') \partial_{x,a} \log g_{\bar \theta_t}(x',a') \right| \\
&\le C \sum_{x', a'} \left| \partial_{x,a} \log g_{\bar \theta_t}(x',a') \right| \\
&= C (1-\eta_t) \sum_{x', a'} \frac{f_{\bar \theta_t}(x',a')}{g_{\bar \theta_t}(x',a')} \left| \partial_{x,a} \log f_{\bar \theta_t}(x',a') \right| \\
&\overset{(a)}{\le} C,
\eae
where step $(a)$ is by
\beq
\frac{f_{\bar\theta_t}(x',a')}{g_{\bar\theta_t}(x',a')} = \frac{f_{\bar\theta_t}(x',a')}{ \frac{\eta_t}{d_A} + (1 - \eta_t) \cdot f_{\bar\theta_t}(x', a')} \le C
\eeq 
and 
\beq
\left| \partial_{x,a} \log f_{\bar\theta_t}(x',a') \right| = \left| \mathbbm{1}_{\{ x' = x\}} \left[ \mathbbm{1}_{\{ a' = a\}} - f_{\bar\theta_t}(x', a) \right] \right| \le 2.
\eeq
The relationship between the value functions
\beq
\label{value relation}
V^{f_{\bar\theta_t}}(x_0, a_0) = r(x_0, a_0) + \gamma \sum_{x'} V^{f_{\bar\theta_t}}(x') p(x'| x_0, a_0), \quad \forall (x_0, a_0),
\eeq
can be combined with \eqref{gradient value} to derive 
\beq
\label{state action bound}
\left\|\nabla_{\theta} V^{g_{\bar\theta_t}}(x,a)\right\| \le C, \quad \forall (x, a).
\eeq
Combining \eqref{theta uniform bound} and \eqref{state action bound},
\beq
\label{zetaBound}
\left|\frac{dV^{g_{\bar\theta_t}}}{dt}(x,a)\right| = \left|\nabla_{\theta} V^{g_{\bar\theta_t}}(x,a) \cdot \frac{d\bar\theta_t}{dt}\right| \le \left\|\nabla_{\theta} V^{g_{\bar\theta_t}}(x,a)\right\| \cdot \left\| \frac{d\bar\theta_t}{dt} \right\| \le C \zeta_t,
\eeq
where $C>0$ is a constant independent with $T$.

Combining \eqref{Yeqn}, \eqref{zetaBound} and \eqref{uniform bound2}, we have 
\bae
\label{critic}
\frac{d Y_t}{d t} & \le - \alpha (1 - \gamma) \min_{x,a} \{ \pi^{g_{\bar\theta_t}}(x,a) \} Y_t+ C \phi_t^{\top} \zeta_t \\
&\le - \alpha C\eta^{n_0}_t (1 - \gamma) Y_t + C \phi_t^{\top} \zeta_t\\
&\le - C \eta^{n_0}_t Y_t + \frac{ \eta^{n_0}_t }{ \eta^{n_0}_t } \| \phi_t \| C\zeta_t \\
&\le - C \eta^{n_0}_t Y_t + \| \phi_t \|^2 \eta^{2n_0}_t + \frac{C\zeta_t^2}{\eta_t^{2n_0}} \\ 
&= - \eta^{n_0}_t ( C - 2\eta^{n_0}_t )Y_t + \frac{C\zeta_t}{\eta^{2n_0}_t} \zeta_t.
\eae
Since $\frac{\zeta_t}{\eta^{2n_0}_t} \to 0$ as $t \to \infty$, there exists $t_0 \ge 2$ such that $\forall t \ge t_0$
\beq
\frac{dY_t}{dt} \le -C \eta^{n_0}_t Y_t + \zeta_t \le -\frac{C}{\log^{2n_0} t} Y_t + \frac{1}{t},
\eeq
where the $C$ is a constant independent with $t$. Then, by Lemma \ref{comparison lemma}, there exists $t_1\ge t_0$ such that 
\beq
\label{pre critic convergence}
Y_t = O\left( \frac{1}{\log^{4} t} \right) = O\left( \eta^2_t \right). 
\eeq

By the policy gradient theorem \eqref{policy gradient theorem}, we have 
\beq
\frac{\partial V^f(x_0)}{\partial_{f(x,a)}} = V^{f}(x,a) \sigma_{x_0}^{f}(x).
\eeq
Thus, by the relationship \eqref{value relation},
\beq
\frac{\partial V^{f_{\bar\theta_t}}(x_0, a_0)}{\partial_{f(x,a)}} =  \gamma \sum_{x'} V^{f_{\bar\theta_t}}(x,a) \sigma_{x'}^{f_{\bar\theta_t}}(x) p(x'| x_0, a_0) \le C.
\eeq
Then, for any $(x, a) \in \bm{\mathcal{X}} \times \bm{\mathcal{A}}$, there exists $\widetilde{t} \in [0,1]$ such that 
\bae
\label{exploration decay}
\left|V^{g_{\bar\theta_t}}(x, a) - V^{f_{\bar\theta_t}}(x ,a)\right| = \left|\nabla_f V^{ \widetilde{t}f_{\bar\theta_t} + (1-\widetilde{t})g_{\bar\theta_t}}(x, a) \cdot \left[ g_{\bar\theta_t} - f_{\bar\theta_t} \right]\right| \le C\eta_t,
\eae
Finally, combining \eqref{pre critic convergence} and \eqref{exploration decay}, we obtain \eqref{critic convergence}.

\end{proof}

\subsection{Actor convergence}

\subsubsection{Convergence to stationary point} 

In order to prove global convergence, we first show that the actor converges to a stationary point. We introduce the following notation:
\bae
\widehat{\nabla}_\theta J(f_{\bar\theta_t}) &:= \sum_{x,a} \sigma_\mu^{f_{\bar\theta_t}}(x, a) \bar Q_t(x,a)\nabla_{\theta} \log f_{\bar \theta_t}(x,a), \\
\widehat{\partial}_{x,a} J(f_{\bar \theta_t})&:= \sum_{x,a} \sigma_\mu^{f_{\bar \theta_t}}(x, a) \bar Q_t(x,a)\partial_{x,a} \log f_{\bar \theta_t}(x,a).
\eae 

Then, the limit ode for $\theta$ in \eqref{limit equations} can be written as  
\beq
\label{gradient flow}
\frac{d\bar \theta_t}{dt} = \zeta_t \widehat{\nabla}_{\theta} J(f_{\bar \theta_t}).
\eeq
By direct calculations,
\bae
\nabla_{\theta} \log f_\theta(x,a) =& \nabla_\theta \left[\theta(x,a) - \log \sum_{a'} e^{\theta(x,a')}\right] \\
=& \nabla_\theta \theta(x,a) - \frac{\sum\limits_{a'} e^{\theta(x,a')} \nabla_\theta \theta(x,a')}{\sum\limits_{a'} e^{\theta(x,a')}} \\
=& \nabla_\theta \theta(x,a) - \sum_{a'} f_\theta(x,a') \nabla_\theta \theta(x,a') \\
=&\nabla_\theta \theta(x,a) - \e_{a' \sim f_\theta(x,\cdot)}[\nabla_\theta \theta(x,a')]\\
=& e_{x,a} - \sum_{a'} e_{x,a'} f_\theta(x, a'),
\eae	
where $e_{x,a}$ is the unit vector where only the $x,a$ element is $1$ and all other elements are $0$. Then, the difference is 
\bae
\nabla_\theta J(f_{\bar\theta_t}) - \widehat{\nabla}_{\theta} J(f_{\bar\theta_t}) &= \sum_{x,a} \sigma^{f_{\bar\theta_t}}(x, a) \left( \bar Q_t(x,a)- V^{f_{\bar\theta_t}}(x,a) \right) \nabla_{\theta} \log f_{\bar\theta_t}(x,a),\\
&= \sum_{x,a} \sigma^{f_{\bar\theta_t}}(x, a) \left( \bar Q_t(x,a)- V^{f_{\bar \theta_t}}(x,a) \right) \left( e_{x,a} - \sum_{a'} e_{x,a'} f_{\bar\theta_t}(x, a')\right),
\eae
which together with \eqref{critic convergence} derives
\beq
\label{critic error}
\|\nabla_\theta J(\bar \theta_t) - \widehat{\nabla}_{\theta} J(\bar\theta_t)\|_2 \le C  \|\bar Q_t- V^{f_{\bar\theta_t}}\|_2 \le C\eta_t.
\eeq
Thus we re-write the gradient flow \eqref{gradient flow} as  
\beq
\label{gradient flow fluctuation}
\frac{d\bar\theta_t}{dt} = \zeta_t\nabla_{\theta} J(f_{\bar\theta_t})+ \zeta_t \sum_{x,a} \sigma^{f_{\bar\theta_t}}(x,a)\left[ \left( \bar Q_t(x,a)- V^{f_{\bar\theta_t}}(x,a) \right) \cdot \nabla_{\theta} \log f_{\bar\theta_t}(x,a)\right].
\eeq

Now we can adapt the proof in \cite{bertsekas2000gradient} to show the gradient flow converges to a stationary point. We first provide a useful lemma.
\begin{lemma}
\label{convergence lemma}
Let $Y_t, W_t$ and $Z_t$ be three functions such that $W_t$ is nonnegative. Asuume that 
\beq
\label{iteration}
\frac{dY_t}{dt} \ge W_t +Z_t, \quad t\ge 0
\eeq
and that $\int_0^{\infty} Z_t dt $ converges. Then, either $Y_t \to \infty$ or else $Y_t$ converges to a finite value and $\int_0^{\infty} W_t dt < \infty$. 
\end{lemma}
\begin{proof}
For any $\bar t > 0$. By integrating the relationship $\frac{dY_t}{dt} \ge Z_t$ from $\bar t$ to $t \ge \bar t$ and taking the limit inferior as $t \to \infty$, we obtain 
\beq
\liminf_{t \to \infty}Y_t \ge Y_{\bar t} + \int_{\bar t}^{\infty} Z_t dt > -\infty.
\eeq
By taking the limit superior of the right-hand side as $\bar t \to \infty$ and using the fact $\displaystyle \lim_{\bar t \to \infty} \int_{\bar t}^{\infty} Z_t dt =0$, we obtain 
\beq
\liminf_{t \to \infty}Y_t \ge \limsup_{\bar t \to \infty}Y_t > -\infty.
\eeq
This proves that either $Y_t \to \infty$ or $Y_t$ converges to a finite value. If $Y_t$ converges to a finite value, we can integrate the relationship \eqref{iteration} to show that 
\beq
\int_0^t W_s ds \le Y_t - Y_0 - \int_0^t Z_s ds, 
\eeq
which implies that $\int_0^\infty W_s ds \le \lim_{t\to \infty}Y_t - Y_0 - \int_0^\infty Z_s ds < \infty$.
\end{proof}
Next we can prove convergence to the stationary point under the learning rate \eqref{learning rates}.
\begin{theorem}
\label{gradient vanish}
Suppose the learning rate $\zeta_t$ satisfies \eqref{learning rates}. Then, for the gradient flow \eqref{gradient flow}, we have that $J(\bar \theta_t)$ converges to a finite value and 
\beq
\label{conv to stationary}
\lim_{t\to +\infty} \nabla_{\theta} J(f_{\bar \theta_t}) = 0.
\eeq
\end{theorem}

\begin{proof}
First we note that by the proof of Lemma $7$ in \cite{mei2020global}, we know that the eigenvalues of the Hessian matrix of $J(f_\theta)$ are smaller than $L :=\frac{8}{(1-\gamma)^3}$ and thus $\nabla_\theta J(f_\theta)$ is $L$-Lipschitz continuous with respect to $\theta$.

Then, by the gradient flow \eqref{gradient flow}, \eqref{critic error}, and chain rule, we can show that
\bae
\label{key1}
\frac{dJ(f_{\bar \theta_t})}{dt} &= \zeta_t \nabla_\theta J(f_{\bar\theta_t}) \widehat{\nabla}_{\theta} J(f_{\bar\theta_t})  \\
&= \zeta_t \|\nabla_\theta J(f_{\bar\theta_t})\|^2 + \zeta_t \nabla_\theta J(f_{\bar\theta_t}) \left( \widehat{\nabla}_{\theta} J(f_{\bar\theta_t}) - \nabla_\theta J(f_{\bar\theta_t}) \right) \\
&\ge \zeta_t\|\nabla_\theta J(f_{\bar\theta_t})\|^2 - C \zeta_t\|\nabla_\theta J(f_{\bar\theta_t})\| \cdot \|Q_t(\cdot, \cdot)- V^{f_{\bar\theta_t}}(\cdot, \cdot)\|_2\\
&\overset{(a)}{\ge} \zeta_t\|\nabla_\theta J(f_{\bar\theta_t})\|^2 - C \zeta_t\eta_t\|\nabla_\theta J(f_{\bar\theta_t})\|\\
&\overset{(b)}{\ge} (\zeta_t- C\zeta_t\eta_t)\|\nabla_\theta J(f_{\bar\theta_t})\|^2 - C\zeta_t\eta_t\\
&\overset{(c)}{\ge} C\zeta_t\|\nabla_\theta J(f_{\bar\theta_t})\|^2 - C\zeta_t\eta_t.
\eae
where the step $(a)$ follows \eqref{critic error}. Step $(b)$ is by using the relationship $\|\nabla_\theta J(f_{\bar\theta_t})\| \le 1+ \|\nabla_\theta J(f_{\bar\theta_t})\|_2^2$ and step $(c)$ is because $\eta_t \to 0$ and $C_1,C_2$ are some sufficiently large enough constants. Then, by Lemma \ref{convergence lemma} and the assumption in \eqref{learning rates}, we can show that either $J(f_{\bar\theta_t})\to \infty$ or $J(f_{\bar\theta_t})$ converges to a finite value and 
\beq
\label{key bound}
\int_0^{+\infty} \zeta_t\|\nabla_\theta J(f_{\bar\theta_t})\|^2dt < \infty.
\eeq 	
\textcolor{black}{Note that $J(f_\theta) = \e_{f_\theta}\left[\sum_{k=0}^{+\infty} \gamma^k r(x_k, a_k)\right]$. Therefore, the objective function $J$ is bounded by Assumption \ref{MDP} and thus we know $J(\bar\theta_t)$ converges to a finite value and \eqref{key bound} is valid.}

If there existed an $\epsilon_0 >0$ and $\bar t >0$ such that $\|\nabla_\theta J(f_{\bar\theta_t})\| \ge \epsilon_0$ for all $t \ge \bar t$, we would have  
\beq
\int_{\bar t}^{+\infty} \zeta_t\|\nabla_\theta J(f_{\bar\theta_t})\|^2dt \ge \epsilon_0^2 \int_{\bar t}^{+\infty} \zeta_tdt = \infty,
\eeq
which contradicts \eqref{key bound}. Therefore, $\displaystyle \liminf_{t\to \infty}\|\nabla_\theta J(f_{\bar\theta_t})\| = 0$. To show that  $\displaystyle \lim_{t\to \infty}\|\nabla_\theta J(f_{\bar\theta_t})\| = 0$, assume the contrary; that is $\displaystyle \limsup_{t\to \infty}\|\nabla_\theta J(f_{\bar\theta_t})\| > 0$. Then we can find a constant $\epsilon_1>0$ and two increasing sequences $\{a_n\}_{n\ge 1}, \{b_n\}_{n\ge 1}$ such that 
\bae
a_1 <b_1 <a_2 <b_2 <a_3 <b_3 < \cdots,\\
\|\nabla_\theta J(f_{\bar\theta_{a_n}})\| < \frac{\epsilon_1}{2},\quad \|\nabla_\theta J(f_{\bar\theta_{b_n}})\| > \epsilon_1.
\eae 
Define the following cycle of stopping times:
\bae
t_n &:= \sup\{s | s\in (a_n, b_n), \|\nabla_\theta J(f_{\bar\theta_s})\| < \frac{\epsilon_1}{2} \},\\
i(t_n) &:= \inf\{s | s\in (t_n, b_n), \|\nabla_\theta J(f_{\bar\theta_s})\| > \epsilon_1 \}.
\eae
Note that $ \|\nabla_\theta J(f_{\bar\theta_t})\| $ is continuous against $t$, thus 
we have 
\bae
\label{property}
&a_n \le t_n < i(t_n) \le b_n \\
&\|\nabla_\theta J(f_{\bar\theta_{t_n}})\| = \frac{\epsilon_1}{2}, \quad \|\nabla_\theta J(f_{\bar\theta_{i(t_n)}})\| =\epsilon_1\\
&\frac{\epsilon_1}{2} \le \|\nabla_\theta J(f_{\bar\theta_s})\| \le \epsilon_1, \quad s\in(t_n, i(t_n)).
\eae
Then, by the $L$-Lipschitz property of the gradient, we have for any $t_n$
\bae
\frac{\epsilon_1}{2} &=  \|\nabla_\theta J(f_{\bar\theta_{i(t_n)}})\| - \|\nabla_\theta J(f_{\bar\theta_{t_n}})\|\\
&\le \| \nabla_\theta J(f_{\bar\theta_{i(t_n)}}) - \nabla_\theta J(f_{\bar\theta_{t_n}}) \|\\
&\le L\| \bar\theta_{i(t_n)} - \bar\theta_{t_n} \|\\
&\le L \int_{t_n}^{i(t_n)} \zeta_s\| \nabla_\theta J(f_{\bar\theta_s})\| ds + L\int_{t_n}^{i(t_n)} \zeta_s \|\widehat{\nabla}_{\theta} J(f_{\bar\theta_s}) - \nabla_\theta J(f_{\bar\theta_s})\|ds \\
&\le L \epsilon_1 \int_{t_n}^{i(t_n)} \zeta_s ds + CL\int_{t_n}^{i(t_n)} \zeta_s \eta_s ds.
\eae
From this and by \eqref{hyper property} it follows that 
\beq
\label{key2}
\frac{1}{2L} \le \liminf_{n\to \infty} \int_{t_n}^{i(t_n)} \zeta_s ds.
\eeq

Using \eqref{key1} and \eqref{property}, we see that 
\beq
J(f_{\bar\theta_{i(t_n)}}) - J(f_{\bar\theta_{t_n}}) \ge C_1 (\frac{\epsilon_1}{2})^2 \int_{t_n}^{i(t_n)} \zeta_s ds - C_2\int_{t_n}^{i(t_n)} \zeta_s\eta_s ds.
\eeq
Due to the convergence of $J(f_{\theta_{t_n}})$ and the assumption of the learning rate, this implies that 
\beq
\lim_{n\to \infty} \int_{t_n}^{i(t_n)} \zeta_s ds = 0,
\eeq
which contradicts \eqref{key2} and thus the convergence to the stationary point is proven. 
\end{proof}

\subsubsection{Global convergence} 

We now prove the global convergence rate \eqref{actor convergence} for the actor dynamic using the following steps:
\begin{itemize}
\item Derive non-uniform Łojasiewicz inequalities.
\item Adapt the method in \cite{agarwal2020optimality} to obtain the global convergence.
\item Set up the uniform Łojasiewicz inequalities and the ODE for actor convergence.
\item Analyse the ODE by a comparison lemma to get the convergence rate.
\end{itemize}

Since the objective function $J(f_\theta)$ is non-concave, the convergence to a stationary point in Theorem \ref{gradient vanish} does not guarantee global convergence to the optimal policy. As a first step, we establish the following non-uniform Łojasiewicz inequalities that show that the gradient of the objective function for any parameter value dominates the sub-optimality of the parameter. Actually, \eqref{gradient control1} is used for the case that the best action at any state $x$ is unique, while \eqref{gradient control2} is for the non-unique optimal action case.

\begin{lemma}[Non-uniform Łojasiewicz Bound]
\label{gradient control}
Choose any deterministic optimal policy $f^{*}$.
\begin{itemize}
\item Suppose for any state $\forall x \in \bm{\mathcal{X}}$, there exists unique optimal action, then we have
\beq
\label{gradient control1}
\left\| \nabla_\theta J(f_\theta) \right\| \geq \frac{1}{\sqrt{ \bm{|\mathcal{X}|}} } \cdot\left\|\frac{\nu_{\mu}^{f^*}}{\nu_{\mu}^{f_\theta}}\right\|_{\infty}^{-1} \cdot \min_{x} f_{\theta}\left(x, a^{*}(x) \right) \cdot\left[ J(f^*) - J(f_\theta)\right]
\eeq
where $a^{*}(x)=\displaystyle \arg \max _{a} V^{f^{*}}(x,a), \forall x \in \bm{\mathcal{X}}$.
\item When under some state $x \in \bm{\mathcal{X}}$, there is an ``optimal action set'':  
\beq
\label{best actions}
\bm{\mathcal{A}^*}(x) :=\left\{a^*(x) \in \bm{\mathcal{A}}: V^{f^*}(x, a^*(x))= \displaystyle \max_{a} V^{f^*}(x, a)\right\},
\eeq
i.e. all actions $a^*(x) \in \bm{\mathcal{A}^*}(x)$ are the greedy action w.r.t. the optimal state-action value functin $V^{f^*}$. Given any policy $f_\theta$, construct the following optimal policy 
\beq
\label{new optimal policy}
f_{\theta}^{*}(x, a)= \begin{cases}\frac{f_{\theta}(x,a)}{\sum\limits_{a^{\prime} \in \bm{\mathcal{A}}^{*}(x)} f_{\theta}\left(x, a^{\prime}\right)}, & \text { if } a \in \bm{\mathcal{A}}^{*}(x), \\ 0, & \text { otherwise }\end{cases}
\eeq
It is obvious that $f_{\theta}^{*}$ is an optimal policy, since for all $x \in \mathcal{X}$,
$$
\sum_{a \in \bm{\mathcal{A}}^{*}(x)} f_{\theta}^{*}(x, a)=\frac{\sum\limits_{a \in \bm{\mathcal{A}}^{*}(x)} f_{\theta}(x, a)}{\sum\limits_{a^{\prime} \in \bm{\mathcal{A}}^{*}(x)} f_{\theta}\left(x, a^{\prime} \right)}=1 .
$$
Now we have 
\beq
\label{gradient control2}
\left\| \nabla_\theta J(f_\theta) \right\| \geq \frac{1}{\sqrt{ |\bm{\mathcal{X}}| |\bm{\mathcal{A}|}} } \cdot \left\|\frac{\nu_{\mu}^{f^*_\theta}}{\nu_{\mu}^{f_\theta}}\right\|_{\infty}^{-1} \cdot\left[\min_{x} \sum_{ a^*(x) \in \bm{\mathcal{A}}^*(x)} f_{\theta}(x, a^*(x))\right] \cdot \left[ J(f^*) - J(f_\theta)\right].
\eeq
\end{itemize}
\end{lemma}

\begin{remark}
As the proof of Lemma \ref{gradient control} is similar as in \cite{mei2020global}, we move the detailed proof into Appendix \ref{appendix gradient}. 
\end{remark}

Lemma \ref{gradient control} is not sufficient to prove a global convergence rate (or even global convergence). For example, the term $\min\limits_{x\in \bm{\mathcal{X}}} f_{\bar \theta_t}\left(x, a^{*}(x) \right)$ in \eqref{gradient control1} could converge to zero as $t \rightarrow \infty$. Thus to obtain \eqref{actor convergence}, we follow the steps.
\begin{itemize}
\item [(\romannumeral1)] Prove the global convergence
\beq
\label{actor}
J(f^*) - J(f_{\bar \theta_t}) \to 0, \quad t \to \infty,
\eeq
This global convergence can be proven by adapting the method in \cite{agarwal2020optimality} to the setting in our paper. 
\item [(\romannumeral2)] Due to the convergence \eqref{actor}, if for each state $x$ the best action $a^*(x)$ is unique, we will have 
\beq
\lim\limits_{t\to \infty} f_{\bar\theta_t}\left(x, a^{*}(x)\right) = 1, \quad \forall x\in \bm{\mathcal{X}}
\eeq
and thus 
\beq
\inf\limits_{x\in \mathcal{X}, t\ge 0} f_{\bar\theta_t}\left(x, a^{*}(x) \right) >0.
\eeq
If for some state $x$, the best action is not unique, then the convergence \eqref{actor} implies that
\beq
\lim\limits_{t \to \infty}\sum_{ a^*(x) \in \bm{\mathcal{A}}^{*}(x)} f_{\bar\theta_t}(x, a^*(x)) = 1, \quad \forall x \in \bm{\mathcal{X}}
\eeq
and thus
\beq
\inf\limits_{x\in \mathcal{X}, t\ge 0} \sum_{ a^*(x) \in \bm{\mathcal{A}}^{*}(x)} f_{\bar\theta_t}(x, a^*(x)) > 0.
\eeq 
\item [(\romannumeral3)] The lower bound for $\min\limits_{x} f_{\bar\theta_t}\left(x, a^{*}(x) \right),\ \min\limits_{x} \sum\limits_{ a^*(x) \in \bm{\mathcal{A}}^{*}(x)} f_{\bar\theta_t}(x, a^*(x)) $ and \eqref{gradient control1}, \eqref{gradient control2} can be used to derive the uniform Łojasiewicz inequality for MDP with unique or non-unique optimal action. By analysing the gradient flow, we can prove the convergence rate \eqref{actor convergence}. 
\end{itemize}

Now we adapt the method in \cite{agarwal2020optimality} to obtain the global convergence \eqref{actor}. For the gradient flow
\beq
\label{gradient update}
\frac{d\bar\theta_t}{dt} = \zeta_t \widehat\nabla_{\theta} J(f_{\bar\theta_t}),
\eeq
where $\widehat{\nabla}_\theta J(\bar\theta_t):= \sum\limits_{x,a} \sigma_\mu^{f_{\bar\theta_t}}(x, a) \bar Q_t(x,a)\nabla_{\theta} \log f_{\bar\theta_t}(x,a)$, with the similar calculations in \eqref{element}, it can be shown that
\bae
\label{gradient update element}
\frac{d}{dt}\bar\theta_t(x,a) &= \zeta_t \widehat\partial_{x,a} J(f_{\bar\theta_t}) \\
&= \zeta_t \sum_{x', a'} \nu_\mu^{f_\theta}(x')f_{\bar\theta_t}(x', a') \mathbbm{1}_{\{ x' = x\}} \left[ \mathbbm{1}_{\{ a' = a\}} - f_{\bar\theta_t}(x', a) \right] \bar Q_t(x', a') \\
&= \zeta_t \sum_{a'} \nu_\mu^{f_{\bar \theta_t}}(x)f_{\bar \theta_t}(x, a')  \left[ \mathbbm{1}_{\{ a' = a\}} - f_{\bar \theta_t}(x, a) \right] \bar Q_t(x, a') \\
&= \zeta_t \nu_\mu^{f_{\bar\theta_t}}(x)f_{\bar \theta_t}(x, a) \bar Q_t(x, a) - \zeta_t \nu_\mu^{f_{\bar\theta_t}}(x)f_{\bar\theta_t}(x, a)  \left[ \sum_{a'}f_{\bar\theta_t}(x, a') \bar Q_t(x, a') \right]\\
&= \zeta_t \sigma_\mu^{f_{\bar\theta_t}}(x,a)  \left[\bar Q_t(x,a) - \sum_{a'} \bar Q_t(x,a')f_{\bar\theta_t}(x,a')\right].
\eae

The following lemma is important in our proof.
\begin{lemma}[The performance difference lemma (\cite{kakade2002approximately})]
\label{difference}
For all policies $f$, $f'$ and state $x_0$, 
\beq
V^{f}(x_0) - V^{f'}(x_0) =  \sum_{x,a}\sigma_{x_0}^{f}(x,a) A^{f'}(x,a),
\eeq
where $\sigma_{x_0}^{f}$ is the visiting measure for the MDP $\bm{\mathcal{M}}$ with initial distribution $\delta_{x_0}$ and policy $f$.
\end{lemma}

We first prove the following convergence lemma for value functions  $V^{f_{\bar\theta_t}}(x)$ and $V^{f_{\bar\theta_t}}(x,a)$.
\begin{lemma}
\label{convergence}
There exists value $V^{\infty}(x)$ and $V^\infty(x,a)$ for every state $x$ and action $a$ such that 
$$
\lim\limits_{t \to \infty} V^{f_{\bar\theta_t}}(x) = V^{\infty}(x), \quad \lim\limits_{t \to \infty} V^{f_{\bar\theta_t}}(x,a) = V^{\infty}(x,a).
$$ Then, by the critic convergence \eqref{critic convergence}, we immediately have when $t \to \infty$
\bae
&\bar Q_t(x,a) \to V^{\infty}(x,a)\\
&\bar Q_t(x) := \sum_a \bar Q_t(x,a) f_{\bar\theta_t}(x,a) \to V^{\infty}(x). 
\eae
Define 
\beq
\label{Delta}
\Delta = \min_{ \{x,a| A^\infty(x,a) \ne 0\} } |A^\infty(x,a)|,
\eeq
where $A^\infty(x,a) = V^\infty(x,a) - V^\infty(x)$. Then there exists a $T_0$ such that $\forall t > T_0, (x, a) \in \bm{\mathcal{X}} \times \bm{\mathcal{A}}$, we have
\beq
\label{conv}
V^{\infty}(x,a) - \frac{\Delta}{4} \le 	Q_t(x,a) \le V^{\infty}(x,a) + \frac{\Delta}{4}.
\eeq
\end{lemma}

\begin{remark}
\label{remark}
Here we can suppose that $\Delta >0$ because if $\Delta = 0$, then we have  for any states and actions $A^\infty(x,a) = 0$. By Lemma \ref{difference},
\bae
\label{global}
&\lim_{t \to \infty} [J(f^*) - J(f_{\bar\theta_t})] \\
=& \lim_{t \to \infty} \sum_{x_0} \mu(x_0) \left[V^{f^*}(x_0) - V^{f_{\bar\theta_t}}(x_0)\right] \\
=& \lim_{t \to \infty} \sum_{x_0} \mu(x_0) \left[ \sum_{x,a}\sigma_{x_0}^{f^*}(x,a) \left[V^{f_{\bar\theta_t}}(x,a) - V^{f_{\bar\theta_t}}(x)\right] \right]\\
=& \lim_{t \to \infty} \sum_{x,a}\sigma_\mu^{f^*}(x,a) A^{f_{\bar\theta_t}}(x,a)\\
=& 0,
\eae
which immediately concludes the global convergence.
\end{remark}

\begin{proof}
For any fixed state $x_0$, treat the state value $V^{f_{\theta}}(x_0)$ as the objective function for an MDP whose initial distribution is $\delta_{x_0}$ and, by the policy gradient theorem \eqref{policy gradient theorem}, we have
\beq
\nabla_\theta V^{f_{\bar\theta_t}}(x_0) = \sum_{x,a} \sigma_{x_0}^{f_{\bar\theta_t}}(x,a) V^{f_{\bar\theta_t}}(x,a) \nabla_{\theta} \log f_{\bar\theta_t}(x,a),
\eeq
where $\sigma_{x_0}^{f_{\bar\theta_t}}(x,a)$ denotes the visiting measure of the MDP starting from $x_0$ under the policy $f_{\bar\theta_t}$. Thus, using the same calculations as in \eqref{policy gradient}, we have
\beq
\label{policy gradient 1}
\frac{\partial}{\partial_{\theta(x,a)}} V^{f_{\bar\theta_t}}(x_0) = \sigma_{x_0}^{f_{\bar\theta_t}}(x,a)A^{f_{\bar\theta_t}}(x,a)
\eeq
Let $\beta_t(x,a) = \bar Q_t(x,a) - V^{f_{\bar\theta_t}}(x,a)$ denote the critic error. Due to \eqref{critic convergence}, we know that for any state-action pair $(x,a)$, $|\beta_t(x,a)| \le C\eta_t$. Combining \eqref{gradient update element} with \eqref{policy gradient 1} and using the chain rule, we have 
\bae
\frac{d}{dt} V^{f_{\bar\theta_t}}(x_0) &= \nabla_\theta V^{f_{\bar\theta_t}}(x_0) \cdot \frac{d}{dt}\bar\theta_t \\
&= \sum_{x,a} \frac{\partial}{\partial_{\theta(x,a)}} V^{f_{\bar\theta_t}}(x_0) \frac{d}{dt}\bar\theta_t(x,a) \\
&=\zeta_t\sum_{x,a} \sigma_{x_0}^{f_{\bar\theta_t}}(x,a)  A^{f_{\bar\theta_t}}(x,a) \sigma_\mu^{f_{\bar\theta_t}}(x,a) \left[\bar Q_t(x,a) - \sum_{a'} \bar Q_t(x,a')f_{\bar\theta_t}(x,a')\right]\\
&=\zeta_t \sum_{x,a} \sigma_{x_0}^{f_{\bar\theta_t}}(x,a)  A^{f_{\bar\theta_t}}(x,a) \sigma_\mu^{f_{\bar\theta_t}}(x,a) \left[\beta_t(x,a) - \sum_{a'} \beta_t(x,a')f_{\bar\theta_t}(x,a') + A^{f_{\bar\theta_t}}(x,a)\right]\\
&\ge  \zeta_t \sum_{x,a} \sigma_{x_0}^{f_{\bar\theta_t}}(x,a) \sigma_\mu^{f_{\bar\theta_t}}(x,a) (A^{f_{\bar\theta_t}}(x,a))^2 -C\zeta_t\eta_t,
\eae
where the last inequality follows from \eqref{critic convergence}. Thus, by Lemma \ref{convergence lemma} and the boundedness of the value functions, we obtain the convergence for the state value function. Then, due to 
\beq
\label{relationship}
V^{f_{\bar\theta_t}}(x, a) = r(x, a) + \gamma \sum_{x'} V^{f_{\bar\theta_t}}(x')p(x'| x,a),
\eeq
the convergence for the state action value function is concluded. The convergence for $Q_t$ is immediately follows from the critic convergence \eqref{critic}. Combining the convergence for value functions, $\Delta > 0$, and the finiteness of the action space, we obtain \eqref{conv}. 
\end{proof}

Next, partition the action space $\mathcal{A}$ into three sets according to the value $V^\infty(x)$ and $V^\infty(x,a)$,
\bae
\label{type}
I_0^x \ &:= \{a| V^\infty(x,a) = V^\infty(x) \}\\
I_+^x \ &:= \{a| V^\infty(x,a) > V^\infty(x) \}\\
I_-^x \ &:= \{a| V^\infty(x,a) < V^\infty(x) \}.
\eae
The following steps can be used to prove the global convergence \eqref{actor}.
\begin{itemize}
\item Show that the probabilities 
$$
\lim\limits_{t \to \infty} f_{\bar\theta_t}(x,a) = 0, \quad 	\forall a\in I_+^x \cup I_-^x.
$$
\item Show that for actions $a \in I_-^x$, $\lim_{t \to \infty} \bar\theta_t(x,a) = -\infty$ and, for all actions $a\in I_+^x$, $\bar\theta_t(x,a)$ is bounded below as $t \to \infty$.
\item Prove that the set $I_+^x$ is empty by contradiction for all states $x$ and conclude the global convergence \eqref{actor}. 
\end{itemize}

\begin{lemma}
\label{advantage}
Define the advantage function for the critic as
\beq
A_t (x,a) := \bar Q_t(x,a) - \bar Q_t (x).
\eeq
Then, there exists a $T_1$ such that $\forall t \ge T_1, x \in \bm{\mathcal{X}}$, we have 
\beq
A_t(x,a) <-\frac{\Delta}{4} \quad \forall a\in I_-^x; \quad A_t(x,a) >\frac{\Delta}{4} \quad \forall a\in I_+^x.
\eeq
\end{lemma}
\begin{proof}
Since $\bar Q_t(x) \to V^\infty(x)$, we have that there exists $T_1>T_0$ such that for all $t \ge T_1$, 
\beq
\label{conv2}
V^\infty(x) - \frac{\Delta}{4} < Q_t(x) < V^\infty(x) + \frac{\Delta}{4}.
\eeq
Then, for any actions $a \in I_-^x$, we have for any $t \ge T_1 > T_0$
\bae
A_t(x,a) &= \bar Q_t(x,a) - \bar Q_t(x) \\ 
& \overset{(a)}{\le} V^{\infty}(x, a) + \frac{\Delta}{4} - \bar Q_t(x) \\ 
& \overset{(b)}{\le} V^{\infty}(x, a) + \frac{\Delta}{4} -V^{\infty}(x) + \frac{\Delta}{4} \\
&\overset{(c)}{\le} -\Delta + \frac{\Delta}{2} \\
&< - \frac{\Delta}{4},
\eae
where step (a) is by \eqref{conv}, step (b) is by \eqref{conv2} and step (c) is by the definition of $I_-^x$ in \eqref{type} and $\Delta$ in \eqref{Delta}. Similarly, for $a \in I_+^x$,
\bae
A_t(x,a) &= \bar Q_t(x,a) - \bar Q_t(x) \\ 
& \ge V^{\infty}(x, a) - \frac{\Delta}{4} - \bar Q_t(x) \\ 
& \ge V^{\infty}(x, a)- \frac{\Delta}{4} - V^{\infty}(x) - \frac{\Delta}{4} \\
&\ge \Delta - \frac{\Delta}{2} \\
&> \frac{\Delta}{4}.
\eae
\end{proof}

\begin{lemma}
\label{prob}
For any state action pair $(x,a) \in \bm{\mathcal{X}} \times \bm{\mathcal{A}}$, we have $\lim\limits_{t\to \infty} \widehat{\partial}_{x,a} J(f_{\bar\theta_t}) = 0$. This implies that
$$
\lim\limits_{t \to \infty} f_{\bar\theta_t}(x,a) = 0, \quad 	\forall a\in I_+^x \cup I_-^x,
$$
and thus 
\beq
\label{middle set prob 1}
\lim\limits_{t \to \infty} \sum_{a\in I_0^x} f_{\bar\theta_t}(x,a) = 1.
\eeq 
\end{lemma}		
\begin{lemma}[Monotonicity in $\bar\theta_t(x,a)$]
\label{monotonicity}
For all $a \in I_+^x$, $\bar\theta_t(x,a)$ is strictly increasing for $t\ge T_1$. For all $a \in I_-^x$, $\bar\theta_t(x,a)$ is strictly decreasing for $t \ge T_1$.
\end{lemma}
\begin{lemma}
\label{theta}
For any state $x$ with the set $I_+^x \ne \emptyset$, we have: 
\beq
\max_{a \in I_0^x}\bar\theta_t(x,a) \to \infty, \quad \min_{a \in \mathcal{A}}\bar\theta_t(x,a) \to -\infty.
\eeq
\end{lemma}
\textcolor{black}{The proofs of Lemmas \ref{prob}, \ref{monotonicity}, and \ref{theta} are the same as in \cite{agarwal2020optimality} and therefore are omitted.}

\begin{lemma}
\label{increasing lemma}
For all states $x$ with the set $I_+^x \ne \emptyset$, choose any $a_{+} \in I_{+}^{x}$. Then, for any $a \in I_{0}^{x}$, if there exists $t \geq T_{0}$ such that $f_{\bar\theta_t}(x,a) \le f_{\bar\theta_t}(x, a_{+})$, we have
\beq
\label{always control}
f_{\bar\theta_\tau}(x,a) \le f_{\bar\theta_\tau}(x, a_{+}), \quad \forall	\tau \ge t.
\eeq
\end{lemma}
\begin{proof}
If $f_{\bar\theta_t}(x,a) \le f_{\bar\theta_t}(x, a_{+})$, we know $\bar\theta_t(x,a) \le \bar\theta_t(x,a_{+})$ and there exists a small $\epsilon_0 >0$ such that $f_{\bar\theta_t}(x, a_{+}) \ge \epsilon_0$. Therefore, 
\bae
\widehat\partial_{x,a} J(f_{\bar\theta_t})
&= \nu_\mu^{f_{\bar\theta_t}}(x) f_{\bar\theta_t}(x,a) \left[\bar Q_t(x, a) - \bar Q_t(x) \right] \\ 
& \overset{(a)}{\le} \nu_\mu^{f_{\bar\theta_t}}(x) f_{\bar\theta_t}(x,a_+)\left[ \bar Q_t(x,a_+) - \bar Q_t(x) - \frac{\Delta}{4} \right] \\	
& \le \nu_\mu^{f_{\bar\theta_t}}(x) f_{\bar\theta_t}(x,a_+) \left[ \bar Q_t(x,a_+) - \bar Q_t(x) \right] - \epsilon_0 \nu_\mu^{f_{\bar \theta_t}}(x) \left[ \bar Q_t(x,a_+) - \bar Q_t(x) \right]  \\
&\le \widehat\partial_{x,a_+} J(f_{\bar\theta_t}) - \nu_\mu^{f_{\bar\theta_t}}(x) \frac{\Delta \epsilon_0}{4},
\eae
where the step (a) follows from $t>T_{0}$, $a\in I_0^x$ and $a_+ \in I_+^x$,
\beq
\bar Q_t(x, a_{+}) \ge V^{\infty}(x, a_{+}) - \frac{\Delta}{4} \ge V^{\infty}(x) + \Delta - \frac{\Delta}{4} = V^{\infty}(x,a) + \frac34 \Delta  > \bar Q_t(x, a) + \frac{\Delta}{4},
\eeq
and the fact that $\beta_t(x,a)$ decay exponentially. Let $C = \nu_\mu^{f_{\bar\theta_t}}(x) \frac{\Delta \epsilon_0}{4} $ and note that 
\beq
\widehat\partial_{x,a_+} J(f_{\bar\theta_t}) - C \ge 0.
\eeq
Then, we have
\beq
\bar\theta'_t(x,a) \le \bar\theta'_t(x,a_+) -C\zeta_t.
\eeq

By the gradient flow \eqref{gradient update}, Theorem \ref{gradient vanish}, and \eqref{critic error}, we have for any action $a$
\beq
\frac{\frac{d}{dt}\bar\theta_t(x,a)}{\zeta_t} = \widehat\partial_{x,a} J(f_{\bar\theta_t}) \to 0, \quad t \to \infty.
\eeq
Thus, without lose of generality, we can suppose that constant $T_0$ is large enough such that for any $t \ge T_0$ and any action $a\in \mathcal{A}$,
\beq
-\frac{C}{3} \zeta_t \le \bar\theta'_t(x,a) \le \frac{C}{3} \zeta_t.
\eeq
Thus, for any $s > t> T_0$,
\bae
\bar\theta'_s(x,a) &= \bar\theta'_s(x,a) - \bar\theta'_t(x,a) + \bar\theta'_t(x,a)\\
& \overset{(a)}{\le} \frac{C}{3} \zeta_t + \frac{C}{3} \zeta_t +\bar\theta'_t(x,a_+) -C\zeta_t\\
&\le \bar\theta'_s(x,a_+) - \frac{C}{3}\zeta_t,
\eae
where step (a) use $\zeta_t$ is decreasing. Finally, we have for any $T_0 < t \le \tau$,
\bae
\bar\theta_\tau(x,a) &= \bar\theta_t(x,a) + \int_t^\tau \bar\theta'_s(x,a) ds\\
&\le \bar\theta_t(x,a_+) + \int_t^\tau \bar\theta'_s(x,a_+) ds\\
&= \bar\theta_\tau(x,a_+).
\eae
and therefore \eqref{always control} is true. 
\end{proof}

For any $a_{+} \in I_{+}^{x}$, we divide the set $I_{0}^{x}$ into two sets $B_{0}^{x}\left(a_{+}\right)$ and $\bar{B}_{0}^{x}\left(a_{+}\right)$ as follows:
$B_{0}^{x}\left(a_{+}\right)$ is the set of all $a \in I_{0}^{x}$ such that for all $t \geq T_{0}, f_{\bar\theta_t}(x, a_{+}) < f_{\bar\theta_t}(x,a)$ and $\bar{B}_{0}^{x}\left(a_{+}\right)$ contains the remainder of the actions from $I_{0}^{x}$. By the definition of $B_{0}^{x}\left(a_{+}\right)$, we immediately have two Lemmas.
\begin{lemma}
\label{main set prob 1}
Suppose for a state $x\in \bm{\mathcal{X}}$, $I_+^{x} \neq \emptyset$. Then, $\forall a_{+} \in I_+^{x}$ we have that $B_{0}^{x}\left(a_{+}\right) \neq \emptyset$ and that
\beq
\lim\limits_{t \to \infty} \sum_{a \in B_{0}^x(a_+)}  f_{\bar\theta_t}(x,a) = 1,
\eeq
which also derives
\beq
\max _{a \in B_{0}^{x}\left(a_{+}\right)} \bar\theta_t(x, a) \rightarrow \infty.
\eeq
\end{lemma}
\begin{lemma}
\label{policy}
Consider any $x$ with $I_{+}^{x} \neq \emptyset$. Then, for any $a_{+} \in I_{+}^{x}$, there exists an $T_{a_+}$ such that for all $a \in \bar{B}_{0}^{x}\left(a_{+}\right)$
$$
f_{\bar\theta_t}(x,a_+) \ge f_{\bar\theta_t}(x,a), \quad \forall t > T_{a_{+}}.
$$
\end{lemma}
\textcolor{black}{The proofs of Lemmas \ref{main set prob 1}  and \ref{policy} are the same as in \cite{agarwal2020optimality} and therefore are omitted.}

\begin{lemma}
\label{parameter bound}
For all actions $a \in I_{+}^{x},$ we have that $\bar\theta_t(x, a)$ is bounded from below as $t \rightarrow \infty$. For all actions $a \in I_-^x,$ we have that $\bar\theta_t(x, a) \rightarrow -\infty$ as $t \rightarrow \infty$.
\end{lemma} 
\begin{proof}
From Lemma \ref{monotonicity}, we know that when $t \ge T_{1}$ and for any $a \in I_{+}^{x}$, $\bar\theta_t(x, a)$ is strictly increasing. Thus $\bar\theta_t(x, a)$ is bounded from below for any $a \in I_{+}^{x}$. For the second claim, from Lemma \ref{monotonicity} we know that when $t \ge T_{1}$, $\bar\theta_t(x, a)$ is strictly decreasing for $a \in I^{x}_-$. Therefore, by monotone convergence theorem, $\lim\limits_{t \rightarrow \infty} \bar\theta_t(x, a)$ exists and is either $-\infty$ or some constant $\epsilon_{0}$. Next, we prove the convergence to $-\infty$ by contradiction. 

Suppose for some $a \in I_{-}^{x}$ that there exists a $\epsilon_{0}$ such that $\bar\theta_t(x,a)>\epsilon_{0}, \forall t \geq T_{1}$. By Lemma \ref{theta}, we know that there exists an action $a^{\prime} \in \bm{\mathcal{A}}$ such that
\beq
\liminf _{t \rightarrow \infty} \bar\theta_t(x, a')=-\infty.
\eeq
Choose a constant $\delta>0$ such that $\bar\theta_{T_1}(x, a^{\prime}) \geq \epsilon_{0}-\delta$. Then, we can find an increasing sequence $\{t_n\}_{n\ge 0}$ larger than $T_1$ and converging to $\infty$ such that 
\beq
\label{contradict}
\theta_{t_n}(x,a') < \epsilon_0-\delta, \quad \lim_{n\to \infty}\bar\theta_{t_n}(x,a') = -\infty.
\eeq
Define $\tau_n$ as
\beq
\tau_n := \sup\{s| s\in [T_1, t_n],\  \bar\theta_s(x,a') \ge  \epsilon_0 -\delta \}
\eeq
where
\beq
\mathcal{T}^{(n)} := \{s| s\in (\tau_n, t_n),\ \widehat{\partial}_{ x,a'} J(f_{\bar\theta_s})< 0 \}
\eeq
By the continuity of $\widehat{\nabla}_\theta J(f_\theta)$, we know $\mathcal{T}^{(n)}$ is a Lebesgue measurable set. Note that the Lebesgue measure of $\mathcal{T}^{(n)}$ should be positive for all $n$. Suppose there is a constant $n$ such that $\mathcal{L}(\mathcal{T}^{(n)}) = 0$, then by $\bar\theta_{\tau_n}(x,a') \ge  \epsilon_0 -\delta$, we will have 
\bae
\bar\theta_{t_n}(x,a') &= \bar\theta_{\tau_n}(x,a') + \int_{\tau_n}^{t_n} \zeta_s  \widehat{\partial}_{ x,a'} J(f_{\bar\theta_s}) ds \\
&= \bar\theta_{\tau_n}(x,a') + \int_{ (\tau_n,t_n) \setminus \mathcal{T}^{(n)} }\zeta_s  \widehat{\partial}_{ x,a'} J(f_{\bar\theta_s}) ds \\
&\ge \bar\theta_{\tau_n}(x,a') \\
&\ge \epsilon_0-\delta,
\eae
which contradicts \eqref{contradict}. 

Define the sequence $\{Z_{n}\}_{n\ge 0}$ as 
$$
Z_{n}:=\int_{\mathcal{T}^{(n)}} \zeta_s  \widehat{\partial}_{ x,a'} J(f_{\bar\theta_s}) ds.
$$
Then,
\beq
Z_n \le \int_{\tau_n}^{t_n} \zeta_s  \widehat{\partial}_{ x,a'} J(f_{\bar\theta_s}) ds \le \bar\theta_{t_n}(x,a') - ( \epsilon_0-\delta).
\eeq
By \eqref{contradict}, this implies that
\beq
\label{contradictory}
\lim_{n \rightarrow \infty} Z_n=-\infty
\eeq
For the positive measure set $\mathcal{T}^{(n)}$, we have for any $t^{\prime} \in \mathcal{T}^{(n)}$,
\beq
\left|\frac{ \widehat{\partial}_{x,a} J(f_{\bar\theta_{t'}}) }{ \widehat{\partial}_{x,a'} J(f_{\bar\theta_{t'}})}\right|=\left|\frac{f_{\bar\theta_{t'}}(x, a) A_{t'}(x, a)}{ f_{\bar\theta_{t'}}(x, a') A_{t'} (x,a')}\right|  \geq \exp \left( \epsilon_0 - \bar\theta_{t'}(x, a')\right) \frac{(1-\gamma) \Delta}{2} 
\geq \exp (\delta) \frac{(1-\gamma) \Delta}{2}
\eeq
where we have used that $|A^{f_{\bar\theta_{t'}}}(x,a')| \le \frac{1}{1-\gamma}$, $|A^{f_{\bar\theta_{t'}}}(x,a') - A_{t'}(x,a')| \to 0$  
and $|A^{f_{\bar \theta_{t'}}}(x,a)| \ge \frac{\Delta}{4}$ for all $t'>T_{1}$ (from Lemma \ref{advantage}). Note that since $\widehat{\partial}_{x,a} J(f_{\bar\theta_{t'}})<0$ and $\widehat{\partial}_{ x,a'} J(f_{\bar\theta_{t'}})<0$  for all $t' \in \mathcal{T}^{(n)}$, we have
\beq
\label{control}
\widehat{\partial}_{x,a} J(f_{\bar\theta_{t'}})  
\le \exp (\delta) \frac{(1-\gamma) \Delta}{2} \widehat{\partial}_{x,a'} J(f_{\bar\theta_{t'}}).
\eeq
Thus
\bae
\label{contradictory2}
\bar\theta_{t_n}(x,a) &= \bar\theta_{T_1}(x, a) + \int_{T_1}^{t_n} \zeta_s \widehat{\partial}_{ x,a } J(f_{\bar\theta_{s}})ds \\
& \overset{(a)}{\le} \bar\theta_{T_1}(x, a) + \int_{ \mathcal{T}^{(n)}} \zeta_s \widehat{\partial}_{ x,a } J(f_{\bar\theta_{s}})ds\\
&  \overset{(b)}{\le}\ \bar\theta_{T_1}(x, a) + \exp (\delta) \frac{(1-\gamma) \Delta}{2} \int_{ \mathcal{T}^{(n)}} \zeta_s \widehat{\partial}_{ x,a' } J(f_{\bar\theta_{s}})ds \\
& = \bar\theta_{T_1}(x, a) + \exp (\delta) \frac{(1-\gamma) \Delta}{2} Z_{n}.
\eae
where the step $(a)$ follows from $\widehat{\partial}_{ x,a } J(f_{\bar\theta_{s}}) <0$ for any $s \ge T_1$ (Lemma \ref{monotonicity}) and step (b) is from \eqref{control}. Since \eqref{contradictory} and \eqref{contradictory2} contradict that $\bar\theta_t(x,a)$ is bounded from below, the proof is completed. 
\end{proof}

\begin{lemma}
\label{parameter}
Consider any state $x$ with $I_{+}^{x} \neq \emptyset$. We have for any $a_{+} \in I_{+}^{x}$,
\beq
\label{aim}
\lim\limits_{t \to \infty} \sum_{a \in B_{0}^{x}\left(a_{+}\right)} \bar\theta_t(x, a) = \infty
\eeq
\end{lemma}
\begin{proof}
By definition of $B_{0}^{x}(a_+)$, we know when $t\ge T_{0}$,
$$
f_{\bar\theta_t}(x,a_{+})<f_{\bar\theta_t}(x,a), \quad \forall a \in B_{0}^{x}(a_+),
$$ 
which implies $\bar\theta_t(x, a_{+}) < \bar\theta_t(x, a)$. By Lemma \ref{parameter bound}, we know $\bar\theta_t(x, a_+)$ is lower bounded as $t \rightarrow \infty$, and thus for all $a \in B_{0}^{x}(a_+)$, $\bar\theta_t(x, a)$ is lower bounded as $t \rightarrow \infty$, which together with $\max\limits_{a \in B_0^x\left(a_{+}\right)} \bar\theta_t(x, a) \rightarrow \infty$ in Lemma \ref{main set prob 1} derive \eqref{aim}.
\end{proof}

We are now ready to prove the global convergence of tabular actor-critic algorithm by following the same method in \cite{agarwal2020optimality}.
\begin{lemma}[Global convergence]
\label{global lemma}
For any optimal policy $f^{*}$, 
\beq
\label{global conv}
J(f^*) - J(f_{\bar\theta_t}) \to 0, \quad t \to \infty.
\eeq
\end{lemma}
\begin{proof}
We only need to prove $I_{+}^{x}$ is empty for any $x$. If so, by \eqref{global}
\beq
0 \le \lim_{t \to \infty} [J(f^*) - J(f_{\bar\theta_t})] = \lim_{t \to  \infty} \sum_{x,a}\sigma_\mu^{f^*}(x,a) A^{f_{\bar\theta_t}}(x,a) = \sum_{x,a}\sigma_\mu^{f^*}(x,a) \left[V^{\infty}(x,a) - V^{\infty}(x)\right] \le 0,
\eeq
which implies the global convergence \eqref{global conv}. 

Now we prove $I_+^x = \emptyset, \forall x \in \bm{\mathcal{X}}$ by contradiction. Suppose $I_+^x$ is non-empty for some state $x\in \bm{\mathcal{X}}$ and let $a_+ \in I_+^x$. Then, from Lemma \ref{parameter}, we must have 
\beq
\label{aim convergence}
\sum_{a \in B_0^x(a_+)} \bar\theta_t(x, a) \rightarrow \infty.
\eeq
By Lemma \ref{parameter bound}, we know for any $a \in I_{-}^{x}$, $\bar\theta_t(x, a) \rightarrow-\infty$ and $\bar\theta_t(x,a_+)$ is bounded from below. Thus we have
\beq
\frac{f_{\bar\theta_t}(x,a)}{f_{\bar\theta_t}(x,a_+)} = \exp\{ \bar\theta_t(x,a)-\bar\theta_t(x, a_{+}) \} \rightarrow 0,
\eeq
and there exists $T_2> T_0$ such that $\forall t \ge T_2$
\beq
\frac{f_{\bar\theta_t}(x,a)}{f_{\bar\theta_t}(x,a_+)} < \frac{(1-\gamma) \Delta}{16|\bm{\mathcal{A}}|},
\eeq
or equivalently
\beq
\label{bound1}
-\sum_{a \in I_-^{x}} \frac{f_{\bar\theta_t}(x,a)}{1-\gamma} > -f_{\bar\theta_t}(x,a_+) \frac{\Delta}{16}.
\eeq

Noting that  $\bar{B}_{0}^{x} \subset I_{0}^{x}$, we have 
\beq
\label{middle set}
\lim\limits_{t \to \infty} A_t(x, a) = 0, \quad \forall a \in \bar{B}_{0}^{x}(a_+). 
\eeq 
By Lemma \ref{policy}, 
$$
\frac{f_{\bar\theta_t}(x,a_+)}{f_{\bar\theta_t}(x,a)} \ge 1, \quad  \forall t>T_{a_{+}}, 
$$
which together \eqref{middle set} derives that there exists $T_{3}>T_{2}, T_{a_{+}}$ such that 
\beq
\left|A_t(x, a)\right| < \frac{f_{\bar\theta_t}(x,a_+)}{f_{\bar\theta_t}(x,a)} \frac{\Delta}{16|\bm{\mathcal{A}}|}, \quad \forall t \ge T_3.
\eeq
Thus we have
\beq
\sum_{a \in \bar{B}_{0}^{x}(a_+)} f_{\bar\theta_t}(x,a) \left|A_t(x, a)\right| < f_{\bar\theta_t}(x,a_+) \frac{\Delta}{16},
\eeq
or equivalently
\beq
\label{bound2}
-f_{\bar\theta_t}(x,a_+) \frac{\Delta}{16} < \sum_{a \in \bar{B}_{0}^{x}(a_+)} f_{\bar\theta_t}(x,a) A_t(x, a) < f_{\bar\theta_t}(x,a_+) \frac{\Delta}{16}.
\eeq
Then, we have for $t>T_{3}$,
\bae
0&\overset{(a)}{=} \sum_{a \in I_0^x} f_{\bar\theta_t}(x,a) A_t(x, a) + \sum_{a \in I_+^x} f_{\bar\theta_t}(x,a) A_t(x, a) + \sum_{a \in I_-^x} f_{\bar\theta_t}(x,a) A_t(x, a)\\
& \overset{(b)}{\ge} \sum_{a \in B_0^x(a_+)} f_{\bar\theta_t}(x,a) A_t(x, a) + \sum_{a \in \bar{B}_0^x(a_+)} f_{\bar\theta_t}(x,a) A_t(x, a) + f_{\bar\theta_t}(x,a_+) A_t(x, a_+) + \sum_{a \in I_-^x} f_{\bar\theta_t}(x,a) A_t(x, a)\\
& \overset{(c)}{\ge} \sum_{a \in B_0^x(a_+)} f_{\bar\theta_t}(x,a) A_t(x, a) + \sum_{a \in \bar{B}_0^x(a_+)} f_{\bar\theta_t}(x,a) A_t(x, a) + f_{\bar\theta_t}(x,a_+) \frac{\Delta}{4} - \sum_{a \in I_-^x} \frac{2f_{\bar\theta_t}(x,a)}{1-\gamma} \\
& \overset{(d)}{>} \sum_{a \in B_0^x(a_+)} f_{\bar\theta_t}(x,a) A_t(x, a) - f_{\bar\theta_t}(x,a_+) \frac{\Delta}{16}  + f_{\bar\theta_t}(x,a_+) \frac{\Delta}{4} - f_{\bar\theta_t}(x,a_+) \frac{\Delta}{8} \\
& > \sum_{a \in B_0^x(a_+)} f_{\bar\theta_t}(x,a) A_t(x, a),
\eae
where step (a) is from \eqref{gradient update element} and in the step (b) we used $A_t(x, a)>0$ for all actions $a \in I_{+}^{x}$ for $t>T_{3}>T_{1}$ from Lemma \ref{advantage}. Step (c) follows from  $A_t(x, a_+)\ge \frac{\Delta}{4}$ for $t>T_{3}>T_{1}$ from Lemma \ref{advantage}, the fact $A^{f_{\bar\theta_t}}(x, a) \geq-\frac{1}{1-\gamma}$ and the critic convergence $|A^{f_{\bar\theta_t}}(x, a) - A_t(x,a)| \to 0$, while step (d) is by \eqref{bound1} and the left inequality in \eqref{bound2}. This implies that for all $t>T_{3}$
$$
\sum_{a \in B_{0}^{x}(a_+)} \widehat{\partial}_{x,a} J(f_{\bar\theta_t}) < 0.
$$
Then,
\beq
\lim _{t \rightarrow \infty} \sum_{a \in B_{0}^{x}(a_+)}\left(\bar\theta_t(x,a) - \bar\theta_{T_3}(x, a)\right) \le \int_{T_{3}}^{\infty} \zeta_t  \sum_{a \in B_{0}^{x}(a_+)} \widehat{\partial}_{ x,a } J(f_{\bar\theta_t})  dt < \infty,
\eeq
which contradicts \eqref{aim convergence}. Therefore, the set $I_{+}^{x}$ must be empty for all $x \in \bm{\mathcal{X}} $ and then the proof is completed.
\end{proof}

The global convergence in Lemma \ref{global lemma} can also allow one to prove the global convergence of the policy. 
\begin{lemma}
\label{policy convergence}
For any deterministic optimal policy $f^*$, let $a^*(x) = \argmax_{a} f^*(x,a),\ \forall x \in \bm{\mathcal{X}}$. Recall thate the optimal actions set
$$
\bm{\mathcal{A}}^*(x):=\left\{a^*(x) \in \bm{\mathcal{A}}: V^{f^*}(x, a^*(x))=\max\limits_{a} V^{f^*}(x, a)\right\}, \quad \forall x \in \bm{\mathcal{X}}.
$$ 
Then, by the convergence \eqref{global conv}, if for each state $x$ the best action $a*(x)$ is unique, we will have 
\beq
\label{policy conv1}
\lim\limits_{t\to \infty} f_{\bar\theta_t}\left(x, a^{*}(x)\right) = 1, \quad \forall x\in \bm{\mathcal{X}}
\eeq
and thus 
\beq
\label{policy bound1}
\inf\limits_{x\in \bm{\mathcal{X}}, t\ge 0} f_{\bar\theta_t}\left(x, a^{*}(x) \right) >0.
\eeq
If for some state $x$, the best action is not unique, then the convergence \eqref{global conv} will imply
\beq
\label{policy conv2}
\lim\limits_{t \to \infty}\sum_{ a^*(x) \in \bm{\mathcal{A}}^*(x)} f_{\bar\theta_t}(x, a^*(x)) = 1, \quad \forall x \in \bm{\mathcal{X}}
\eeq
and thus
\beq
\label{policy bound2}
\inf\limits_{x\in \bm{\mathcal{X}}, t\ge 0} \sum_{ a^*(x) \in \bm{\mathcal{A}}^*(x)} f_{\bar\theta_t}(x, a^*(x)) > 0.
\eeq 
\end{lemma}
\begin{proof}
As in \eqref{global}, we have 
\bae
\label{expand}
J(f^*) - J(f_{\bar\theta_t})&=  \sum_{x}\nu_\mu^{f^*}(x) \sum\limits_a f^*(x,a) A^{f_{\bar\theta_t}}(x,a) \\
&= \sum_{x}\nu_\mu^{f^*}(x) A^{f_{\bar\theta_t}}(x,a^*(x)) \\
&= \sum_{x}\nu_\mu^{f^*}(x) \left[ V^{f_{\bar\theta_t}}(x,a^*(x)) - \sum\limits_{a'} V^{f_{\bar\theta_t}}(x,a')f_{\bar\theta_t}(x,a')\right] \\
\eae
By \eqref{global conv}, we have the convergence
\beq
0 = \lim\limits_{t \to \infty} \left[ J(f^*) - J(f_{\bar\theta_t}) \right] = \sum_{x} \mu(x) \left[ V^{f^*}(x) - V^{f_{\bar\theta_t}}(x) \right],
\eeq
which together with $\mu(x) >0, V^{f^*}(x) - V^{f_{\bar\theta_t}}(x)\ge 0,\forall x\in \bm{\mathcal{X}}$ and the relationship \eqref{relationship} leads to
\bae
\label{value conv}
&\lim\limits_{t \to \infty} V^{f^*}(x) - V^{f_{\bar\theta_t}}(x) = 0 , \quad \forall x\in \bm{\mathcal{X}}\\
&\lim\limits_{t \to \infty} V^{f^*}(x,a) - V^{f_{\bar\theta_t}}(x,a) = 0 , \quad \forall (x, a) \in \bm{\mathcal{X}}\times \bm{\mathcal{A}}.
\eae

Combining \eqref{expand} and \eqref{value conv}, we have 
\bae
\label{expand2}
0 &= \lim\limits_{t \to \infty} \left[ J(f^*) - J(f_{\bar\theta_t}) \right] \\
&= \lim\limits_{t \to \infty}\sum_{x}\nu_\mu^{f^*}(x) \left[ V^{f_{\bar\theta_t}}(x,a^*(x)) - \sum\limits_{a'}V^{f_{\bar\theta_t}}(x,a')f_{\bar\theta_t}(x,a')\right] \\
&= \lim\limits_{t \to \infty} \sum_{x}\nu_\mu^{f^*}(x) \left[ \max_a V^{f^*}(x, a) - \sum\limits_{a'} V^{f^*}(x,a')f_{\bar\theta_t}(x,a')\right] \\
&\overset{(a)}{\ge} \lim\limits_{t \to \infty} \sum_{x} \mu(x) \left[ \max_a V^{f^*}(x, a) - \sum\limits_{a'} V^{f^*}(x,a')f_{\bar\theta_t}(x,a')\right]
\eae
where step (a) is due to 
$$
\max _a V^{f^*}(x, a) - \sum\limits_{a'} V^{f^*}(x,a')f_{\bar\theta_t}(x,a') \ge 0, \quad \forall x \in \bm{\mathcal{X}}.
$$
Then we have 
\beq
\label{conv to best}
\lim\limits_{t \to \infty} \left[ V^{f^*}(x,a^*(x)) - \sum\limits_{a'} V^{f^*}(x,a')f_{\bar\theta_t}(x,a')\right] = 0, \quad \forall x \in \bm{\mathcal{X}}.
\eeq
Thus if the best action $a^*(x)$ for any state $x \in \bm{\mathcal{X}}$ is unique, \eqref{conv to best} derives 
$$
\lim\limits_{t\to \infty} f_{\bar\theta_t}\left(x, a^{*}(x)\right) = 1, \quad \forall x\in \bm{\mathcal{X}}.
$$
When there exist multiple optimal actions in $\bm{\mathcal{A}^*}(x)$, \eqref{conv to best} derives 
$$
\lim\limits_{t \to \infty}\sum_{ a^*(x) \in \bm{\mathcal{A}}^*(x)} f_{\bar\theta_t}(x, a^*(x)) = 1, \quad \forall x \in \bm{\mathcal{X}}. 
$$

Finally, noting that $f_{\theta}$ being a softmax policy and the bound in Lemma \ref{AC bound}, for any finite $t>0$, the policy is positive. Thus \eqref{policy bound1} and \eqref{policy bound2} are direct corollary of \eqref{policy conv1} and \eqref{policy conv2}.
\end{proof}

Finally, combining Lemma \ref{gradient control} and Lemma \ref{policy convergence}, we can obtain the uniform Łojasiewicz inequality, which will prove the convergence rate \eqref{actor convergence}. 
\begin{proof}[Proof of \eqref{actor convergence}:]
Define the actor error
$$
Y_t := J(f^*) - J(f_{\bar\theta_t}).
$$
Then, by chain rule,
\bae
\label{key3}
\frac{dY_t}{dt} &= -\zeta_t \nabla_\theta J(f_{\bar\theta_t}) \widehat{\nabla}_{\theta} J(f_{\bar\theta_t})  \\
&= -\zeta_t \|\nabla_\theta J(f_{\bar\theta_t})\|^2 + \zeta_t \nabla_\theta J(f_{\bar\theta_t}) \left( \nabla_\theta J(f_{\bar\theta_t}) - \widehat{\nabla}_{\theta} J(f_{\bar\theta_t}) \right) \\
&\le -\zeta_t\|\nabla_\theta J(f_{\bar\theta_t})\|^2 + C \zeta_t\|\nabla_\theta J(f_{\bar\theta_t})\| \cdot \|\bar Q_t- V^{f_{\bar\theta_t}}\|_2\\
&\le -\zeta_t\|\nabla_\theta J(f_{\bar\theta_t})\|^2 + C \zeta_t\eta_t\|\nabla_\theta J(f_{\bar\theta_t})\|\\
&\le -(\zeta_t- C\zeta_t\eta_t)\|\nabla_\theta J(f_{\bar\theta_t})\|^2 + C\zeta_t\eta_t\\
&\le -C\zeta_t\|\nabla_\theta J(f_{\bar\theta_t})\|^2 + C\zeta_t\eta_t.
\eae
By Lemma \ref{gradient control} and Lemma \ref{policy convergence}, there exists a constant $C>0$ such that
\beq
\|\nabla_\theta J(f_{\bar\theta_t})\| \ge C \left[ J(f^*) - J(f_{\bar\theta_t}) \right] = C Y_t,
\eeq
which together with \eqref{key3} derives
\bae
\frac{dY_t}{dt} &\leq- C\zeta_t Y_t^2 + C\zeta_t \eta_t\\
&< -\frac{C}{t} Y_t^2 + \frac{C}{t\log^2 t}.
\eae

Consider the comparison ODE: 
\bae
\frac{dZ_t}{dt} &= -\frac{C}{t} Z_t^2 + \frac{C}{t\log^2 t}, \quad t \ge 2.\\
Z_2 &> Y_2,
\eae
By the Basic Comparison Theorem in \cite{mcnabb1986comparison}, we have
\beq
0\le Y_t < Z_t \quad t\ge 2.
\eeq
Then, if we can establish a convergence rate for $Z_t$, we will have a convergence rate for $Y_t$.

Without loss of generality, we suppose the constant $C=1$ and define function 
$$
0 \le X_t = Z_t \log t , \quad t \ge 2.
$$ 
Thus,
\bae
\label{convergence ode}
\frac{dX_t}{dt} &= \frac{1}{t} Z_t + \log t \left( -\frac{1}{t}Z_t^2 + \frac{1}{t \log^2 t} \right)\\
&= \frac{1}{t\log t}\left( Z_t \log t - Z_t^2 \log^2 t + 1\right)\\
&= \frac{1}{t\log t}\left( X_t - X_t^2  + 1 \right), \quad t \ge 2.
\eae
Noting that $\frac{1-\sqrt 5}{2}$ and $\frac{1+\sqrt 5}{2}$ are two stationary solution of \eqref{convergence ode}, the solution $X_t$ will decrease if it is larger than $\frac{1+\sqrt 5}{2}$ and it will increase for $X_t \in [0, \frac{1+\sqrt 5}{2}]$. Thus, for a solution $X_t$ starting from $X_2 \ge 0$, there are two cases:
\begin{itemize}
\item[(1)] If the starting point $X_2 \ge \frac{1+\sqrt 5}{2}$, the solution $X_t$ will decrease and always be larger than $\frac{1+\sqrt 5}{2}$ by the uniqueness theorem for ODEs (Theorem 2.2 of \cite{teschl2012ordinary}).
\item[(2)] If the starting point $X_2 \in [0, \frac{1+\sqrt 5}{2}]$, the solution $X_t$ will increase and always be smaller than $\frac{1+\sqrt 5}{2}$ by the uniqueness theorem for ODEs (Theorem 2.2 of \cite{teschl2012ordinary}).
\end{itemize}
Thus, no matter where $X_t$ starts from, we always have 
\beq
0 \le X_t \le \max\{ X_2, \frac{1+\sqrt 5}{2} \}, \quad t \ge 2,
\eeq
which shows that
\beq
0 \le Y_t < Z_t \le \frac{C}{\log t}, \quad t \ge 2,
\eeq
and therefore the convergence rate \eqref{actor convergence} is proven.
\end{proof}

\section*{Acknowledgement}

This research has been supported by the EPSRC Centre for Doctoral Training in Mathematics of Random Systems: Analysis, Modelling and Simulation (EP/S023925/1).

\section*{Appendix}

\appendix
\renewcommand{\appendixname}{Appendix~\Alph{section}}

\section{Verification of \eqref{hyper property}}\label{learning rate property}

\bae
\int_{0}^{\infty} \zeta_{s} \eta_{s} d s &=\int_{0}^{2} \zeta_{s} \eta_{s} d s+\int_{2}^{\infty} \zeta_{s} \eta_{s} d s \\
& \leq C+\int_{2}^{\infty} \frac{1}{t \log ^{2} t} d t \\
&=C-\left.\frac{1}{\log t}\right|_{2} ^{\infty}<\infty, \\
\lim_{t \to \infty}\frac{\zeta_t}{\eta^{n}_t} &= \lim_{t \to \infty}\frac{\log^{2n} t}{t} \overset{(a)}{=} 0
\eae
where step (a) is by L'Hospital's Rule.

\section{Proof of Corollary \ref{origin mdp estimate}}\label{appendix origin mdp}

\begin{proof}	
Recall the exploration policy in \eqref{ActorwithExploration} with the decreasing exploration rate $\eta_k^N$. Then, we have  for $\forall k \le NT$,	
\beq
g_k(x,a) \ge \frac{\eta^N_{\lfloor NT \rfloor}}{d_A}, \quad \forall x,a \in \bm{\mathcal{X} \times \mathcal{A}}.
\eeq
Then, for any $\xi, \xi'$ and $k \le NT$, with the constant $C$ from \eqref{exploration lower bound},
\bae
\label{uniform lower bound}
\prob^{n_0}_{\theta_k}(\xi; \xi') &= \sum_{\xi_1, \cdots, \xi_{n_0-1}} \prob_{\theta_k}(\xi; \xi_1) \cdots \prob_{\theta_k}(\xi_{n_0-1}; \xi')\\
&= \sum_{\xi_1, \cdots, \xi_{n_0-1}} p(x_1|x,a) g_k(x_1,a_1) \cdots  p(x'|x_{n_0-1}, a_{n_0-1}) g_k(x', a')\\
&\ge C \left(\eta_{\lfloor NT \rfloor}^N \right)^{n_0}.
\eae
Thus, we can derive a lower bound for the stationary distribution 
\bae
\inf_{k \le NT} \pi^{g_k}(x', a') &= \inf\limits_{k \le NT} \sum_{x, a} \pi^{g_k}(x, a) \prob^{n_0}_{\theta_k}(x, a; x', a')\\
&\ge \inf_{k \le NT} \sum_{x, a} \pi^{g_k}(x, a) C \left(\eta_{\lfloor NT \rfloor}^N \right)^{n_0}\\
&\overset{(a)}{=} C \left(\eta_{\lfloor NT \rfloor}^N \right)^{n_0}\\
&>0,
\eae
where the step (a) is because $\pi^{g_k}$ is a probability and thus the summation equals to 1.
For the uniform geometric ergodicity, we can choose $\beta_T = \inf\limits_{k \le NT} \min\limits_{\xi, \xi'} \prob^{n_0}_{\theta_k}(\xi, \xi') >0 $ in \eqref{geometric2}, where $\beta_T >0$ is by \eqref{uniform lower bound}. Thus for $\forall k \le NT$, the Markov chain with transition probability $\prob_{\theta_k}$ satisfies the Doeblin's condition, then by Theorem 16.2.4 of \cite{meyn2012markov}, we can derive the uniform geometric ergodicity \eqref{geometric2}.
\end{proof}

\section{Proof of Lemma \ref{concentration lemma}}\label{appendix concentration}

\begin{proof}
As in the proof for the decay of $M_t^N$, we use two steps to prove the result. 
\begin{itemize}
\item [(\romannumeral1)] Prove that the fluctuations of the data samples around a dynamic stationary distribution $\pi^{g_k}$ decay when the number of iteration steps becomes large.
\item [(\romannumeral2)] Use the same method as in Lemma \ref{limit lemma} to prove the stochastic fluctuation terms vanish as $N \rightarrow \infty$. 
\end{itemize}

(\romannumeral1) To prove that for any fixed state action pair $\xi = (x,a), \forall T> 0$
\beq
\label{online convergence 2}
\lim_{N \to 0} \e\left| \frac{1}{N}\sum_{ k=0 }^{ \lfloor NT \rfloor - 1 } \left[ \mathbbm{1}_{\{ \xi_{k} = \xi\}} - \pi^{g_k}(\xi)\right] \right| = 0,
\eeq
we first introduce a similar Poisson equation for any fixed state-action pair $\xi = (x,a)$, $N \in \mathbb{N}$, $T < \infty$ and $k \le NT$,  
\beq
\label{possion 2}
\bar\nu_{\theta_k}(\xi') - \prob_{\theta_k}\bar\nu_{\theta_k}(\xi') =  \mathbbm{1}_{\{ \xi' = \xi\}} - \pi^{g_k}(\xi), \quad \xi' \in \bm{\mathcal{X}\times \mathcal{A}}.
\eeq
A solution of \eqref{possion 2} can be expressed as 
\beq
\label{nu 2}
\bar\nu_{\theta_k}(\xi') := \sum_{n\ge 0} \left[ \prob^n_{\theta_k} (\xi'; \xi) - \pi^{g_k}(\xi) \right].
\eeq
By Corollary \ref{origin mdp estimate}, there exists a constant $C_T$ (which only depends on $T$) such that 
\beq
\label{uniform 2}
\sup_{k \le NT} \left|\bar\nu_{\theta_k}(\xi')\right| \leq C_T, \quad \forall \xi' \in \bm{\mathcal{X}\times \mathcal{A}}.	
\eeq
Then, as in the proof of Lemma \ref{fluctuation}, we define the error $\bar\epsilon_k$ as
\bae
\bar\epsilon_{k} :=& \mathbbm{1}_{\{ \xi_{k+1} = \xi\}} - \pi^{g_k}(\xi)  \\
=& 	\bar\nu_{\theta_k}(\xi_{k+1}) - \prob_{\theta_k}\bar\nu_{\theta_k}(\xi_{k+1}) \\
=& \left[ \bar\nu_{\theta_{k}}( \xi_{k+1}) - \prob_{\theta_k}\bar\nu_{\theta_k}(\xi_{k}) \right] + \left[ \prob_{\theta_k}\bar\nu_{\theta_k}(\xi_{k}) - \prob_{\theta_{k}} \bar\nu_{\theta_{k}}\left( \xi_{k+1}\right) \right].
\eae
Let
\beq
\bar\psi_{\theta}(y) = \prob_{\theta} \bar\nu_{\theta}(y).
\eeq
Then, we have
\bae
\sum_{k=0}^{\lfloor NT \rfloor -1} \bar\epsilon_{k}=& \sum_{k=0}^{\lfloor NT \rfloor-1}  \left[ \bar\nu_{\theta_{k}}(\xi_{k+1})-\prob_{\theta_{k}} \bar\nu_{\theta_{k}}\left(\xi_{k}\right) \right] + \sum_{k=0}^{\lfloor NT \rfloor-1} \left[ \bar\psi_{\theta_{k}}\left(\xi_{k}\right)- \bar\psi_{\theta_{k}}\left(\xi_{k+1}\right) \right] \\
=& \sum_{k=0}^{\lfloor NT \rfloor-1}  \left[ \bar\nu_{\theta_{k}}(\xi_{k+1})-\prob_{\theta_{k}} \bar\nu_{\theta_{k}}\left(\xi_{k}\right) \right] + \sum_{k=1}^{\lfloor NT \rfloor-1} \left[ \bar\psi_{\theta_{k}}\left(\xi_{k}\right)- \bar\psi_{\theta_{k-1}}\left(\xi_{k}\right) \right] \\
+& \bar\psi_{\theta_{0}}\left(\xi_{0}\right) - \bar\psi_{\theta_{\lfloor NT \rfloor-1}}\left(\xi_{\lfloor NT \rfloor}\right)
\eae
Define the error term as 
\beq
\label{decompose 2}
\sum_{k=0}^{\lfloor NT \rfloor-1} \bar\epsilon_{k} = \sum_{k=0}^{\lfloor NT \rfloor-1} \bar\epsilon_{k}^{(1)} + \sum_{k=1}^{\lfloor NT \rfloor-1} \bar\epsilon_{k}^{(2)} + \bar\rho_{\lfloor NT \rfloor ; 0}
\eeq
where
\bae
\bar\epsilon_{k}^{(1)} &= \bar\nu_{\theta_{k}}(\xi_{k+1})-\prob_{\theta_{k}} \bar\nu_{\theta_{k}}\left(\xi_{k}\right) \\
\bar\epsilon_{k}^{(2)} &=  \bar\psi_{\theta_{k}}\left(\xi_{k}\right)- \bar\psi_{\theta_{k-1}}\left(\xi_{k}\right)\\
\bar\rho_{\lfloor NT \rfloor ; 0} &= \bar\psi_{\theta_{0}}\left(\xi_{0}\right)  - \bar\psi_{\theta_{\lfloor NT \rfloor-1}}\left(\xi_{\lfloor NT \rfloor}\right).
\eae
To prove the convergence \eqref{online convergence 2}, it suffices to appropriately bound the fluctuation term $ \left| \sum\limits_{k=0}^{\lfloor NT \rfloor-1} \bar\epsilon_k \right|$. The first term can be bounded using the martingale property while the second term can be bounded using the uniform geometric ergodicity and Lipschitz continuity. The third term is bounded using \eqref{uniform 2}. 

For the first term in \eqref{decompose 2}, note that
\beq
\e\left\{ \bar\nu_{\theta_{k}}\left(\xi_{k+1}\right) \mid \mathscr{F}_{k}\right\} = \prob_{\theta_{k}} \bar\nu_{\theta_{k}}\left(\xi_{k}\right).
\eeq
Therefore,
$$
\left\{\bar Z_n = \sum_{k=0}^{n-1} \bar\epsilon_k^{(1)}, \ \mathscr{F}_n \right\}_{n\ge 0} 
$$
is a martingale and since the conditional expectation is a contraction in $L^{2}$, we have 
\beq
\e \left| \prob_{\theta_{k}} \bar\nu_{\theta_{k}}\left(\xi_{k}\right)\right|^{2} \leq \e\left| \bar\nu_{\theta_{k}}\left(\xi_{k+1}\right) \right|^2.
\eeq
Then,
\bae
\e \left| \frac{1}{N} \sum_{k=0}^{\lfloor NT \rfloor-1} \bar\epsilon_k^{(1)} \right|^2 &=\frac{1}{N^2} \sum_{k=0}^{\lfloor NT \rfloor-1}\e\left| \bar\nu_{\theta_{k}}(\xi_{k+1})-\prob_{\theta_{k}} \bar\nu_{\theta_{k}}\left(\xi_{k}\right) \right|^{2}\\
& \leq \frac{4}{N^2} \sum_{k=0}^{\lfloor NT \rfloor-1} \e\left| \bar\nu_{\theta_{k}}(\xi_{k+1}) \right|^{2} \\
& \overset{(a)}{\le} \frac{4C_T}{N},
\eae
where the step (a) is by the uniform boundedness \eqref{uniform 2}. Thus we have for any $T >0$
\beq
\label{Q error1}
\lim_{N \to \infty} \e \left| \frac{1}{N} \sum_{k=0}^{\lfloor NT \rfloor-1} \bar\epsilon_k^{(1)} \right| = 0.
\eeq 

For the second term of \eqref{decompose 2}, by the uniform geometric ergodicity \eqref{geometric2}, for any fixed $\gamma_0>0$ we can choose $N_0$ large enough such that 
\beq
\sup_{k \le NT} \sum_{n = \lfloor N_0T \rfloor}^{\infty} \left| \prob^n_{\theta_k}(y, \xi) - \pi^{g_k}(\xi) \right| < \gamma_0, \quad \forall y \in \bm{\mathcal{X} \times \mathcal{A}} \\
\eeq
\bae
\label{lipschitz bound 2}
&\left| \frac{1}{N} \sum_{k=1}^{\lfloor NT \rfloor-1} \bar\epsilon_{k}^{(2)} \right|\\
=& \left|  \frac{1}{N} \sum_{k=1}^{\lfloor NT \rfloor-1} \left[ \bar\psi_{\theta_{k}}\left(\xi_{k}\right)- \bar\psi_{\theta_{k-1}}\left(\xi_{k}\right)\right] \right| \\
\leq& \left|  \frac{1}{N} \sum_{k=1}^{\lfloor NT \rfloor-1}  \left[\sum_{n =1}^{ \floor{N_0T}-1} \left[\prob_{\theta_{k}}^{n}\left(\xi_k, \xi\right)- \pi^{g_k}(\xi)\right] - \sum_{n=1}^{\floor{N_0T} -1} \left[\prob_{\theta_{k-1}}^{n}\left(\xi_k, \xi\right)-\pi^{g_{k-1}}(\xi)\right]\right]  \right| + 2 C_T \gamma_{0} \\
=& \left| \frac{1}{N} \sum_{k=1}^{\lfloor NT \rfloor-1} \sum_{n=1} ^{\lfloor N_0T \rfloor -1} \left[\prob_{\theta_{k}}^{n}\left(\xi_k, \xi\right) - \prob_{\theta_{k-1}}^{n}\left( \xi_k, \xi\right)\right] \right| +  \frac{\lfloor N_0T \rfloor}{N} \left| \sum_{k=1}^{\lfloor NT \rfloor-1} \left[\pi^{g_k}(\xi)-\pi^{g_{k-1}}(\xi)\right] \right| + 2C_T \gamma_{0}\\
:=& \bar I^N_1 + \bar I^N_2 + 2C_T \gamma_0.
\eae
With the exploration policy $g_k$ in \eqref{ActorwithExploration} and Lipschitz continuity in Assumption \ref{ergodic assumption}, we have
\beq
\left\| g_k- g_{k-1} \right\| \le \sum_{x,a\in \mathcal{X} \times \mathcal{A} } \left| g_k(x, a)- g_{k-1}(x, a) \right| \le C\left| \eta^N_{k} - \eta^N_{k-1} \right| + C \left\| \theta_{k} - \theta_{k-1} \right\|.
\eeq 
For any finite $n$,
\beq
\prob^{n}_{\theta_k}(\xi; \xi') = \sum_{\xi_1, \cdots, \xi_{n-1}} p(x_1|x,a) g_{k}(x_1,a_1) \cdots  p(x'|x_{n-1}, a_{n-1}) g_{k}(x', a'), \quad \forall \xi,\xi'\in \bm{\mathcal{X}\times \mathcal{A}}.
\eeq
is Lipschitz continuous in the policy $g_k$. Then, there exists a constant $C_T$ which only depends on the fixed $N_0, T$ such that
\bae
\bar I_1^N &\le \frac{ \lfloor N_0T \rfloor}{N} \sum_{k=1}^{\lfloor NT \rfloor-1} C\left\|g_k - g_{k-1}\right\| \le \frac{C_T}{N} \left[ \eta_0^N + \sum_{k=1}^{\lfloor NT \rfloor-1} \left\|\theta_k - \theta_{k-1}\right\| \right] \overset{(a)}{\le} \frac{C_T}{N},\\
\bar I_2^N &\le \frac{ \lfloor N_0T \rfloor }{N} \sum_{k=1}^{\lfloor NT \rfloor-1} C\left\|g_{k} - g_{k-1}\right\| \le \frac{C_T}{N} \left[ \eta_0^N + \sum_{k=1}^{\lfloor NT \rfloor-1} \left\|\theta_k - \theta_{k-1}\right\| \right] \overset{(a)}{\le} \frac{C_T}{N},
\eae
where step (a) is due to Lemma \ref{AC bound}:
$$
\left\|\theta_{k}-\theta_{k-1}\right\| \leq \frac{C_T}{N}, \quad \forall k \le NT. 
$$
Thus, when $\mathrm{N}$ is large enough,
\beq
\left|\frac{1}{N} \sum_{k=1}^{\lfloor NT \rfloor-1} \bar \epsilon_{k}^{(2)} \right| \leq 4C_T \gamma_{0}
\eeq
Since $\gamma_{0}$ is arbitrary, 
\beq
\label{Q error2}
\lim _{N \rightarrow \infty} \mathbf{E}\left|\frac{1}{N} \sum_{k=1}^{\lfloor NT \rfloor-1} \bar\epsilon_{k}^{(2)}\right|=0
\eeq

Obviously, for the last term of \eqref{decompose 2} by the bound in \eqref{uniform 2} we have
$$
\lim_{N\to \infty} \frac{1}{N} \bar\rho_{\lfloor NT \rfloor ; 0} = 0,
$$
which together with \eqref{Q error1} and \eqref{Q error2} derive the convergence of $ \frac{1}{N} \sum\limits_{k=0}^{\lfloor NT \rfloor-1} \bar\epsilon_k$ and \eqref{online convergence 2}.

(\romannumeral2) Following the same method in Lemma \ref{limit lemma}, we can prove the convergence of the stochastic error $M^{i,N}_t$ for $i = 1, 2,3$. 

For any $K \in \mathbb{N}$ and $\Delta=\frac{t}{K},$ we have
\begin{equation}
\begin{aligned}
&-M_{t}^{1,N}(\xi) \\
=& \sum_{j=0}^{K-1} \Delta \frac{1}{\lfloor\Delta N\rfloor} \sum_{k=j\lfloor\Delta N\rfloor}^{(j+1)\lfloor\Delta N\rfloor-1} \left(Q_k(\xi_k) \partial_{\xi} Q_k(\xi_k) - \sum_{\xi' \in \bm{\mathcal{X}} \times \bm{\mathcal{A}}} Q_k(\xi')\partial_{\xi} Q_k(\xi') \pi^{g_{k}}(\xi')\right) + o(1) \\
=& \sum_{j=0}^{K-1} \Delta \frac{1}{\lfloor\Delta N\rfloor} \sum_{k=j\lfloor\Delta N\rfloor}^{(j+1)\lfloor\Delta N\rfloor-1}  \left(Q_{j\lfloor\Delta N\rfloor}(\xi_k)\partial_{\xi} Q_{j\lfloor\Delta N\rfloor}(\xi_k) - \sum_{\xi' \in \bm{\mathcal{X}} \times \bm{\mathcal{A}}} Q_{j\lfloor\Delta N\rfloor}(\xi')\partial_{\xi} Q_{j\lfloor\Delta N\rfloor}(\xi') \pi^{g_{k}}(\xi')\right) \\
+& \sum_{j=0}^{K-1} \Delta \frac{1}{\lfloor\Delta N\rfloor} \sum_{k=j\lfloor\Delta N\rfloor}^{(j+1)\lfloor\Delta N\rfloor-1}   \Bigg[ \left(Q_k(\xi_k) \partial_{\xi} Q_k(\xi_k) - \sum_{\xi' \in \bm{\mathcal{X}} \times \bm{\mathcal{A}}} Q_k(\xi')\partial_{\xi} Q_k(\xi') \pi^{g_{k}}(\xi')\right) \\
-& \left(Q_{j\lfloor\Delta N\rfloor}(\xi_k)\partial_{\xi} Q_{j\lfloor\Delta N\rfloor}(\xi_k) - \sum_{\xi' \in \bm{\mathcal{X}} \times \bm{\mathcal{A}}} Q_{j\lfloor\Delta N\rfloor}(\xi')\partial_{\xi} Q_{j\lfloor\Delta N\rfloor}(\xi') \pi^{g_{k}}(\xi')\right) \Bigg] + o(1)\\
:=& \sum_{j=0}^{K-1} \Delta I^N_{5,j} + \sum_{j=0}^{K-1} \Delta I^N_{6,j} + o(1),       
\end{aligned}
\end{equation}
where the term $o(1)$ goes to zero, at least, in $L^{1}$ as $N \rightarrow \infty$. 

To prove the convergence of the first term, note that
\begin{equation}
\begin{aligned}
&Q_{j\lfloor\Delta N\rfloor}(\xi_k)\partial_{\xi} Q_{j\lfloor\Delta N\rfloor}(\xi_k) - \sum_{\xi'} Q_{j\lfloor\Delta N\rfloor}(\xi')\partial_{\xi} Q_{j\lfloor\Delta N\rfloor}(\xi') \pi^{g_{k}}(\xi') \\
=& \sum_{\xi'} Q_{j\lfloor\Delta N\rfloor}(\xi') \partial_{\xi} Q_{j\lfloor\Delta N\rfloor}(\xi') \mathbbm{1}_{\{ \xi_k = \xi'\}} -  \sum_{\xi'} Q_{j\lfloor\Delta N\rfloor}(\xi')\partial_{\xi} Q_{j\lfloor\Delta N\rfloor}(\xi') \pi^{g_{k}}(\xi')\\
=& \sum_{\xi'} Q_{j\lfloor\Delta N\rfloor}(\xi') \partial_{\xi} Q_{j\lfloor\Delta N\rfloor}(\xi') \left[ \mathbbm{1}_{\{ \xi_k = \xi'\}} - \pi^{g_{k}}(\xi')\right].
\end{aligned}
\end{equation}
Thus, for any $j \in 0,1, \ldots, K$,
\bae
\left| I^N_{5,j} \right| &=  \left| \frac{1}{\lfloor\Delta N\rfloor} \sum_{k=j\lfloor\Delta N\rfloor}^{(j+1)\lfloor\Delta N\rfloor-1} \sum_{\xi'} Q_{j\lfloor\Delta N\rfloor}(\xi') \partial_{\xi} Q_{j\lfloor\Delta N\rfloor}(\xi') \left[ \mathbbm{1}_{\{ \xi_k = \xi'\}} - \pi^{g_{k}}(\xi')\right] \right| \\
&= \left| \sum_{\xi'} Q_{j\lfloor\Delta N\rfloor}(\xi') \partial_{\xi} Q_{j\lfloor\Delta N\rfloor}(\xi') \frac{1}{\lfloor\Delta N\rfloor}\sum_{k=j\lfloor\Delta N\rfloor}^{(j+1)\lfloor\Delta N\rfloor-1} \left[ \mathbbm{1}_{\{ \xi_k = \xi'\}} - \pi^{g_{k}}(\xi')\right] \right|\\
&\le C \sum_{\xi'}\left| \frac{1}{\lfloor\Delta N\rfloor}\sum_{k=j\lfloor\Delta N\rfloor}^{(j+1)\lfloor\Delta N\rfloor-1}  \left[ \mathbbm{1}_{\{ \xi_k = \xi'\}} - \pi^{g_{k}}(\xi')\right] \right|,
\eae
which together with Lemma \ref{fluctuation} proves
\beq
\lim_{N \to \infty} \e\left| I^N_{5,j} \right| = 0.
\eeq
Thus,
\beq
\sum_{j=0}^{K-1} \Delta I^N_{5,j}  = \Delta \sum_{j=0}^{K-1} O(1) = t \frac{\sum_{j=0}^{K-1} O(1)}{K}, 
\eeq
which proves the convergence of the first term.

For the second term, by the bound in Lemma \ref{AC bound}, for any $k\le TN $ we have  
\begin{equation}
\begin{aligned}
&\sup _{\xi' \in \mathcal{X} \times \mathcal{A}} \left|Q_{k}(\xi')\right| \le C,\\
&\sup_{\xi' \in \mathcal{X} \times \mathcal{A}} \left| Q_k(\xi') - Q_{k-1}(\xi') \right| \le \frac{C}{N}.
\end{aligned}
\end{equation}
Note that
$$
\partial_{\xi} Q_k(\xi') = \mathbbm{1}_{\{ \xi' = \xi \}}.
$$ 
Then, by the Lipschitz continuity of the softmax transformation and the bound in Lemma \ref{AC bound},
\beq
\left| Q_{k}(\xi') \partial_{\xi} Q_k(\xi') - Q_{k-1}(\xi') \partial_{\xi} Q_{k-1}(\xi') \right| =\mathbbm{1}_{\{ \xi_k = \xi \}} \left| Q_{k}(\xi') - Q_{k-1}(\xi') \right| \le \frac{C}{N}.
\eeq
Then, for any $j \in 0,1, \cdots, K-1$ and any $k\in [j\lfloor\Delta N\rfloor, (j+1)\lfloor\Delta N\rfloor -1 ]$, 
\beq
\left| Q_k(\xi')\partial_{\xi} Q_k(\xi') - Q_{j\lfloor\Delta N\rfloor}(\xi') \partial_\xi Q_{j\lfloor\Delta N\rfloor}(\xi') \right| \le \frac{C(k-j\lfloor\Delta N\rfloor )}{N}.
\eeq
Therefore,
\bae
\sum_{j=0}^{K-1} \Delta I^N_{6,j}
&\leq C \sum_{j=0}^{K-1} \Delta \frac{1}{\lfloor\Delta N\rfloor} \sum_{k=j \lfloor \Delta N\rfloor}^{(j+1)\lfloor\Delta N\rfloor-1}  \frac{k-j\lfloor\Delta N\rfloor}{N}\\
&= C \sum_{j=0}^{K-1} \Delta \frac{1}{\lfloor\Delta N\rfloor} \sum_{k=0 }^{\lfloor\Delta N\rfloor-1} \frac{k}{N}\\
&\le  C \sum_{j=0}^{K-1} \Delta \frac{1}{\lfloor\Delta N\rfloor} \frac{\lfloor\Delta N\rfloor^2}{N}\\
&\le  C \sum_{j=0}^{K-1} \Delta  \frac{\lfloor\Delta N\rfloor}{N}\\
&\le  C \sum_{j=0}^{K-1} \Delta^2 \\
&\le  C \Delta.
\eae

Collecting our results, we have shown that
\beq
\lim _{N \rightarrow \infty} \sup _{t \in(0, T]} \e\left|M_{t}^{1,N}\right| \leq C \frac{T}{K}
\eeq
Note that $K$ was arbitrary. Consequently, we obtain
\beq
\lim _{N \rightarrow \infty} \sup _{t \in(0, T]} \e\left|M_{t}^{1,N}\right|=0,
\eeq
Using the same approach, one can prove the claim for $M_{t}^{2,N}$
and $M_{t}^{3,N}$. The details of the proof are omitted due to the similarity of the argument.
\end{proof}

\section{Proof of Lemma \ref{gradient control}}\label{appendix gradient}

\begin{proof}	
To prove \eqref{gradient control1}, note that 
\bae
\left\| \nabla_\theta J(f_\theta) \right\| &= \left[ \sum_{x, a}\left(\partial_{x,a} J(f_\theta)\right)^2 \right]^{\frac{1}{2}} \\
&\geq\left[\sum_{x}\left(\frac{\partial J(f_\theta)} {\partial \theta\left(x, a^{*}(x)\right)}\right)^{2}\right]^{\frac{1}{2}}\\
&\overset{(a)}{\geq} \frac{1}{\sqrt{|\bm{\mathcal{X}|}}} \sum_{x}\left|\frac{\partial J(f_\theta)} {\partial \theta\left(x, a^{*}(x)\right)}\right|\\
&\overset{(b)}{=} \frac{1}{\sqrt{|\bm{\mathcal{X}|}}} \sum_{x} \left| \nu_{\mu}^{f_{\theta}}(x) \cdot f_{\theta}\left(x, a^{*}(x)\right) \cdot A^{f_{\theta}}\left(x, a^{*}(x)\right)\right|\\
&= \frac{1}{\sqrt{|\bm{\mathcal{X}|}}} \sum_{x} \nu_{\mu}^{f_{\theta}}(x) \cdot f_{\theta}\left(x, a^{*}(x)\right) \cdot \left|A^{f_{\theta}}\left(x, a^{*}(x)\right)\right|,
\eae
where step (a) is by Cauthy-Schwarz inequality and step (b) is by Lemma
\ref{softmax gradient}.  

Define the coefficient as
$$
\left\|\frac{\nu_{\mu}^{f^{*}}}{\nu_{\mu}^{f}}\right\|_{\infty}=\max _{x} \frac{\nu_{\mu}^{f^{*}}(x)}{\nu_{\mu}^{*}(x)}.
$$ 
We then have the inequality:
\bae
\left\| \nabla_\theta J(f_\theta) \right\| &\geq \frac{1}{\sqrt{|\bm{\mathcal{X}|}}} \sum_{x} \frac{\nu_{\mu}^{f_{\theta}}(x)}{\nu_{\mu}^{f^{*}}(x)} \cdot \nu_{\mu}^{f^{*}}(x) \cdot f_{\theta}\left(x, a^{*}(x) \right) \cdot\left|A^{f_{\theta}}\left(x, a^{*}(x)\right)\right|  \\
&\geq \frac{1}{\sqrt{|\bm{\mathcal{X}|}}} \cdot\left\| \frac{\nu_{\mu}^{f^{*}}}{\nu_{\mu}^{f_{\theta}}}\right\|_{\infty}^{-1} \cdot \min_{x} f_{\theta}\left(x, a^{*}(x)\right) \cdot \sum_{x} \nu_{\mu}^{f^{*}}(x) \cdot\left|A^{f_{\theta}}\left(x, a^{*}(x)\right)\right| \\
&\geq \frac{1}{\sqrt{|\bm{\mathcal{X}|}}} \cdot\left\| \frac{\nu_{\mu}^{f^{*}}}{\nu_{\mu}^{f_{\theta}}}\right\|_{\infty}^{-1} \cdot \min_{x} f_{\theta}\left(x, a^{*}(x)\right) \cdot \sum_{x} \nu_{\mu}^{f^{*}}(x) \cdot A^{f_{\theta}}\left(x, a^{*}(x)\right)  \\
&\overset{(a)}{=} \frac{1}{\sqrt{|\bm{\mathcal{X}|}}} \cdot\left\| \frac{\nu_{\mu}^{f^{*}}}{\nu_{\mu}^{f_{\theta}}}\right\|_{\infty}^{-1}  \cdot \min_{x} f_{\theta}\left(x, a^{*}(x)\right) \cdot \sum_{x} \nu_{\mu}^{f^{*}}(x) \sum_{a} f^{*}(x, a) \cdot A^{f_{\theta}}(x, a) \\
&= \frac{1}{\sqrt{ |\bm{\mathcal{X}|}} } \cdot\left\|\frac{\nu_{\mu}^{f^*}}{\nu_{\mu}^{f_\theta}}\right\|_{\infty}^{-1} \cdot \min_{x} f_{\theta}\left(x, a^{*}(x) \right) \cdot\left[ J(f^*) - J(f_\theta)\right]
\eae
where step (a) uses the fact that $f^{*}$ is deterministic and in state $x$ selects $a^{*}(x)$ with probability one. The last equality uses Lemma \ref{difference}.

To prove the second claim, given a policy $f$, recall the greedy action set for each state $x$:
$$
\bm{\mathcal{A}^*}(x) =\left\{a^*(x) \in \bm{\mathcal{A}}: V^{f^*}(x, a^*(x))= \displaystyle \max_{a} V^{f^*}(x, a)\right\},
$$
By similar arguments as before, we can show that
\bae
\left\| \nabla_\theta J(f_\theta) \right\| &\geq \frac{1}{\sqrt{ |\bm{\mathcal{X}}| |\bm{\mathcal{A}}|} } \sum_{x, a}\left|\frac{\partial J(f_\theta)}{\partial \theta(x, a)}\right|\\
&= \frac{1}{\sqrt{ |\bm{\mathcal{X}}| |\bm{\mathcal{A}}|} } \sum_{x} \nu_{\mu}^{f_{\theta}}(x) \sum_{a} f_{\theta}(x, a) \cdot\left|A^{f_{\theta}}(x, a)\right| \quad \text { (by Lemma  \ref{softmax gradient})} \\
&\geq  \frac{1}{\sqrt{ |\bm{\mathcal{X}}| |\bm{\mathcal{A}}|} } \sum_{x} \nu_{\mu}^{f_{\theta}}(x) \sum_{ a^*(x) \in \bm{\mathcal{A}}^{*}(x)} f_{\theta}(x, a^*(x)) \cdot\left|A^{f_{\theta}}(x, a^*(x))\right| \\
&= \frac{1}{\sqrt{ |\bm{\mathcal{X}}| |\bm{\mathcal{A}}|} } \sum_{x} \nu_{\mu}^{f_{\theta}}(x) \sum_{ a^*(x) \in \bm{\mathcal{A}}^{*}(x)} \frac{ f_{\theta}(x, a^*(x))}{\sum\limits_{a' \in \bm{\mathcal{A}^*}(x)} f_\theta(x, a')} \cdot \left[ \sum\limits_{a' \in \bm{\mathcal{A}^*}(x)} f_\theta(x, a') \right] \cdot\left|A^{f_{\theta}}(x, a^*(x))\right|  \\ 
&\overset{(a)}{=} \frac{1}{\sqrt{ |\bm{\mathcal{X}}| |\bm{\mathcal{A}}|} } \sum_{x} \nu_{\mu}^{f_{\theta}}(x) \sum_{ a^*(x) \in \bm{\mathcal{A}}^{*}(x)}  f^*_{\theta}(x, a^*(x)) \cdot \left[ \sum\limits_{a' \in \bm{\mathcal{A}^*}(x)} f_\theta(x, a') \right] \cdot\left|A^{f_{\theta}}(x, a^*(x))\right|  \\ 
&\geq  \frac{1}{\sqrt{ |\bm{\mathcal{X}}||\bm{\mathcal{A}}|} } \cdot \left\|\frac{\nu_{\mu}^{f_\theta^*}}{\nu_{\mu}^{f_\theta}}\right\|_{\infty}^{-1} \cdot \left[\min_{x} \sum\limits_{a' \in \bm{\mathcal{A}^*}(x)} f_\theta(x, a') \right] \cdot \sum_{x} \nu_{\mu}^{f_\theta^{*}}(x) \sum_{ a^*(x) \in \bm{\mathcal{A}}^{*}(x)}  f^*_{\theta}(x, a^*(x)) A^{f_{\theta}}(x, a^*(x))\\
&\overset{(b)}{=} \frac{1}{\sqrt{ |\bm{\mathcal{X}}||\bm{\mathcal{A}}|} } \cdot \left\|\frac{\nu_{\mu}^{f_\theta^*}}{\nu_{\mu}^{f_\theta}}\right\|_{\infty}^{-1} \cdot \left[\min_{x} \sum\limits_{a' \in \bm{\mathcal{A}^*}(x)} f_\theta(x, a') \right] \cdot \sum_{x} \nu_{\mu}^{f_\theta^{*}}(x) \sum_{ a \in \bm{\mathcal{A}}}  f^*_{\theta}(x, a) A^{f_{\theta}}(x, a)\\
&\overset{(c)}{=} \frac{1}{\sqrt{ |\bm{\mathcal{X}}||\bm{\mathcal{A}}|} } \cdot \left\|\frac{\nu_{\mu}^{f_\theta^*}}{\nu_{\mu}^{f_\theta}}\right\|_{\infty}^{-1} \cdot \left[\min_{x} \sum\limits_{a' \in \bm{\mathcal{A}^*}(x)} f_\theta(x, a') \right] \cdot \left[ J(f^*_\theta) - J(f_\theta) \right]\\
&\overset{(d)}{=} \frac{1}{\sqrt{ |\bm{\mathcal{X}}||\bm{\mathcal{A}}|} } \cdot \left\|\frac{\nu_{\mu}^{f_\theta^*}}{\nu_{\mu}^{f_\theta}}\right\|_{\infty}^{-1} \cdot \left[\min_{x} \sum\limits_{a' \in \bm{\mathcal{A}^*}(x)} f_\theta(x, a') \right] \cdot \left[ J(f^*) - J(f_\theta) \right],
\eae
where step $(a)$ and $(b)$ are by the definition of the optimal policy \eqref{new optimal policy}, step $(c)$ by due to difference lemma \ref{difference} and step $d$ is because of $f^*_\theta$ is optimal policy and thus $J(f_\theta^*) = J(f^*)$.
\end{proof}

\bibliographystyle{plain}
\bibliography{cite}

\begin{thebibliography}{10}

\bibitem{agarwal2020optimality}
Alekh Agarwal, Sham~M Kakade, Jason~D Lee, and Gaurav Mahajan.
\newblock Optimality and approximation with policy gradient methods in markov
  decision processes.
\newblock In {\em Conference on Learning Theory}, pages 64--66. PMLR, 2020.

\bibitem{bertsekas2000gradient}
Dimitri~P Bertsekas and John~N Tsitsiklis.
\newblock Gradient convergence in gradient methods with errors.
\newblock {\em SIAM Journal on Optimization}, 10(3):627--642, 2000.

\bibitem{bhandari2019global}
Jalaj Bhandari and Daniel Russo.
\newblock Global optimality guarantees for policy gradient methods.
\newblock {\em arXiv preprint arXiv:1906.01786}, 2019.

\bibitem{borkar1997stochastic}
Vivek~S Borkar.
\newblock Stochastic approximation with two time scales.
\newblock {\em Systems \& Control Letters}, 29(5):291--294, 1997.

\bibitem{borkar2009stochastic}
Vivek~S Borkar.
\newblock {\em Stochastic approximation: a dynamical systems viewpoint},
  volume~48.
\newblock Springer, 2009.

\bibitem{borkar1997actor}
Vivek~S Borkar and Vijaymohan~R Konda.
\newblock The actor-critic algorithm as multi-time-scale stochastic
  approximation.
\newblock {\em Sadhana}, 22(4):525--543, 1997.

\bibitem{borkar2000ode}
Vivek~S Borkar and Sean~P Meyn.
\newblock The ode method for convergence of stochastic approximation and
  reinforcement learning.
\newblock {\em SIAM Journal on Control and Optimization}, 38(2):447--469, 2000.

\bibitem{dalal2018finite}
Gal Dalal, Gugan Thoppe, Bal{\'a}zs Sz{\"o}r{\'e}nyi, and Shie Mannor.
\newblock Finite sample analysis of two-timescale stochastic approximation with
  applications to reinforcement learning.
\newblock In {\em Conference On Learning Theory}, pages 1199--1233. PMLR, 2018.

\bibitem{gupta2019finite}
Harsh Gupta, Rayadurgam Srikant, and Lei Ying.
\newblock Finite-time performance bounds and adaptive learning rate selection
  for two time-scale reinforcement learning.
\newblock {\em Advances in neural information processing systems}, 32, 2019.

\bibitem{kakade2002approximately}
Sham Kakade and John Langford.
\newblock Approximately optimal approximate reinforcement learning.
\newblock In {\em In Proc. 19th International Conference on Machine Learning}.
  Citeseer, 2002.

\bibitem{kakade2001natural}
Sham~M Kakade.
\newblock A natural policy gradient.
\newblock {\em Advances in neural information processing systems}, 14, 2001.

\bibitem{khodadadian2021finite2}
Sajad Khodadadian, Zaiwei Chen, and Siva~Theja Maguluri.
\newblock Finite-sample analysis of off-policy natural actor-critic algorithm.
\newblock In {\em International Conference on Machine Learning}, pages
  5420--5431. PMLR, 2021.

\bibitem{khodadadian2021finite}
Sajad Khodadadian, Thinh~T Doan, Siva~Theja Maguluri, and Justin Romberg.
\newblock Finite sample analysis of two-time-scale natural actor-critic
  algorithm.
\newblock {\em arXiv preprint arXiv:2101.10506}, 2021.

\bibitem{khodadadian2021linear}
Sajad Khodadadian, Prakirt~Raj Jhunjhunwala, Sushil~Mahavir Varma, and
  Siva~Theja Maguluri.
\newblock On the linear convergence of natural policy gradient algorithm.
\newblock {\em arXiv preprint arXiv:2105.01424}, 2021.

\bibitem{konda2002actor}
V~Konda.
\newblock Actor-critic algorithms (ph. d. thesis).
\newblock {\em Department of Electrical Engineering and Computer Science,
  Massachusetts Institute of Technology}, 2002.

\bibitem{konda2000actor}
Vijay~R Konda and John~N Tsitsiklis.
\newblock Actor-critic algorithms.
\newblock In {\em Advances in neural information processing systems}, pages
  1008--1014. Citeseer, 2000.

\bibitem{konda2003linear}
Vijay~R Konda and John~N Tsitsiklis.
\newblock Linear stochastic approximation driven by slowly varying markov
  chains.
\newblock {\em Systems \& control letters}, 50(2):95--102, 2003.

\bibitem{konda1999actor}
Vijaymohan~R Konda and Vivek~S Borkar.
\newblock Actor-critic--type learning algorithms for markov decision processes.
\newblock {\em SIAM Journal on control and Optimization}, 38(1):94--123, 1999.

\bibitem{kumar2019sample}
Harshat Kumar, Alec Koppel, and Alejandro Ribeiro.
\newblock On the sample complexity of actor-critic method for reinforcement
  learning with function approximation.
\newblock {\em arXiv preprint arXiv:1910.08412}, 2019.

\bibitem{lawler2006introduction}
Gregory~F Lawler.
\newblock {\em Introduction to stochastic processes}.
\newblock CRC Press, 2006.

\bibitem{mcnabb1986comparison}
Alex McNabb.
\newblock Comparison theorems for differential equations.
\newblock {\em Journal of mathematical analysis and applications},
  119(1-2):417--428, 1986.

\bibitem{mei2021understanding}
Jincheng Mei, Bo~Dai, Chenjun Xiao, Csaba Szepesvari, and Dale Schuurmans.
\newblock Understanding the effect of stochasticity in policy optimization.
\newblock {\em Advances in Neural Information Processing Systems}, 34, 2021.

\bibitem{mei2020global}
Jincheng Mei, Chenjun Xiao, Csaba Szepesvari, and Dale Schuurmans.
\newblock On the global convergence rates of softmax policy gradient methods.
\newblock In {\em International Conference on Machine Learning}, pages
  6820--6829. PMLR, 2020.

\bibitem{melo2001convergence}
Francisco~S Melo.
\newblock Convergence of q-learning: A simple proof.
\newblock {\em Institute Of Systems and Robotics, Tech. Rep}, pages 1--4, 2001.

\bibitem{meyn2012markov}
Sean~P Meyn and Richard~L Tweedie.
\newblock {\em Markov chains and stochastic stability}.
\newblock Springer Science \& Business Media, 2012.

\bibitem{norris1998markov}
James~R Norris, John~Robert Norris, and James~Robert Norris.
\newblock {\em Markov chains}.
\newblock Number~2. Cambridge university press, 1998.

\bibitem{peters2008natural}
Jan Peters and Stefan Schaal.
\newblock Natural actor-critic.
\newblock {\em Neurocomputing}, 71(7-9):1180--1190, 2008.

\bibitem{schulman2015trust}
John Schulman, Sergey Levine, Pieter Abbeel, Michael Jordan, and Philipp
  Moritz.
\newblock Trust region policy optimization.
\newblock In {\em International conference on machine learning}, pages
  1889--1897. PMLR, 2015.

\bibitem{schulman2017proximal}
John Schulman, Filip Wolski, Prafulla Dhariwal, Alec Radford, and Oleg Klimov.
\newblock Proximal policy optimization algorithms.
\newblock {\em arXiv preprint arXiv:1707.06347}, 2017.

\bibitem{sirignano2019asymptotics}
Justin Sirignano and Konstantinos Spiliopoulos.
\newblock Asymptotics of reinforcement learning with neural networks.
\newblock {\em arXiv preprint arXiv:1911.07304}, 2019.

\bibitem{sutton2018reinforcement}
Richard~S Sutton and Andrew~G Barto.
\newblock {\em Reinforcement learning: An introduction}.
\newblock MIT press, 2018.

\bibitem{sutton1999policy}
Richard~S Sutton, David~A McAllester, Satinder~P Singh, Yishay Mansour, et~al.
\newblock Policy gradient methods for reinforcement learning with function
  approximation.
\newblock In {\em NIPs}, volume~99, pages 1057--1063. Citeseer, 1999.

\bibitem{szepesvari1997asymptotic}
Csaba Szepesv{\'a}ri.
\newblock The asymptotic convergence-rate of q-learning.
\newblock {\em Advances in neural information processing systems}, 10, 1997.

\bibitem{teschl2012ordinary}
Gerald Teschl.
\newblock {\em Ordinary differential equations and dynamical systems}, volume
  140.
\newblock American Mathematical Soc., 2012.

\bibitem{wang2019neural}
Lingxiao Wang, Qi~Cai, Zhuoran Yang, and Zhaoran Wang.
\newblock Neural policy gradient methods: Global optimality and rates of
  convergence.
\newblock {\em arXiv preprint arXiv:1909.01150}, 2019.

\bibitem{watkins1992q}
Christopher~JCH Watkins and Peter Dayan.
\newblock Q-learning.
\newblock {\em Machine learning}, 8(3):279--292, 1992.

\bibitem{wu2020finite}
Yue Wu, Weitong Zhang, Pan Xu, and Quanquan Gu.
\newblock A finite time analysis of two time-scale actor critic methods.
\newblock {\em arXiv preprint arXiv:2005.01350}, 2020.

\bibitem{xu2020non}
Tengyu Xu, Zhe Wang, and Yingbin Liang.
\newblock Non-asymptotic convergence analysis of two time-scale (natural)
  actor-critic algorithms.
\newblock {\em arXiv preprint arXiv:2005.03557}, 2020.

\bibitem{yang2019global}
Zhuoran Yang, Yongxin Chen, Mingyi Hong, and Zhaoran Wang.
\newblock On the global convergence of actor-critic: A case for linear
  quadratic regulator with ergodic cost.
\newblock {\em arXiv preprint arXiv:1907.06246}, 2019.

\bibitem{yang2018finite}
Zhuoran Yang, Kaiqing Zhang, Mingyi Hong, and Tamer Ba{\c{s}}ar.
\newblock A finite sample analysis of the actor-critic algorithm.
\newblock In {\em 2018 IEEE conference on decision and control (CDC)}, pages
  2759--2764. IEEE, 2018.

\end{thebibliography}

\end{document}